\documentclass[11pt]{article}
\usepackage{amsmath,amsfonts,amsthm,mathrsfs,xcolor,bm,bbm, mathtools, mathdots}
\usepackage{amssymb}
\usepackage{verbatim}
\usepackage{graphicx}
\usepackage{setspace}
\usepackage{natbib}
\usepackage{fancyhdr}
\usepackage{makecell}
\usepackage{threeparttable}
\usepackage{multirow}
\usepackage[usestackEOL]{stackengine}
\usepackage[colorlinks=true,urlcolor=blue,citecolor=blue,linkcolor=blue,bookmarks=true]{hyperref}
\usepackage{sectsty}
\usepackage{tcolorbox}
\usepackage{float}
\usepackage[utf8]{inputenc}
\usepackage[left=1in, right=1in, top=1in, bottom=1in]{geometry}
\usepackage{caption}
\def\NN{\mathbb N}
\def\ZZ{\mathbb Z}
\def\RR{\mathbb R}

\newtheorem{assumption}{Assumption}
\newtheorem{theorem}{Theorem}
\newtheorem{lemma}{Lemma}[section]
\newtheorem{proposition}{Proposition}

\newtheorem{definition}{Definition}
\newtheorem{remark}{Remark}

\providecommand{\keywords}[1]
{
  \small	
  \textbf{\textit{Keywords:}} #1
}
\begin{document}

\title{Classification of Data Generated by Gaussian Mixture Models Using Deep ReLU Networks}

\author{Tian-Yi Zhou \and Xiaoming Huo\thanks{H. Milton Stewart School of Industrial and Systems Engineering, Georgia Institute of Technology, USA (tzhou306@gatech.edu, huo@isye.gatech.edu)}}

\maketitle
\begin{abstract}
This paper studies the binary classification of unbounded data from ${\mathbb R}^d$ generated under Gaussian Mixture Models (GMMs) using deep ReLU neural networks. We obtain — for the first time —
non-asymptotic upper bounds and convergence rates of the excess risk (excess misclassification error) for the classification without restrictions on model parameters. The convergence rates we derive
do not depend on dimension $d$, demonstrating that deep ReLU networks can overcome the curse of dimensionality in classification. While the majority of existing generalization analysis of classification algorithms relies on a bounded domain, we consider an unbounded domain by 
leveraging the analyticity and fast decay of Gaussian distributions. 
To facilitate our analysis, we give a novel approximation error bound for general analytic functions using ReLU networks, which may be of independent interest. 
Gaussian distributions can be adopted nicely to model data arising in applications, e.g., speeches, images, and texts; our results provide a theoretical verification of the observed efficiency of deep neural networks in practical classification problems.
\end{abstract}


\section{Introduction} \label{sec:intro}
This paper studies  the binary classification of unbounded data generated by a mixture of Gaussian distributions using neural networks. 
We assume our data in  $\RR^d$ follows a class of distribution largely used
to model real-world data, namely the Gaussian Mixture Model (GMM).
Many studies have shown that GMM is an effective model for audio, speech, image, and text processing, e.g., see \citep{reynolds2000speaker, portilla2003image, blekas2005spatially}.
The universality of GMM \citep{goodfellow2016deep} motivates us to study the classification problem under such distributional assumptions on data.

In this paper, we consider data  $X \in \RR^d$ drawn from a GMM with two classes, denoted as $\{-1,1\}$,  and members of each class
are drawn from a mixture of Gaussian distributions. 
Denote the domain by  $\mathcal{X}=\RR^d$
and the output set by $\mathcal{Y}=\{-1,1\}$. We also denote  by $\rho$ a joint  distribution on $Z:= \mathcal{X} \times \mathcal{Y}$ for a GMM to be specified later. 
We are interested in learning a binary classifier $f: \RR^d \rightarrow \{-1,1\}$ using deep neural networks (DNNs).
To evaluate the effectiveness of a classifier $f$, we conduct a misclassification error analysis. 
Specifically, we examine its 
excess risk (excess misclassification error).
For any classifier $\text{sgn}(f)$ induced by a function $f:\RR^d \rightarrow \RR$, its misclassification error is defined as 
 $$ R(f) := \mathrm{E}[\mathbbm{1}\{Y\cdot \text{sgn} (f(X))=-1\}]=
    \mathrm{P}(Y\cdot \text{sgn} (f(X))=-1).$$
A Bayes classifier $f_c$ minimizes the misclassification error 
and gives the best prediction of $Y$ for a given $X$:
   \begin{equation} \label{bayes}
        f_c(X) = \begin{cases}
      1, & \text{if } \mathrm{P}(Y=1|X) \geq \mathrm{P}(Y=-1|X), \\
      -1, & \text{if }  \mathrm{P}(Y=1|X) < \mathrm{P}(Y=-1|X).
    \end{cases}  
    \end{equation}
We aim to learn a classifier $f$ as close as possible to $f_c$ using trainable DNNs.
 The accuracy of a classifier can be characterized by the excess risk  given by:  $R(f) - R(f_c)$. 
In this work, we establish fast convergence rates of excess risk of classifiers under the GMM model
generated by DNNs (given in Theorem \ref{main3}). 

The mathematical analysis of classification algorithms was initiated upon the introduction of support vector machines \citep{cortes1995support, vapnik1999nature} with a focus on margin-based analysis.
Shortly after, the universality of classification induced by kernel-based regularization schemes was established in \citep{steinwart2001influence}.
Tsybakov's noise condition \citep{tsybakov2004optimal}, together with a comparison theorem \citep{zhang2004statistical},  have facilitated the analysis of the excess risk of classification algorithms. 
Since then, a significant body of literature has emerged to study the theoretical guarantees of kernel methods in classification. All the existing work is carried out on a bounded domain \citep{steinwart2008support, campbell2011learning}.
The study of classification algorithms continues to be an active area of research in both theory and practice.


Today,  neural networks are widely considered a popular choice for classification tasks in the machine learning community, often preferred over kernel methods. Since the last decade, the development of powerful GPUs and large data sets has enabled the training of deep and complex neural networks. 
These led to breakthroughs in many fields, including computer vision, speech recognition, and natural language processing.
A rapidly growing line of literature demonstrates the accuracy and effectiveness
of DNNs in tackling classification tasks arising in practice, e.g., text and image classifications \citep{krizhevsky2017imagenet,he2016deep}.

Given the unboundedness of Gaussian distributions, we study the classification of GMM on an unbounded domain.
We would like to highlight that all existing results of classification, whether by ReLU neural networks or kernel-based classifiers, 
rely on a bounded input domain, e.g., the unit cube $[0,1]^d$ \citep{kim2021fast, bos2022convergence, shen2022approximation}, the unit sphere $ \mathbb{S}^{d-1}$ \citep{feng2021generalization}.
However, since Gaussian distributions are unbounded, existing results cannot be applied. 
In contrast to the prior works, our paper considers the unbounded domain $\RR^d$. 
Many existing approaches in the mathematical analysis of classification problems, such as covering numbers and integral operators, do not apply to unbounded input spaces. 
We extend the analysis from a bounded to an unbounded domain by leveraging the fast decay and analyticity of Gaussian functions. 
By not restricting data in a bounded set, our work speaks directly to many modern classification tasks in practice.

To conduct a generalization analysis of neural network classifiers, we adopt the Hinge loss function. 
Given a random sample $z:= \{(x_i, y_i)\}_{i=1}^n$ drawn from $\RR^d \times \{-1,1\}$, it is natural to find a classifier that minimizes the empirical risk $\frac{1}{n} \sum_{i=1}^n \mathbbm{1}\{y_i\cdot \text{sgn} (f(x_i))=-1\}$. However,  minimizing the empirical risk with the $0$-$1$ loss considered by \citep{tsybakov2004optimal, audibert2007fast} is NP-hard and thus computationally infeasible \citep{bartlett2006convexity}. 
In this paper, we adopt the well-known Hinge loss function $\phi(t):= \max \{0,1-t\}$ to make computations feasible. Learning a neural network classifier with Hinge loss is relatively straightforward owing to the gradient descent algorithm \citep{molitor2021bias, george2023training}. 
Also, there exists a well-established and neat comparison theorem from \citep{zhang2004statistical} with respect to Hinge loss, which facilitates the  generalization analysis of neural network classifiers. 

The effectiveness of a classifier can be evaluated by its excess risk.
Excess risk bounds are typically given regarding the underlying distribution $\rho$, the loss function, and the classification algorithm. 
Before we get into the main results of this paper, we would like to review some findings on the excess risk of different classifiers in the literature. We would like to pay special attention to their assumptions on the distribution $\rho$ and their uses of loss functions. 

\subsection{Related Work}

Here, we review some related work. Previously, \citep{jalali2019efficient} studied the classification of GMM data in $\RR^d$ using $1$-layer and $2$-layer neural networks with $C^\infty$ sigmoid-type activation functions. 
It considered the set $\mathcal{S}_{D,t}= \{x\in \RR^d: D(x) \geq t\}$ with $t>0$ and $D$ being the GMM discriminant function. This is a bounded set on $\RR^d$ depending on the threshold $t$.
It established a bound of the relative error $\left|\frac{\hat D(X) - D(X)}{D(X)}\right|$ for approximation on this set by $2$-layer sigmoid networks.
Neither estimates of approximation error nor excess risk is given.
They imposed several regularity assumptions on the activation function, which ReLU does not satisfy.

Due to the availability of scalable computing and stochastic optimization techniques, sigmoid neural networks have taken a back seat to ReLU networks in the last decade. 
Deep ReLU networks are extensively used nowadays in practice because they have overcome optimization hurdles of vanishing gradients and exhibit superior empirical performances. 




In the past three years,  a handful of stimulating papers, e.g., \citep{feng2021generalization, kim2021fast,shen2022approximation}, have  studied the theoretical guarantees of ReLU neural networks for binary classification on bounded domains under structural assumptions of the regression function, noise or decision boundary.
Note that the regression function is defined as the conditional mean $f_\rho(X) = \mathrm{E}[Y|X]$. 
Tsybakov's noise condition \citep{tsybakov2004optimal} with noise exponent $q \geq 0$ assumes that  $\mathrm{P}(\{X\in \mathcal{X} : |f_\rho(X)| \leq t\}) = \mathcal{O}(t^q)$.
A noteworthy work is \citep{kim2021fast}. It showed that, with Hinge loss and noise exponent $q$, the empirical risk minimizer generated from ReLU fully-connected neural networks (ReLU FNNs) achieves  rates of $\mathcal{O}\left(n^{-\frac{\alpha(q+1)}{\alpha(q+2)+(d-1)(q+1)}}\right)$ and $\mathcal{O}\left(n^{-\frac{\alpha(q+1)}{\alpha(q+2)+d}}\right)$ for the excess risk when the decision boundary is $\alpha$-Hölder smooth or when the regression function is $\alpha$-Hölder smooth, under some additional conditions on the density of $X$.

Moreover, two recent works studied the binary classification by ReLU convolutional neural networks (ReLU CNNs) \citep{feng2021generalization,shen2022approximation}. 
Feng et al. \citep{feng2021generalization} considered the $p$-norm loss $\phi(t):= \max \{0,1-t\}^p, p\geq 1$ ($1$-norm loss is the hinge loss) and input data supported on the sphere $\mathbb{S}^{d-1}$ in $\RR^d$. 
The approximation error bound and the excess risk for a target function in the Sobolev space $W^r_p(\mathbb{S}^{d-1}) \ (\text{with } r>0, p\geq 1)$ are derived under a varying power condition. 
Two quantities including $\beta = \max\{1, (d+3+r)/(2(d-1))\}$ and $\gamma \in [0,1]$ are involved in the convergence rates.
More recently, \citep{shen2022approximation} established convergence rates of the excess risk for classification with a class of convex loss functions. They considered data drawn from $d-$dimensional cube $[0,1]^d$, which is compact.
We would like to point out that the excess risk estimates of the above-mentioned works all contain a constant term depending on $r$ or $\alpha$. For example, Feng et al. gave excess risk bounds  containing a constant $2^\beta$ that increases exponentially with the smoothness index $r>0$ \citep[Theorem 2]{feng2021generalization}.

Table \ref{table:1}  summarizes convergence rates of the excess risks in the existing literature and this paper, where the logarithm factors are dropped for brevity. 
The table is based on Table 3 in \citep{shen2022approximation}.

The regression function under a GMM is entire, implying that it is infinitely differentiable. 
\textbf {All the above-mentioned excess risk estimates in the literature increase to infinity when we take the smoothness index goes to be infinity, due to the constant terms involved. }
In our work, we use the analyticity of the GMM regression function and obtain a novel result. 
We establish an excess risk estimate of order $\mathcal{O}\left(n^{-\frac{q+1}{q+2}}\right)$, where $q\geq 0$. Our result does not depend on any smoothness index or the dimension $d$. From Table \ref{table:1}, we can see that our convergence rate is faster than all existing results in the literature.

In the field of statistics, there are two classes of classification approaches --- namely  
 generative classification and discriminative classification. For GMM, the generative approach is to train a classifier by estimating the parameters (means and covariances) of the Gaussian components and then derive the Bayes classifier using the parameter estimates. 
 On the other hand, the discriminative approach is to estimate the Bayes classifier from samples directly.
 Detailed and full definitions of generative and discriminative classifications can be found 
 in \citep{ng2001discriminative, christmann2002classification, li2015fast}. 
 We would like to point out that  the classification of GMM by neural network is a discriminative approach since the procedure does not involve the estimation of GMM  parameters.

\begin{table}[h!]
\centering
\begin{tabular}{ |c|c|c|c|c|}
\hline
Reference & Function Space & Loss & Condition & Rate \\ \hline \hline
\citep{tsybakov2004optimal} & \makecell[c]{Measurable\\Functions}& $0$-$1$ loss & \multirow{2}{*}{\Longstack{ $q$-noise condition;\\$\alpha$-Hölder decision \\ boundary} } & $\mathcal{O}\left(n^{-\frac{\alpha(q+1)}{\alpha(q+2)+(d-1)q}}\right)$\\ \cline{1-3} \cline{5-5}
 \multirow{2}{*}{\citep{kim2021fast}} & \multirow{2}{*}{ReLU FNNs} & \multirow{2}{*}{Hinge} & & $\mathcal{O}\left(n^{-\frac{\alpha(q+1)}{\alpha(q+2)+(d-1) (q+1)}}\right)$\\  \cline{4-5}
 & &&  \makecell[c]{ $q$-noise condition;\\ $\alpha$-Hölder $f_\rho$ }& $\mathcal{O}\left(n^{-\frac{\alpha(q+1)}{\alpha(q+2)+d}}\right)$  \\
 \hline
\multirow{3}{*}{\citep{feng2021generalization}}& \multirow{5}{*}{ReLU CNNs} & Hinge & \multirow{2}{*}{$f_\rho \in W^r_p (\mathbb{S}^{d-1})$}  & $\mathcal{O}\left(n^{-\frac{r}{2\beta(d-1) +r(2-\tau)}}\right)$  \\ \cline{3-3} \cline{5-5}
& &$p$-norm & & $\mathcal{O}\left(n^{-\frac{pr}{2(\beta+1)(d-1) +2pr(2-\tau)}}\right)$ \\ \cline{3-5} 
 & & $2$-norm & \makecell[c]{$f_\rho \in W^r_p (\mathbb{S}^{d-1})$;\\ $q$-noise condition} & $\mathcal{O}\left(n^{-\frac{2rq}{(2+q)((\beta+1)(d-1) +2r}}\right)$ \\\cline{1-1} \cline{3-5} 
  \multirow{3}{*}{\citep{shen2022approximation}}& & Hinge & \makecell[c]{$f_\rho \in W^r_p ([0,1]^d)$;\\ $q$-noise condition} & $\mathcal{O}\left(n^{-\frac{r(q+1)}{d+2r(q+1)}}\right)$ \\ \cline{3-5}
  & & Logistic & \multirow{2}{*}{$f_\rho\in W^r_p ([0,1]^d)$} & $\mathcal{O}\left(n^{-\frac{r}{2d+4r}}\right)$ \\ \cline{3-3}  \cline{5-5} 
 & & Least Square & & $\mathcal{O}\left(n^{-\frac{4r}{3d+16r}}\right)$ \\ 
  \hline\hline
  \makecell[c]{\textbf{Theorem \ref{main3}} \\\textbf{in this work}}&ReLU FNNs & Hinge & \makecell[c]{ $q$-noise condition; \\ \textbf{GMM} }& $\mathcal{O}\left(n^{-\frac{q+1}{q+2}}\right)$ \\
  \hline
\end{tabular}
\caption{\small {This table compares the excess risk in the existing literature and this paper. The logarithm factors are dropped for brevity.}}
\label{table:1}
\end{table}
We now turn our attention to the related works in  classifications of GMM in the statistics literature.
There is a long and continuing history of research on GMM. 
A complete review of GMM-related literature is not feasible. 
Here, we would like to focus on papers most related to our work, which are papers studying the classification of GMM. 

A handful of statistics papers have studied the binary classification of GMM by generative approaches \citep{li2015fast, li2017minimax}. 
 They considered a special kind of GMM ---  GMM consisting of only two Gaussian distributions in $\RR^d$  (each Gaussian distribution corresponds to one class) with identical covariance $\Sigma$. 
 Let $\mu_0$ and $\mu_1$ denote the Gaussian means. 
 Their analysis relies on the identical covariance assumption that the Bayes classifier is a linear function of a given sample $x\in \RR^d$. 
 Under a sparsity condition $\|\Sigma^{-1}(\mu_1-\mu_0)\|_0=s$ and some additional constraints on $\mu_0$ and $\mu_1$, Li et al. \citep{li2015fast} derived an excess risk estimate depending on $s,d$, and $n$, achieved by some regularized logistic regression classifiers (under the $0$-$1$ loss).

Subsequent to the above-mentioned work, Li et al. \citep{li2017minimax} derived the excess risk estimates   (under the $0$-$1$ loss) achieved by Fisher's linear discriminant under the same GMM setting, but with an additional assumption that the Gaussian distributions are isotropic (that is, $\Sigma = \sigma^2 I$ for some known $\sigma$).
Since the analysis in the papers \citep{li2015fast} and \citep{li2017minimax} require the Bayes classifier to be linear, their approaches cannot be applied to a general GMM setting.

In this work, we will first prove that deep ReLU neural networks can generate functions that approximate entire functions well on the whole unbounded domain. 
Using this approximation result and the fact that a GMM regression function is entire, we are able to establish an excess risk estimate of the classification of a general GMM. 
Our results apply to the binary classification of a general GMM without any assumptions on the number of Gaussian components or Gaussian parameters, especially the covariance or sparsity.

\subsection{Our Contributions}
To our best knowledge, this paper presents the best rate for the excess risk of classification with a GMM without restrictions on the domain, model parameters, or the number of Gaussian components.
Our main contributions can be summarized as follows.
{\allowdisplaybreaks
\begin{enumerate}
    \item Our first result (Theorem \ref{main1}) proves that there exists a ReLU DNN that can approximate the GMM discriminant function, which will be defined shortly, to any arbitrary accuracy as the depth of the network grows. 
This novel approximation error bound is given explicitly in terms of the network parameters and model parameters.
By approximating the discriminant function, such a  ReLU
DNN can, in turn, approximate the Bayes classifier $f_c$ well w.r.t. the misclassification error. 
\item Next, our second result (Theorem \ref{main2}) shows that there exists a  ReLU DNN that can approximate a general analytic function well.
This approximation result is of independent interest and may be useful in other problems.

\item We propose a special ReLU fully-connected neural network architecture for learning the Bayes classifier from GMM data (Section \ref{sec:methodology}). The design of the special  network is based on the results in Theorem \ref{main1}, which guarantees a small approximation error when the network is sufficiently deep.
With Hinge loss and a Tsybakov-type noise condition, our third result (Theorem \ref{main3}) establishes a fast convergence rate of the excess risk of order $\mathcal{O}\left(n^{-\frac{q+1}{q+2}} (\log n)^4\right)$ using the proposed ReLU network architecture, where $q\geq 0$ is the noise exponent. We do not require the domain to be bounded.
The convergence rate we obtained is faster than the existing results, and it does not depend on the dimension $d$, which demonstrates that ReLU networks can overcome the curse of dimensionality in classification.
\end{enumerate}}
The rest of the paper is organized as follows. In Section \ref{sec:formulation}, we describe the setup of the binary classification problem and the class of ReLU-activated neural networks used in classification. 
In Section \ref{sec:approx}, we establish convergence rates on approximating GMM discriminant functions (Theorem \ref{main1}) and general analytic functions (Theorem \ref{main2}) via deep ReLU networks. 
In Section \ref{sec:methodology}, we outline the construction of a special ReLU network architecture for learning the Bayes classifier from GMM data.
In Section \ref{sec: generalization}, we show that this network architecture achieves a fast convergence rate of excess risk under Tsybakov-type noise condition (Theorem \ref{main3}). 
The proofs of Theorem \ref{main1} and Theorem \ref{main3} are provided in Sections \ref{proofmain1} and \ref{sec:proofmain3}, respectively. 
Concluding remarks are given in Section \ref{sec:conclusions}.
Whenever possible, we relegate proofs of results, technical lemmas, and propositions to the Appendix.

\section{Problem Formulations} \label{sec:formulation}
In this section, we present the problem formulations. 
In Subsection \ref{GMM}, we describe our binary classification problem and the Gaussian Mixture Model from which our data is drawn. We then establish connections between the defined classification problem with the approximation of GMM discriminant function.
In Subsection \ref{subsec:Tsybakov}, we present a Tsybakov-type noise condition that is crucial for conducting  generalization analysis of neural network classifiers. 
Lastly, in Subsection \ref{subsec:relufnn}, we formulate the ReLU fully-connected  neural networks we consider for classification. 

\subsection{Gaussian Mixture Models (GMM) and their classifiers} \label{GMM}
Consider the binary classification problem with data from the domain $\mathcal{X} = \RR^d$ 
and output label $\mathcal{Y}=\{-1,1\}$. 
A joint distribution $\rho$ on $Z:= \mathcal{X} \times \mathcal{Y}$ can be decomposed into  the marginal distribution $\rho_X$ on $\mathcal{X}$ and the conditional distributions $\rho(\cdot|X)$ at $X\in \mathcal{X}$.

We are interested in learning a binary classifier $f: \RR^d \rightarrow \{-1,1\}$. 
As mentioned earlier, the Bayes classifier $f_c$ (also known as the Bayes decision function) takes the form (\ref{bayes}).
   Since the regression function $f_\rho$ is given by the conditional mean  $f_\rho(X) = \mathrm{E}[Y|X]$ at $X\in \mathcal{X}$, we observe that $f_c = \text{sgn} (f_\rho)$ (i.e., $f_c$ is the sign of $f_\rho$).
 
In this paper, we study the binary classification problem with data generated from a Gaussian Mixture Model (GMM). Let $P^+= \mathrm{P}(y=1)$ and $P^-= \mathrm{P}(y=-1) = 1-P^+$ denote the prior probabilities that a data point is drawn from the positive and negative class, respectively. Assume that members of each class are drawn from a mixture of Gaussian distributions. Also, assume that there are overall $K$ different Gaussian distributions to draw from, where $K$ is a positive integer at least $2$.
Each Gaussian distribution is assigned uniquely to one of the two classes.
The assignment of the Gaussian distributions to the two classes is represented by sets $\mathcal{T}^+$ and $\mathcal{T}^-$  
(i.e., $\mathcal{T}^+ \cap \mathcal{T}^- = \emptyset$ and  $\mathcal{T}^+ \bigcup \mathcal{T}^- = \{1,\ldots, K\}$).

Conditioned on being in the positive class, let $p_i$, $i\in \mathcal{T}^+$, denote the probability that the data comes from Gaussian distribution $i$. Let $\mathbf{card}(S)$ be the cardinality of the set $S$. Then, for the positive class, the data are drawn from a mixture of $\mathbf{card}(\mathcal{T}^+ )$ Gaussian distributions with densities         \begin{equation*}
            \mathcal{N}(x,\mu_i, \Sigma_i) = \frac{1}{\sqrt{(2\pi)^d|\Sigma_i|}} \exp\left(-\frac{1}{2}(x-\mu_i)^T (\Sigma_i)^{-1}(x-\mu_i)\right), \qquad i\in \mathcal{T}^+,
        \end{equation*}   
where $\mu_i \in \RR^d$ and $\Sigma_i \in \RR^{d \times d}$ denote the mean and the covariance matrix of Gaussian distribution $i$. 
We do not impose any constraints on the means or the covariances, except that we require the covariances to be non-degenerate. 
Similarly, data belonging to the negative class is drawn from $\mathbf{card}(\mathcal{T}^-)$ Gaussian distributions with densities $\mathcal{N}(x,\mu_i, \Sigma_i)$, for $i\in \mathcal{T}^-$.

For the positive class, we define its discriminant function $D^+ : \RR^d \rightarrow \RR$ in terms of the prior probabilities $\{p_i\}_{i \in \mathcal{T}^+}$:
\begin{align*} \label{dis}
        D^+(x) := \mathrm{P}(x, y=1) &= P^+\sum_{i\in \mathcal{T}^+}p_i  \mathcal{N}(x,\mu_i, \Sigma_i) ,
\end{align*}
while the discriminant function for the negative class $D^- : \RR^d \rightarrow \RR$ is defined in the same way:
\begin{align*} 
        D^-(x) :=  \mathrm{P}(x, y=-1) &=P^-\sum_{j\in \mathcal{T}^-}p_j  \mathcal{N}(x,\mu_j, \Sigma_j).
\end{align*}

Observe that the density function of the marginal distribution $\rho_X$ equals $D^+(X) + D^-(X)$. The regression function $f_\rho$ can thus be expressed as
\begin{equation} \label{frho}
    f_\rho(X) = \mathrm{E}[Y|X] = \mathrm P(Y=1|X) - \mathrm P(Y=-1|X)= \frac{D^+(X)-D^-(X)}{D^+(X)+D^-(X)},
\end{equation}
which implies $$f_c(X) = \text{sgn}(f_\rho(X))= \text{sgn}\left(D^+(X)-D^-(X)\right).$$ 
In other words, we can learn the Bayes classifier $f_c(X)$ by learning the sign of  $D^+(X) - D^-(X)$.
Denote by $D(X)=D^+(X) - D^-(X)$ our function of interest.
Notice that $D$ is of the form
\begin{equation}  \label{target}
    D(x) = D^+(x) - D^-(x) = \sum_{j=1}^K \beta_j \exp(-u_j(x)),\qquad x\in \RR^d,
\end{equation}
where $\beta_j$ are the constant coefficients, and $u_j(x)$ is a quadratic term given by
\begin{equation} \label{muj}
   u_j(x) =(x-\mu_j)^T (\Sigma_j)^{-1}(x-\mu_j)/2,\qquad \forall j=1,\ldots, K.
\end{equation} 
Essentially, $D$ is a linear combination of $K$ Gaussian functions. Since all exponential functions are analytic everywhere, we know that $D$ is entire and thus infinitely differentiable. This nice property of  $D$ enables us to learn a binary classifier that learns the Bayes classifier $f_c$ with fast learning rates. 

\subsection{Tsybakov-type Noise Condition}
\label{subsec:Tsybakov}
The Tsybakov noise condition is widely used to study the quantitative behaviors  of classification algorithms, e.g., \citep{tsybakov2004optimal, audibert2007fast, kim2021fast, feng2021generalization}. 
The original Tsybakov noise condition \citep{mammen1999smooth, tsybakov2004optimal} is stated in terms of the regression function $f_\rho$. It assumes that for some $c_0>0$ and $q\in [0, \infty)$, $\mathrm{P}(\{X\in \mathcal{X} : |f_\rho(X)| \leq t\}) \leq c_0t^q,\ \forall t>0$.
Throughout this paper, we assume a Tsybakov-type noise condition as follows.  
\begin{assumption} \label{assumption1}
Assume a \textbf{Tsybakov-type noise condition} that for some $c_0>0$ and $q\in [0,\infty)$, there holds
\begin{equation} \label{Tsybakov}
\mathrm{P}(\{X\in \mathcal{X} : |D(X)| \leq t\}) \leq c_0t^q,\qquad \forall t>0,
\end{equation}
where $q$ is often referred to as the \textbf{noise exponent}.
\end{assumption}

 Since $D(X) = (D^+(X) + D^-(X))f_\rho(X)$ in our GMM setting, and we know the factor  $D^+(X) + D^-(X)$ is bounded above on the whole space $\RR^d$ and bounded below on any bounded domain, (\ref{Tsybakov}) is of the same type of Tsybakov noise condition.

\subsection{Formulation of ReLU Fully-connected Neural Network (ReLU FNN)}
\label{subsec:relufnn}
Throughout this paper, we study deep fully-connected neural networks equipped with ReLU activation functions (ReLU FNNs), where the ReLU function is defined by $\sigma(a) = \max \{a,0\}$. We consider deep ReLU FNNs that take $d$-dimensional inputs and produce one-dimensional outputs. 

To mathematically define such a class of deep ReLU FNNs, we adopt the notation used by \citep{schmidt2020nonparametric} with slight modification. The network architecture $(L,\bm{p})$ consists of a positive integer $L$, which indicates the number of hidden layers (also known as the depth),  and a width vector $\bm{p} = (p_1, \ldots,p_L) \in \NN^L$ which indicates the width in each hidden layer. A deep ReLU FNN with architecture $(L,\bm{p})$ can be written in the following compositional form
\begin{equation} \label{relufnn}
    f(x) :=f_\theta(x) = a\cdot \sigma \left(W^{(L)} \cdot \sigma \left(W^{(L-1)}\ldots \sigma \left(W^{(1)} x + b^{(1)}\right)\ldots + b^{(L-1)}\right)+ b^{(L)}\right),
\end{equation}
where $x\in \RR^d$ is the input, $a\in \RR^{p_L}$ is the outer weight, $W^{(i)}$ is a $p_i \times p_{i-1}$ weight matrix with $p_0=d$, and $b^{(i)} \in \RR^{p_i}$ is the bias vector, for $i=1,\ldots,L$. Denote by $\bm{W} = \left\{W^{(i)}\right\}_{i=1}^L$ the set of all weight matrices, $\bm{b} = \left\{b^{(i)}\right\}_{i=1}^L$ the set of all bias vectors, and $\theta = \{\bm{W}, \bm{b},a\}$ the collection of all trainable parameters in the network. 

From now on, we use $\mathcal{F}(L,\bm{p})$ to represent the set of functions of the form (\ref{relufnn}) produced by a class of ReLU FNNs with architecture $(L,\bm{p})$. 
\section{Main Results I: Universal Approximation Theorem for GMM discriminant functions and general analytic functions} \label{sec:approx}

For $x\in \RR^d$, we denote by $\|x\|$ the standard Euclidean norm in $\RR^d$, unless otherwise specified. 
Recall that we assume our GMM model consists of $K$ Gaussian distributions in total, each with mean and covariance denoted by $\mu_j$ and $\Sigma_j$, for $j=1,\ldots, K$. We define $\mu^*$ to be
\begin{equation} \label{mu*}
    \mu^*= \max_{1\leq j\leq K} \|\mu_j\|.
\end{equation}
Let $\tilde{\sigma}$ be the smallest eigenvalue of all the $K$ covariance matrices $\{\Sigma_j\}_{j=1}^K$.

Our first result proves that there exists a ReLU FNN that approximates our function of interest $D(x) = D^+(x) - D^-(x)$ very well for $x$ on a cube $[-b,b]^d$, while keeping bounded outside. The proof of Theorem \ref{main1} is given in Section \ref{proofmain1}. 

\begin{theorem} \label{main1}
Let $b\geq 1,m,\ell \in \NN, C_K^0 \geq \frac{\sqrt{d}}{\sqrt{\tilde\sigma}}, C_K^1 \geq \frac{\mu^*}{\sqrt{\tilde\sigma}}$ and $R_2 \geq \sqrt{d}(b+1)C_K^0 + C_K^1$. Consider the GMM discriminant function $D$ defined in (\ref{target}). If $\ell \geq 2\log ((R_2)^2de)/\log (2)$, there exists a function $\widetilde{D}:\RR^d \rightarrow \RR$ implementable by a ReLU FNN $\in \mathcal{F}(L,\bm{p})$ with $L = m+3+(m+1)(
\ell + 1)$ and the width vector $\bm{p}\in \NN^L$ given by 
\begin{equation*}
   \begin{cases} \nonumber
   p_1 = 4d+2dK, p_2 = 1+dK ,\\
   p_3 = p_{m+4} = 4dK,\\
   p_{i +s(m+1)} = 10dK, &  \text{if } s=0,1, i=4,\ldots, m+3, \\
      p_{(j+2)(m+1)+4} = 5K\left(2^j\right), &\text{if } j=0, \ldots, \ell-1,\\
      p_{(j+2)(m+1)+i} = 11K\left(2^j\right), & \text{if } j=0, \ldots, \ell-1, i=5,\ldots,m+4,
    \end{cases}       
    \end{equation*}
and with all weights and biases taking values in $[-4,4]$ except for the $L$-th layer such that 
\begin{equation} \label{main1bounded}
    \left|\widetilde{D}(x)-D(x)\right| \leq 
 C_{R_2} \left(\frac{2^{\ell+1}}{4^{m+1}}+2^{-\frac{\ell (2^\ell)}{2}}\right),\qquad \forall x \in [-b,b]^d
\end{equation}
and \begin{equation} \label{main1unbounded}
    \left|\widetilde{D}(x)\right| \leq C_{R_2},\qquad \forall x \in \RR^d,
\end{equation}
where $C_{R_2} = \left(\sum_{j=1}^K |\beta_j|\right) e^{(R_2)^2d}$.
\end{theorem}
Here are some interpretations of Theorem \ref{main1}. On one hand, Inequality (\ref{main1bounded}) shows that when the data is bounded in a $d$-dimensional cube, the function $\widetilde{D}$ closely approximates $D$ to any arbitrary accuracy as the depth of the ReLU network grows (i.e., as $m,\ell$  goes to infinity). On the other hand, even when the data is unbounded in $\RR^d$, Inequality (\ref{main1unbounded}) shows that the function value  $\widetilde{D}(x)$ is bounded by some constant. The error bounds given in (\ref{main1bounded}) and (\ref{main1unbounded}) are important tools for deriving the excess risk bound later in Theorem \ref{main3}. 
 
Observe that $D(x) = \sum_{j=1}^K \beta_j \exp(-u_j(x))$ is an entire function. In light of this, we can extend the above universal approximation theorem to a more general class of analytic functions. We first recall the definition of an analytic function. If a function $t$ is analytic throughout a disk $|z - u_0| <R$, centered at $u_0$ and with radius $R>0$, then $t$ has the power series representation
\begin{equation} \label{taylor}
    t(z) = \sum_{i=0}^\infty \frac{t^{(i)}(u_0)}{i!} (z-u_0)^i, \qquad |z- u_0| <R.
    \end{equation}
In other words, series (\ref{taylor}) converges to $t(z)$  when $z$ lies in the aforementioned open disk.

Now, we present our result on approximating a univariate analytic function via a deep ReLU FNN. 
\begin{theorem} \label{main2}
\label{analytic}
Let $m,\ell \in \NN$, $ 1<R_1<R_0\leq \infty$. Consider a univariate function $t(u)$ on $(-R_0, R_0)$, which can be extended to an analytic function on the disk $|z|<R_0$. For input $x\in [-1,1]$, there exists a function $F$ implementable by a deep ReLU FNN $\in \mathcal{F}(\ell(m+1),\bm{p})$ with width vector $\bm{p}\in \NN^{\ell(m+1)}$ given by, for  $j=0\ldots, \ell-1$,
\begin{equation*}
  \begin{cases} \nonumber
      p_{j(m+1)+1} = 5\left(2^j\right), \\
      p_{j(m+1)+i} = 11\left(2^j\right), & \text{for } i=2,\ldots,m+1,
    \end{cases}  
\end{equation*}
and with all weights and biases taking values in $[-4,4]$ except the last layer such that
\begin{equation}
\sup_{x\in [-1,1]}\left|F(x) - t(x)\right|  \leq
C_{R_1}\left(\frac{1}{4^{m+1}}
+ \frac{1}{R_1^{2^\ell}} \right),
\end{equation}
where $C_{R_1} = \frac{2^{7+ \lfloor\log 4 / \log R_1\rfloor }}{(R_1-1)}\sup_{i\in \ZZ_+} \left|\frac{t^{(i)}(0)}{i!} (R_1)^i\right| $. 
\end{theorem} 

Theorem \ref{main2} shows that there exists a  deep ReLU FNN that can approximate a general analytic function defined on $[-1,1]$ to any arbitrary accuracy as the ReLU network grows.
This approximation result is of independent interest and may be useful in other problems. The proof of Theorem \ref{main2} can be found in Appendix \ref{sec:analytic}.

\section{Methodology} \label{sec:methodology}
To solve the GMM classification problem effectively, our primary goal is to learn the optimal Bayes classifier $f_c = \text{sgn}(D)$ well, where $D$ is the GMM discriminant function defined earlier in (\ref{target}). 
To do so, we propose a special ReLU FNN architecture for learning the Bayes classifier from GMM data.
This special ReLU FNN has an expansive binary-tree structure. 
The design of this network architecture is based on the results in Theorem \ref{main1}, which guarantees a small approximation error for a sufficiently deep ReLU network.
For brevity, we refer to this special network as the Expansive Binary-Tree ReLU network (EBTnet).

In this section, we first outline the design of the EBTnet (Subsection \ref{section:EBTnet}). Then, we describe a preprocessing subnetwork (Subsection \ref{section:preprocess}). Lastly, in Subsection \ref{section:hypothesis}, we define our final network architecture and the hypothesis space for classification. 
 
 We will show, later in Section \ref{sec: generalization}, that this network architecture (i.e., EBTnets followed after the preprocessing subnetwork) achieves a good excess risk  bound.

\subsection{Expansive Binary-tree network} \label{section:EBTnet}
Recall that Theorem \ref{main1} shows that there exists a function  $\widetilde D$ implementable by a ReLU FNN that approximates $D$ well. From the Taylor's expansion of exponential functions, we know that $D(x)= \sum_{j=1}^K \beta_j \exp(-u_j(x))$ can be expressed as a linear combination of monomials.

The results in Theorem \ref{main1} guide us to construct a monomial gate, which is a ReLU FNN designed to approximate linear combinations of monomial functions. This network has an expansive binary-tree structure.  We start by introducing two important building blocks of the monomial gate, namely the squaring gate $\widehat{f}_m$ and the product gate $\widehat{\Phi}$.   

The squaring gate $\widehat f_m =: \widehat f_{m,\theta}$ is a ReLU FNN $\in \mathcal{F}(m,(5,5,\ldots,5))$, where $m\in \NN$. Each $\widehat{f}_m$ has parameters -- weights $\bm{W}$, bias vectors $\bm{b}$ and outer weights $a$ -- all taking values on $[-4,4]$. This network architecture, introduced by \citep{yarotsky2017error}, is used to approximate the quadratic function $f(u)=u^2$ for any input $u \geq 0$. It is demonstrated in \citep[Proposition 2]{yarotsky2017error} that there exists a specific function $f_m(u) \in \mathcal{F}(m,(5,5,\ldots,5))$ with all parameters bounded by $4$ such that 
\begin{equation} \label{Yarotskyprop2}
 f_m(u) \in [0,1]\ \hbox{ and }\ |u^2-f_m (u)|\leq 4^{-(m+1)}\qquad \forall u\in [0,1].
\end{equation}
\textbf{Here and later, we use the hat sign in $\widehat f_m$ to denote a network output function with flexible parameter choices while $f_m$ without the hat sign denotes a specific network output function with specific parameter choices.}
\begin{figure}[h]
    \centering
    \includegraphics[width=10cm]{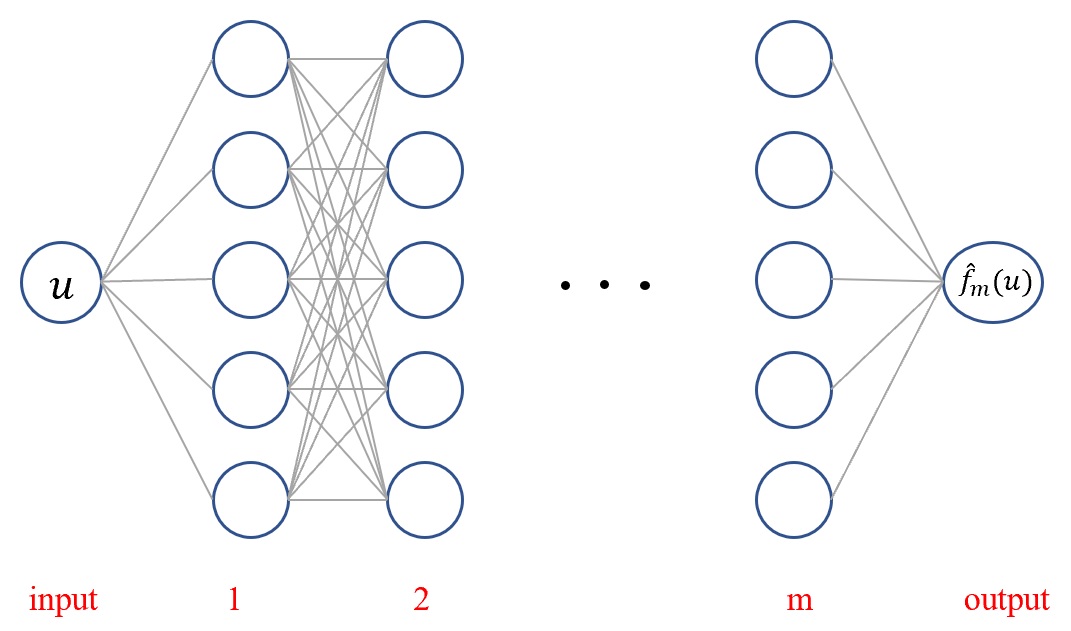}
    \caption{\small {An illustration of the network architecture $\widehat{f}_m$ for approximating $u^2$ for input $u\geq0$. $\widehat{f}_m$ is a ReLU FNN with $m$ hidden layers, and each layer has 5 neurons. The red numbers below the architecture indicate the order of hidden layers.}}
    \label{fig:fm}
\end{figure}

The main idea for constructing this specific squaring gate is to approximate $u^2$ by the network output $f_m(u):= u-\sum_{s=1}^{m}\frac{g_s(u)}{2^{2s}}$. Here, $g_s(u) := \underbrace{g \circ g \circ \cdots g(u)}_{s \hbox{ folds}} $ is a $s$-compositions of hat  functions $g:[0,\infty) \rightarrow [0,1]$ defined as
\begin{equation} \label{gdef} 
    g(u)=2\sigma(u)-4\sigma\left(u-\frac{1}{2}\right)+2\sigma(u-1)=\left\{
\begin{array}{ll}
2u,&\quad  \hbox{if} \ 0\le u\le \frac{1}{2}, \\
2(1-u),&\quad  \hbox{if} \ \frac{1}{2} < u \le 1 ,\\
0,&\quad  \hbox{if}\ u>1.
\end{array}\right.
\end{equation}
We see that the function $g$ can be implemented by a single-layer ReLU network with input $u\geq 0$. It follows that we can construct the squaring gate $\widehat f_m$, which belongs to $\mathcal{F}(m,(5,5,\ldots,5))$. A more detailed discussion on $\widehat f_m$ is given later in Subsection \ref{productgate}.

Next, invoking the identity $$u\cdot v = \left|\frac{u+v}{2}\right|^2 -  \left|\frac{u-v}{2}\right|^2$$ and $\widehat{f}_m$, we are able to construct a product gate $\widehat{\Phi} =: \widehat{\Phi}_\theta$, which is a ReLU FNN belonging to $\mathcal{F}(m+1,(4,10,10,\ldots,10))$ for $m\in \NN$. A similar network construction can be found in \citep[Lemma D.2.2]{suhapproximation}. The first hidden layer of the product gate $\widehat{\Phi}$ takes $u,v \in \RR$ as inputs and outputs $|\frac{u+v}{2}|$ and $|\frac{u-v}{2}|$ via $|u| = \sigma(u) + \sigma (-u)$. Then, $|\frac{u+v}{2}|$ and $|\frac{u-v}{2}|$ become inputs for two identical $\widehat{f}_m$ respectively. We know that $\widehat{f}_m$ takes $m$ hidden layers (each with a width of $5$) to output $\widehat{f}_m(|\frac{u+v}{2}|) \approx |\frac{u+v}{2}|^2$ and $\widehat{f}_m(|\frac{u-v}{2}|)\approx |\frac{u-v}{2}|^2$, respectively. These outputs are merged together via 
\begin{equation*}
 \widehat{\Phi}(u,v) = \widehat{f}_m\left(\left|\frac{u+v}{2}\right|\right) - \widehat{f}_m\left(\left|\frac{u-v}{2}\right|\right).   
\end{equation*}

We prove, later in Proposition \ref{prop:productgate} (in Subsection \ref{productgate}), that a specific function $\Phi\in \mathcal{F}(m+1,(4,10,10,\ldots,10))$ with a specific $f_m$ can approximate the multiplication $u\cdot v$ to any arbitrary accuracy for $m$ sufficiently large.
\begin{figure}[h]
    \centering   \includegraphics[width=14cm]{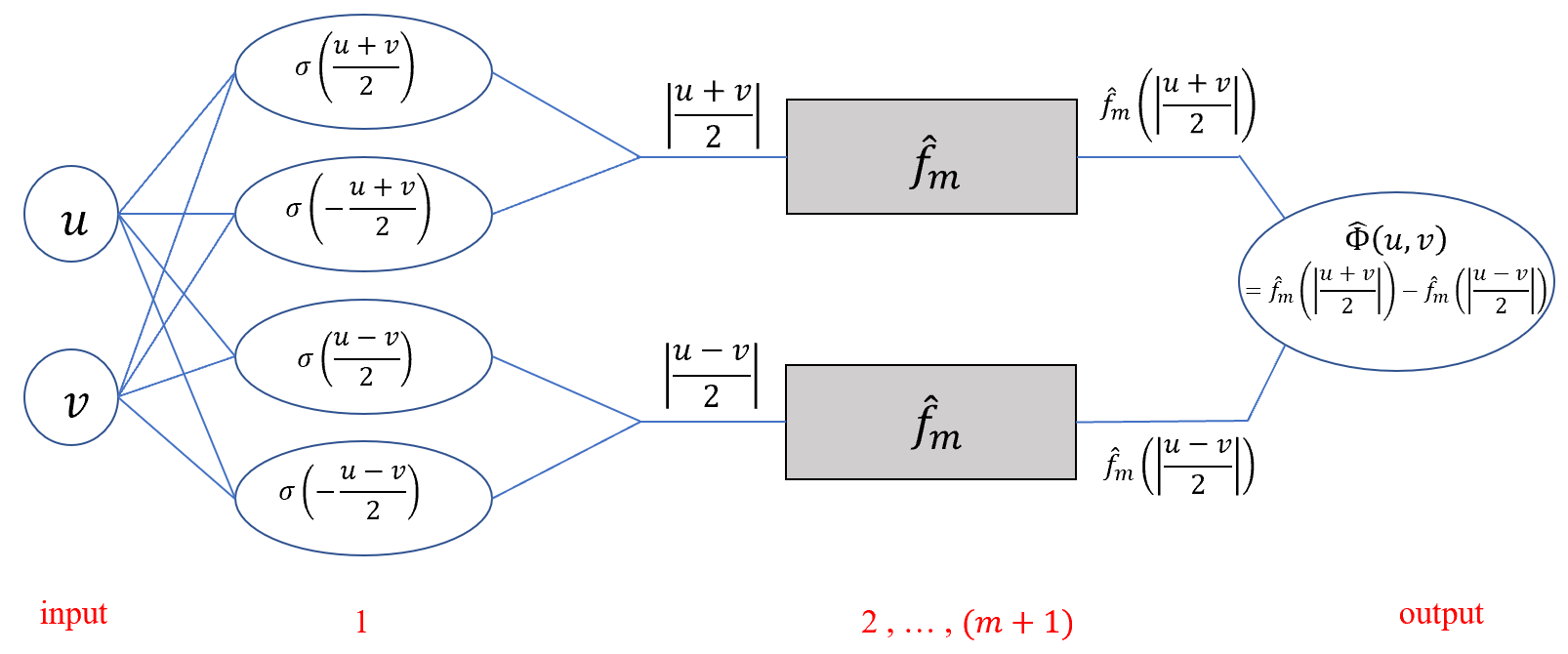}
    \caption{{\small An illustration of $\widehat{\Phi}$ for approximating $u\cdot v$ with inputs $(u,v)\in \RR^2$. $\widehat{\Phi}$ is a ReLU FNN with $m+1$ hidden layers. The first hidden layer generates $|\frac{u+v}{2}|$ and $|\frac{u-v}{2}|$, which become the inputs of two squaring gates $\widehat{f}_m$ respectively. The red numbers below the architecture indicate the order of hidden layers.}}
    \label{fig:Phi}
\end{figure}

Now we are in a position to introduce the monomial gate -- a ReLU network architecture that employs an expansive binary-tree structure. The aforementioned product gate $\widehat{\Phi}$ is an important building block of this network. The monomial gate is constructed to approximate  monomial functions of degree $k\in\NN$, that is $u^k$ with input $u \in [-1,1]$. This ReLU FNN, as illustrated in Figure \ref{fig:ebt}, belongs to $\mathcal{F}(\ell(m+1),\bm{p})$, where the width vector $\bm{p} \in \NN^{\ell(m+1)}$ is given by, for $j=0\ldots, \ell-1$,
\begin{equation*}
   \begin{cases} \nonumber
      p_{j(m+1)+1} = 5\left(2^j\right), \\
      p_{j(m+1)+i} = 11\left(2^j\right), & \text{for } i=2,\ldots,m+1.
    \end{cases}  
\end{equation*}

Notice that $\widehat\Phi(u,u) = \widehat f_m(|u|) - \widehat f_m(0)$.
Define the functions $\left\{\widehat h_k(u)\right\}_{k=1}^{2^\ell}$ on $\RR$ for $\ell\in \NN$  by  \begin{equation} \label{hk1}
 \left\{\widehat h_k(u)\right\}_{k=1}^2 = \left\{\widehat h_1(u)=u,\  \widehat h_2(u) =\widehat\Phi(u,u)=\widehat f_m(|u|)- \widehat f_m(0)\right\},   
\end{equation} and iteratively for $j=1,\ldots, \ell-1$, \begin{equation}\label{hk2}
 \widehat h_{2^j+i}(u) = \widehat \Phi\left(\widehat h_{2^j}(u), \widehat h_i(u)\right),\qquad i=1,\ldots, 2^j.   
\end{equation}

The key idea for constructing the monomial gate is to employ an expansive binary-tree structure. The network comprises $\ell$ subnetworks, each equipped with product gates $\widehat \Phi$. The 1st subnetwork takes the input $u$ and outputs  $\widehat h_2(u) =\widehat\Phi(u,u)$ and  $u$. The outputs of the 1st subnetwork become the inputs of the 2nd subnetwork, which outputs
\begin{equation*}
    \left\{u, \widehat h_2(u)\right\} \rightarrow \left\{u, \widehat h_2(u), \widehat h_3(u) = \widehat\Phi(h_2(u),u), \widehat h_4(u) = \widehat\Phi(h_2(u),h_2(u))\right\}.
\end{equation*}
The $k$-th subnetwork has $2^{k-1}$ product gates $\Phi$. This subnetwork takes in the outputs from the $(k-1)$-th subnetwork and computes
\begin{equation*}
    \left\{u,\widehat h_2(u), \ldots,\widehat h_{2^{k-1}}(u)\right\} \rightarrow \left\{u,\widehat h_2(u), \ldots,\widehat h_{2^{k-1}}(u), \ldots,\widehat h_{2^{k}}(u) \right\}.
\end{equation*}
\begin{figure}[h]
    \centering
    \includegraphics[width=17cm]{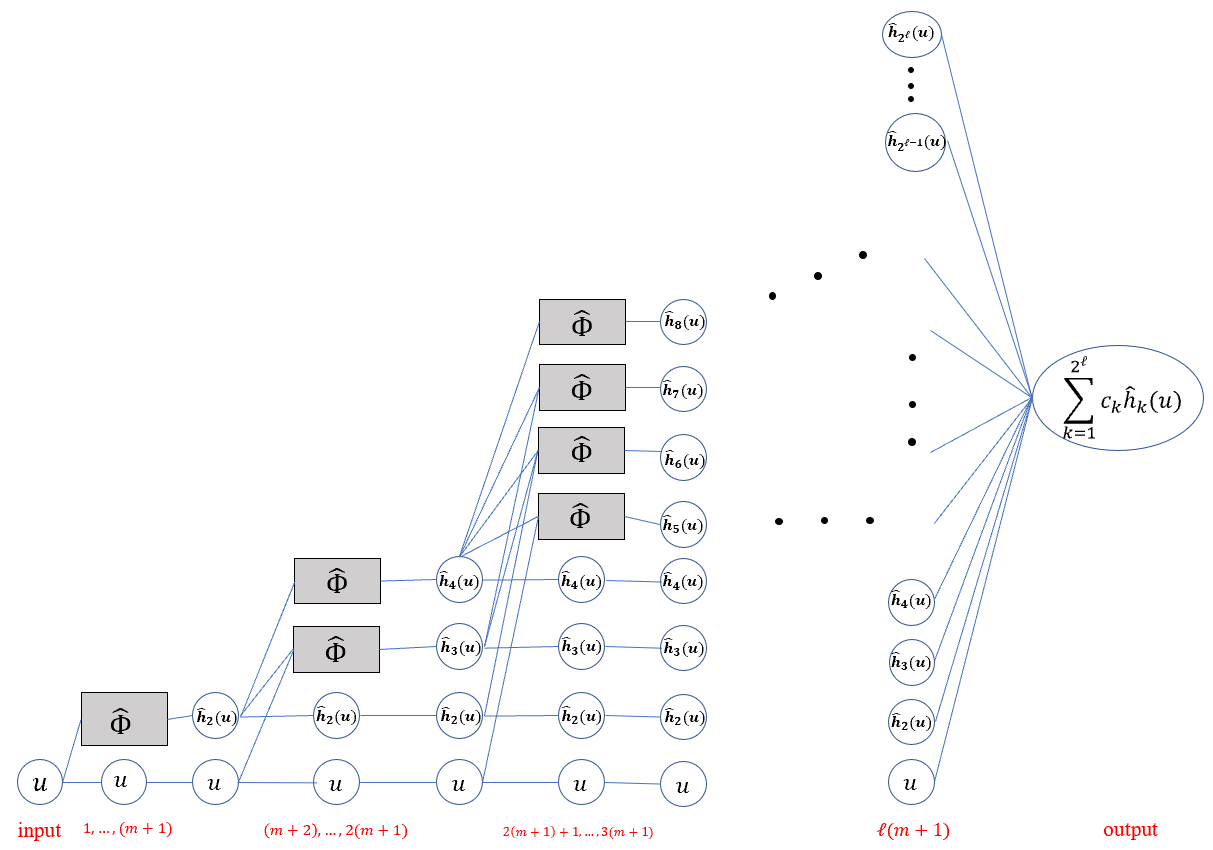}
    \caption{\small {An illustration of the monomial gate for approximating the set of monomials $\{u^k\}_{k=1}^{2^\ell}$ with input $u\in [-1,1]$. 
    The network has an expansive binary-tree structure. It comprises $\ell$ subnetworks equipped with the product gate $\widehat\Phi$. The width of the $k$-th subnetwork is doubled from that of the $ (k-1)$-th subnetwork for $k=1,\ldots,\ell$.}} 
    \label{fig:ebt}
\end{figure}

\textbf {In this way, we can see that the width of the $\bm k$-th subnetwork is doubled from that of the $\bm {(k-1)}$-th subnetwork for $\bm{k=1,\ldots,\ell}$, thereby forming the expansive binary-tree structure.}
The final output of the monomial gate (i.e., an EBTnet) is $\sum_{k=1}^{2^\ell}c_k \widehat h_k(u)$, which is a linear combination of $\widehat h_k(u)$ with some constant coefficients $c_k$ for  $k=1,\ldots,2^\ell$.

Later in Proposition \ref{monomialgate} (in Subsection \ref{subsec:monomial}), we prove that there exists a specific set of functions $\{h_k(u)\}_{k=1}^{2^\ell}$, implemented by this EBTnet with specific parameter choices, approximating the set of monomials $\{u_k\}_{k=1}^{2^\ell}$ up to any arbitrary accuracy, given that the depth of the network is sufficiently large. The approximation error bound is obtained by induction. 
\subsection{Preprocessing subnetwork: High Dimensional Truncation of Unbounded Data}
\label{section:preprocess}
 Notice that input data generated from GMM is unbounded. Later when we conduct a generalization analysis of ReLU network classifiers, we need to estimate the covering numbers of a set of output functions generated from the EBTnet. 
 However, the unboundedness of input data makes estimating covering numbers infeasible.
 We do not choose to make the assumption that data is bounded because Gaussian distributions are unbounded, and real-world features, such as images and speeches, are typically unbounded.
 
 To circumvent the challenges that come with unbounded data, we adopt a specific preprocessing subnetwork which is placed prior to the EBTnets. This preprocessing subnetwork contains a high-dimensional truncation unit.    
 This truncation unit, which will be defined shortly, manually projects  $d$-dimensional unbounded input to a bounded domain for $d\in \NN$.   

We first consider the case $d=1$. Let $b >0 $, define a univariate trapezoid-shaped function $T_b:\RR \rightarrow \RR$ by
{\allowdisplaybreaks
\begin{align} \label{truncation}
     T_b(u) &:= \sigma(u+b+1)-\sigma(u+b)-\sigma(u-b)+\sigma(u-b-1)   \\
     &= \begin{cases} \nonumber
      u+b+1, & \text{if } -b-1 \leq u < -b, \\
      1, & \text{if }  -b \leq u \leq b,\\
      -u+b+1, & \text{if }  b < u \leq b+1,\\
      0, & \text{if } u<-b-1\ \text{or } u> b+1.
    \end{cases}  
\end{align}}
With $u\in \RR$, $T_b$ can be implemented by a single-layer ReLU network without any free parameter. 

Now consider $d>1$. We extend the univariate trapezoid-shaped function to a higher dimensional space. With input $u = (u_1,u_2,\cdots,u_d)\in \RR^d$, define a  $d$-dimensional truncation  function $\Psi_b:\RR^d \rightarrow \RR$ by 
\begin{equation}\label{dtruncation}
   \Psi_b(u):= \sigma \left\{\sum_{i=1}^d T_b(u_i)-(d-1)\right\} .
\end{equation}
The truncation function $\Psi_b$ alone can be implemented by a ReLU FNN with $1$ hidden layer (of width $4d$) without any free parameter, as illustrated below: 
\begin{equation*}
    \begin{bmatrix}
    u_1\\
    \vdots\\
    u_d
    \end{bmatrix}
    \rightarrow
    \begin{bmatrix}
    \sigma(u_1+b+1)\\
    \sigma(u_1+b)\\
    \sigma(u_1-b)\\
    \sigma(u_1-b-1)\\
    \vdots\\
    \vdots\\
    \sigma(u_d+b+1)\\
    \sigma(u_d+b)\\
    \sigma(u_d-b)\\
    \sigma(u_d-b-1)
    \end{bmatrix}
    \rightarrow
    \Psi_b(u)= \sigma \left\{\sum_{i=1}^d T_b(u_i)-(d-1)\right\}.
\end{equation*}
A similar truncation network is given in  \citep{shaham2018provable}. 
The following Lemma presents the truncation property of $\Psi_b$. Its proof is given in Appendix \ref{app:prooflocalization}.

\begin{lemma} \label{localization}
Let $b>0$, and $\Psi_b$ be defined by (\ref{dtruncation}). With input $u\in \RR^d$, we have 
\begin{equation*}
 0 \leq \Psi_b(u)\leq 1   
\end{equation*} and 
\begin{equation*}
    \Psi_b(u) = \begin{cases} 
      0, & \text{if } u \notin [-b-1, b+1]^d, \\
      1, & \text{if }  u\in [-b,b]^d.
    \end{cases} 
\end{equation*}
\end{lemma}

Lemma \ref{localization} tells us that applying $\Psi_b$ to any data  $u \in \RR^d$ will project the data onto the interval $[0,1]$. Particularly, if $u$ lies outside the cube $[-b-1,b+1]^d$, $\Psi_b$ will truncate $u$ to 0 (i.e., makes $u$ vanishes). 

Now, we apply the  truncation function $\Psi_b$ to our data, which is $x\in \RR^d$ generated from some GMM. Recall that we assume our model has a total of $K$ Gaussian distributions, each with mean and covariance denoted by $\mu_j$ and $\Sigma_j$, for $j=1,\ldots, K$. Since $x$ is unbounded, it follows that  $$u_j(x)=(x-\mu_j)^T (\Sigma_j)^{-1}(x-\mu_j)/2, \qquad \forall j=1,\ldots,K,$$ are also unbounded. To learn the GMM discriminant function  $ D(x) = \sum_{j=1}^K \beta_j \exp(-u_j(x))
$ (previously defined in (\ref{target})) by ReLU FNNs, we need to input $u_j(x)$ into the EBTnet. But since $u_j(x)$ is unbounded, we first apply $\Psi_b$ to project $u_j(x)$ onto a bounded domain. 

Let $r_{i,j}(x)$ be the $i$-th component of the vector $(\Sigma_j)^{-1/2}\frac{(x-\mu_j)}{\sqrt{2}}$ for $i=1,\ldots,d$ and $j=1,\ldots,K$. 
They are affine functions of $x$ and thus can be implemented by a ReLU FNN via $r_{i,j}(x)=  \sigma(r_{i,j}(x))-\sigma(-r_{i,j}(x))$.  
We use a ReLU FNN belonging to $\mathcal{F}(2,(2,1))$ with the hypothesis space 
\begin{equation} \label{H0}
 \mathring{\mathcal{H}}= \left\{ \mathring W \cdot x + \mathring b: \mathring W \in \RR^d, \| \mathring W\| \leq C_K^0,  |\mathring b| \leq C_K^1 \right \}
\end{equation}
to learn $r_{i,j}(x)$ by $\widehat r_{i,j}(x) = \mathring W_{i,j} \cdot x + \mathring b_{i,j}$, where $C_K^0$ and $C_K^1$ are tunable parameters. 

Instead of directly applying $\Psi_b$ truncate to $\widehat r(x)$, we make use of $\Psi_b$ and the product gate $\widehat \Phi$ simultaneously. 
\textbf{More specifically, we construct a preprocessing subnetwork to compute $\bm {\widehat \Phi(\Psi_b(x),\widehat r_{i,j}(x))}$ with input $\bm {x\in\RR^d}$.}

Let us explain the purpose of computing $\widehat \Phi(\Psi_b(x),\widehat r_{i,j}(x))$ here.
From Lemma \ref{localization}, we know that
\begin{align}\label{boundtruncat}
\begin{cases}
    \widehat \Phi(\Psi_b(x), \widehat r_{i,j}(x)) = \widehat \Phi(0, \widehat r_{i,j}(x))= 0,  &\text{if } x \notin [-b-1, b+1]^d,   \\
    \widehat \Phi(\Psi_b(x), \widehat r_{i,j}(x))=\widehat \Phi(1, \widehat r_{i,j}(x)),  &\text{if }  x \in [-b, b]^d.
\end{cases}
\end{align}
for all $i=1,\ldots, d$ and  $j=1,\ldots,K$.

Recall that $\Phi(u,v)$ is a specific network architecture that is used to approximate the multiplication $u\cdot v$.  Equations (\ref{boundtruncat}) tells us that if $x$ lies in $[-b, b]^d$, the preprocessing subnetwork generates $\widehat \Phi(1,\widehat r_{i,j}(x))$, which is used to approximate $1 \cdot r_{i,j}(x) = r_{i,j}(x)$. On the other hand, if $x$ lies outside the $d$-dimensional cube $[-b-1, b+1]^d$, the preprocessing subnetwork outputs $0$. This is how the preprocessing network makes $\widehat r_{i,j}(x)$ vanishes when $x\in \RR^d$ is too large.

The following flowchart (\ref{flowchart}) illustrates the preprocessing subnetwork. It is equipped with one truncation unit $\Psi_b$ and $d\cdot K$ product gates $\widehat \Phi$. 
This subnetwork alone belongs to $\mathcal{F}(m+3, (4d+2dK, 1+dK, 4dK, 10dK, 10dK,\cdots, 10dK))$. 
The final outputs is the collection of  $\widehat \Phi(\Psi_b(x), \widehat r_{i,j}(x))$  for $i=1,\ldots, d,$ and $ j=1,\ldots,K$.
\begin{equation}\label{flowchart}
    \begin{bmatrix}
    x_1\\
    \vdots\\
    x_d
    \end{bmatrix}
    \rightarrow \cdot
\rightarrow 
\begin{bmatrix}
\Psi_b(x)\\
\widehat r_{1,1}(x)\\
\vdots\\
\widehat r_{d,1}(x)
\\
\vdots\\
\widehat r_{1,K}(x)
\\
\vdots\\
\widehat r_{d,K}(x)
\end{bmatrix}
\rightarrow \cdots \rightarrow 
\begin{bmatrix}
\widehat \Phi(\Psi_b(x),\widehat r_{1,1}(x))\\
\vdots\\
\widehat \Phi(\Psi_b(x),\widehat r_{d,K}(x))
\end{bmatrix}
\end{equation} 

Next, we will define our final ReLU network architecture. In the final network architecture, the preprocessing subnetwork is placed at the beginning. 

\subsection{Defining the Hypothesis Space}
\label{section:hypothesis}
Here, we define our final neural network architecture and the hypothesis space $\mathcal{H}$ for classification. Recall the Bayes classifier $f_c$ given by 
$$f_c(x) = \text{sgn}(D(x)) = \text{sgn}\left( \sum_{j=1}^K \beta_j \exp(-u_j(x))\right),$$
where $\beta_j$ are constant coefficients given by the model, and $u_j(x)=(x-\mu_j)^T (\Sigma_j)^{-1}(x-\mu_j)/2$.

The function space $\mathcal{H}$ consists of  functions implementable by a ReLU FNN and closely approximate $f_c$ for a given $x\in \mathcal{X}$. 
Before we introduce $\mathcal{H}$, we first define the function $\sigma_\lambda:\RR \rightarrow [-1,1]$ for some $0<\lambda\leq 1$ to be the linear combination of four scaled ReLU units given by  
\begin{align} \label{sigmalambda}
    \sigma_\lambda (u) &:= \sigma\left(\frac{u}{\lambda}\right) - \sigma\left(\frac{u}{\lambda} -1\right) - \sigma\left(-\frac{u}{\lambda}\right)+\sigma\left(-\frac{u}{\lambda}+1\right) 
    =
    \begin{cases} 
      1, & \text{if } u \geq \lambda, \\
      \frac{u}{\lambda}, & \text{if }  u\in [-\lambda,\lambda),\\
      -1, & \text{if } u < -\lambda.
    \end{cases} 
\end{align}
 We can see that if  $\lambda$ is close to $0$, $\sigma_\lambda (u)$ is close to $\text{sgn}(u)$. In other words, we use the function $\sigma_\lambda$ to approximate the sign function. 

Recall that $\mu^*= \max_{1\leq j\leq K} \|\mu_j\|$.
Here, we give the definition of our hypothesis space $\mathcal{H}$. 
We will show, later in Theorem \ref{main3}, that functions in $\mathcal{H}$ can indeed learn $f_c$ with a fast learning rate (i.e., the excess risk converges to $0$ fast). 
\begin{definition}[Hypothesis Space $\mathcal{H}$] \label{hypothesisspace}
Let $0 <\lambda \leq 1$, $b\geq 1$ and $m,\ell \in \NN$. 
Also let $R_2 \geq 1+ \frac{1}{\sqrt{\tilde{\sigma}}} (\sqrt{d}(b+1)+\mu^*)$.
 Define $h^*_{k,j}:\RR^d \rightarrow \RR$ by
\begin{equation}
    h^*_{k,j}(x) = \widehat h_{k}\left(\frac{1}{d}\sum_{i=1}^d \widehat \Phi \left(\widehat \Phi(\Psi_b(x),\widehat r_{i,j}(x)), \widehat \Phi(\Psi_b(x),\widehat r_{i,j}(x))\right)\right).
\end{equation}
With input $x \in \mathcal{X}=\RR^d$, the hypothesis space $\mathcal{H}$ is given by 
\begin{align*}
    \mathcal{H}&= \Biggl\{\sigma_\lambda \left(\sum_{j=1}^{K}\sum_{k=1}^{2^\ell}c_{k,j} h^*_{k,j}(x)  + c_0 \right): |c_{k,j}|,|c_0| \leq  C_K, \bm{W}= \{W_i\}_{i=1}^m \in [-4,4]^{5 \times 5},\\ 
    & \qquad \bm{b}=\{b_i\}_{i=1}^m \in [-4,4]^5, a\in [-4,4]^5, \widehat r_{i,j} \in \mathring{\mathcal{H}}  \Biggl\}.    
\end{align*}
Here, $C_K$ is a positive tunable parameter equal to or greater than $C_{R_2}=\sum_{j=1}^K |\beta_j| \exp \left((R_2)^2d\right)$.
\end{definition}

\begin{remark}\label{remark}
The hypothesis space $\mathcal{H}$ consists of functions implementable by ReLU FNNs $\in \mathcal{F}(L, \bm{p})$ where $L= m+3+ (\ell+1)(m+1)+1$ and the width vector $\bm{p}\in \NN^L$ given by \begin{equation*}
   \begin{cases} \nonumber
   p_1 = 4d+2dK, p_2 = 1+dK ,\\
   p_3 = p_{m+4} = 4dK,\\
   p_{i +s(m+1)} = 10dK, &  \text{if } s=0,1, i=4,\ldots, m+3, \\
      p_{(j+2)(m+1)+4} = 5K\left(2^j\right), &\text{if } j=0, \ldots, \ell-1,\\
      p_{(j+2)(m+1)+i} = 11K\left(2^j\right), & \text{if } j=0, \ldots, \ell-1, i=5,\ldots,m+4,\\
      p_L = 4.
    \end{cases}       
    \end{equation*} 
The beginning of the network is a preprocessing subnetwork that outputs $\widehat \Phi(\Psi_b(x), \widehat r_{i,j}(x))$  for $i=1,\ldots, d,$ and $ j=1,\ldots, K$. 
Each of these outputs is used to approximate $r_{i,j}(x)$. Next, each output from the preprocessing subnetwork enters one product gate $\widehat \Phi$ (i.e., we have $K$ product gates $\widehat \Phi$ followed by the preprocessing subnetwork). This group of product gates together outputs $\frac{1}{d}\sum_{i=1}^d \widehat \Phi \left(\widehat \Phi(\Psi_b(x),\widehat r_{i,j}(x)), \widehat \Phi(\Psi_b(x),\widehat r_{i,j}(x))\right)$ for $j=1,\ldots,K$, which is used to approximate $ \frac{1}{d}\sum_{i=1}^d(r_{i,j}(x))^2= \frac{1}{d}u_j(x)$. 
Then, each of these products enters a monomial gate (EBTnet), and all outputs of the $K$ monomial gates are merged together via $\sum_{j=1}^{K}\sum_{k=1}^{2^\ell}c_{k,j} h^*_{k,j}(x) +c_0$. 
If $\ell$ is large, this output can be used to approximate a linear combination of $\{(u_j(x))^k\}_{k=0}^\infty$, and thereby approximate $D(x)=\sum_{j=1}^{K}\beta_j\exp(-u_j(x))$ with suitable choices of coefficients $c_{k,j}$ and $c_0$. The last layer of the network is a scaling unit $\sigma_\lambda$, which is close to the sign function if $\lambda$ is close to $0$. We use the output function $f(x)\in \mathcal{H}$ to learn our target Bayes classifier $f_c(x) = \text{sgn}(D(x))$.
\end{remark}


\section{Main Results II: Generalization Analysis of ReLU Network Classifier} \label{sec: generalization}

Consider the hypothesis space $\mathcal{H}$ defined above in Definition \ref{hypothesisspace}.
For any function $f$ in $\mathcal{H}$, the misclassification error w.r.t. the probability measure $\rho$ is defined as 
\begin{equation}
    R(f) := \mathrm{E}[\mathbbm{1}\{Y\cdot \text{sgn} (f(X))=-1\}]=
    \mathrm{P}(Y\cdot \text{sgn} (f(X))=-1), \qquad \forall f\in \mathcal{H}
\end{equation}

To show functions in $\mathcal{H}$ can learn $f_c$ sufficiently well, we aim to find a $f\in \mathcal{H}$ that minimizes the excess risk $R(f)- R(f_c) \geq 0$. In practice, the probability measures $\rho$ are usually unknown. The classifier $f$ will be learned based on a random sample $z:= \{(x_i, y_i)\}_{i=1}^n$ drawn independently and identically distributed from $\rho$, where $n$ is the sample size. We find a classifier that minimizes the empirical risk:
\begin{equation*}
    \hat{f}_n:= \arg \ \min_{f\in \mathcal{H}}\ \frac{1}{n}\sum_{i=1}^n \mathbbm{1}\{y_i \cdot \text{sgn}(f(x_i))=-1\}.
\end{equation*}

However, the empirical risk w.r.t. $0$-$1$ loss -- the natural loss function for binary classification -- is non-continuous and non-convex \citep{bartlett2006convexity}.
Instead of minimizing the $0$-$1$ loss, we can adopt some convex loss function $V: \mathcal{X} \times \mathcal{Y} \rightarrow [0,\infty)$ to make computation feasible.  The Hinge loss, defined as  $\phi(t):= \max \{0, 1-t\}$, is one of the most commonly used loss functions in maximum-margin classifications, most notably the support vector machine \citep{rosasco2004loss}. 
The generalization error associated with the Hinge loss $\phi$ for $f$ is defined by
\begin{align*}
    \varepsilon(f) = \int_Z \phi(yf(x))d\rho &= \int_{\mathcal{X}} \int_{\mathcal{Y}} \phi(yf(x)) d\rho(y|x) d\rho_X   \\ 
    &= \int_{\mathcal{X} } \left(\phi(f(x))\mathrm P(y=1|x) +  \phi(-f(x))\mathrm P(y=-1|x) \right)d\rho_X .
\end{align*}
Given a sample $z=\{(x_i, y_i)\}_{i=1}^n$, define the empirical risk of $f$ w.r.t. $\phi$ over $z$ as 
\begin{equation}
    \varepsilon_z (f):=  \frac{1}{n} \sum_{i=1}^n \phi(y_if(x_i)). 
\end{equation}

Our goal is to find a classifier $f \in \mathcal{H}$ that minimizes the empirical risk w.r.t. $\phi$, that is, the empirical risk minimizer (ERM) defined as
\begin{equation} \label{ERM}
    f_z:=  \arg \min_{f\in \mathcal{H}} \varepsilon_z (f) 
\end{equation}

The well-known Comparison Theorem in classification in \citep{zhang2004statistical} suggests that, for the Hinge loss $\phi$ and any measurable function $f:\mathcal{X}\rightarrow \RR$,
\begin{equation}
    R(\text{sgn}(f))-R(f_c) \leq \varepsilon(f) - \varepsilon(f_c).    
\end{equation}
In other words, we can minimize the excess generalization error (also known as the excess $\phi$-error) $\varepsilon(f_z) - \varepsilon(f_c)$ to, in turn, bound the excess risk  $R(\text{sgn}(f_z))-R(f_c)$. To derive the convergence rate of the excess generalization error, we assume the Tsybakov-type noise condition given earlier in Assumption \ref{assumption1}. As a recap, Assumption \ref{assumption1} asserts that for some $c_0>0$ and $q\in [0,\infty)$, there holds $$\mathrm{P}(\{X\in \mathcal{X} : |D(X)| \leq t\}) \leq c_0t^q,\qquad \forall t>0,$$ 
where $q$ is often referred to as the noise exponent.


Our third result derives the convergence rate of the excess generalization error $\varepsilon(f_z) - \varepsilon(f_c)$, which in turn gives the convergence rate of the excess risk. 
In other words, the following theorem establishes the learning rate of ReLU networks on learning the Bayes classifier $f_c$ of GMM data. 
To our best knowledge, this is the first generalization error bound for classifications under a general GMM setting without constraints on model parameters or the number of Gaussian components. 
In particular, we do not impose any sparsity condition on the covariance matrices. 

\begin{theorem}\label{main3}
Let $n \geq 3, q>0, b \geq 1, C_K^0 \geq \frac{\sqrt{d}}{\sqrt{\tilde\sigma}}, C_K^1 \geq \frac{\mu^*}{\sqrt{\tilde\sigma}},R_2 = 2\sqrt{d}(C_K^0 + C_K^1)b, C_K=c_2 \exp(c_3 b^2)$ with
$c_2 \geq \sum_{j=1}^K |\beta_j|, c_3 \geq 4d^2(C_K^0 + C_K^1))^2$. Let $f_z$ be the empirical risk minimizer of the hypothesis space $\mathcal{H}$ with the chosen $C_K$.
Suppose the noise condition (\ref{Tsybakov}) holds for some noise exponent $q$ and constant $c_0>0$. Take $\lambda = n^{-\frac{1}{q}}, m= \ell(2^{\ell-1}), b = \sqrt{c_q^\prime \ell (2^{\ell-1})}$ with $c_q^\prime = \frac{q(\log 2)}{\frac{1}{16\sigma^*} + c_3q}$ and $\ell$ to be the smallest integer satisfying  $\ell 2^{\ell-1} \geq \frac{16\sigma^*}{c_q^\prime} \log n $. 
For any $0 < \delta<1$, with probability $1-\delta$, there holds,
\begin{equation}
    R(\text{sgn} (f_z))-R(f_c) \leq \varepsilon(f_z) - \varepsilon(f_c) \leq C_{q,d}\log \left(\frac{2}{\delta}\right)(\log n)^4 \left(\frac{1}{n}\right)^{\frac{q+1}{q+2}},
\end{equation}
where $C_{q,d}$ is a positive constant independent of $n$ or $\delta$.
\end{theorem}
We can see that the excess risk bound depends on $q$, the noise exponent. Theorem \ref{main3} tells us that when $q=0$ (no noise assumption), the convergence rate of the excess risk is of $\mathcal{O}\left((1/n)^{\frac{1}{2}} (\log n)^4 \right)$. 
When $q$ increases (more noise), the convergence rate of the excess risk approaches $\mathcal{O}\left((1/n)(\log n)^4 \right)$.

In the classic literature, convergence rates of the excess risk of order $\mathcal{O}((1/n)^{\frac{1}{2}})$ with $0$-$1$ loss were established \citep{mammen1999smooth}. 
Moreover, it is proven in \citep{tsybakov2004optimal} that, when  the ERM is taken over all measurable classifiers, the minimax lower bound of the excess risk is $\mathcal{O}\left(n^{-\frac{\alpha(q+1)}{\alpha(q+2)+(d-1)q}}\right)$, under Tsybakov's noise condition and when the decision boundary is $\alpha$-Hölder smooth. 

Recall Table \ref{table:1} given earlier in Section \ref{sec:intro}. 
This table compares our finding in Theorem \ref{main3} with the existing results on excess risk in the literature. 
Comparatively, our convergence rates do not depend on the dimension $d$. This demonstrates that the deep ReLU networks overcome the curse of dimensionality in classifications. 
More importantly, our result does not depend on any smoothness or regularity measure, whereas the existing results from \citep{feng2021generalization, kim2021fast, shen2022approximation} contain constant terms depending on smoothness index. For example, the excess risk estimate given in \citep[Theorem 2]{feng2021generalization} contains a constant term that increases exponentially with the smoothness index $r > 0$. Since the regression function under GMM is infinitely differentiable, existing results increase to infinity when we take the smoothness index to be infinity.

\section{Proof of Theorem \ref{main1}: Approximation of GMM Discriminant Function by ReLU Network} \label{proofmain1}
In this section, we present the proof of Theorem \ref{main1}.
The proof of Theorem \ref{main1} can be divided into three steps. First, we prove that the new product gate can approximate multiplication well (Subsection \ref{productgate}). Then, using the new product gate, we study how ReLU networks can approximate monomial functions (Subsection \ref{subsec:monomial}). After that, we prove Theorem \ref{main1} by showing how the GMM discriminant function $D$ can be well approximated by ReLU networks (Subsection \ref{subsec:proofmain1}). 

\subsection{A Novel Product gate} \label{productgate}

From \citep{yarotsky2017error}, it is shown that a $m$-layer ReLU FNN can approximate $f(u)=u^2$ to an accuracy $\frac{1}{4^{m+1}}$ with input $u\in [0,1]$. Inspired by Yarotsky's results, we propose a new product gate that achieves the same approximation accuracy on $[0,1]^2$ and has a linear increment on $\RR^2$ while adopting a structure that enables efficient approximation of monomials later. 

Recall the hat function $g: [0,\infty) \rightarrow [0,1]$ we defined earlier in (\ref{gdef}) by extending the construction by \citep{yarotsky2017error} on the interval $[0,1]$ to $[0,\infty)$.
The function $g$ can be regarded as the output of a ReLU network with $1$ hidden layer and $3$ neurons.  The $s$-composition $g_s$ of $g$ with itself on $[0,\infty)$ takes the form 
\begin{align*}
g_s(u) := \underbrace{g \circ g \circ \cdots g(u)}_{s \hbox{ folds}}
=\left\{
\begin{array}{ll}
2^s\left(u-\frac{2k}{2^s}\right),& \hbox{if} \ u \in \left[\frac{2k}{2^s}, \frac{2k+1}{2^s}\right], k=0,1,\ldots, 2^{s-1} -1, \\
2^s\left(\frac{2k}{2^s}-u\right),& \hbox{if} \ u \in \left[\frac{2k-1}{2^s}, \frac{2k}{2^s}\right], k=1,2,\ldots, 2^{s-1},\\
0,&  \hbox{if} \ u>1.
\end{array}\right.    
\end{align*}
With input $u\geq 0$, we can generate $g_s$ for $s=1,2,\ldots,m$ by a ReLU network of $m$ layers. Denote a function vector $\bm{\sigma}$ by 
\begin{equation}
\bm{\sigma}(u) =\begin{bmatrix}
\sigma(u-1)\\
\sigma(u-1/2)\\
 \sigma(u)
\end{bmatrix}.    
\end{equation} 
By $g_s$ and $\sigma(u) =u$ for $u\geq 0$, the following flow chart illustrates how the functions $\{u, g_1(u), g_2(u), \ldots, g_m (u)\}$ are produced by a ReLU FNN with $m$ hidden layers:
\begin{equation}\label{fm}
u \rightarrow   \begin{bmatrix}
\bm{\sigma}(u)
\end{bmatrix}  \rightarrow 
\begin{bmatrix}
\bm{\sigma}\left(g_1(u)\right)\\
u
\end{bmatrix} \rightarrow 
\begin{bmatrix}
\bm{\sigma}\left(g_2(u)\right)\\
g_1(u)\\
u
\end{bmatrix} \rightarrow \cdots \rightarrow 
\begin{bmatrix}
\bm{\sigma}\left(g_{m-1}(u)\right)\\
\sum_{s=1}^{m-2}\frac{g_s(u)}{2^{2s}}\\
u
\end{bmatrix} 
=: G_m(u)  \rightarrow f_m(u),
\end{equation}
where $f_m:[0,\infty) \rightarrow [0,\infty)$ is defined as  a linear combination of $\{g_s\}_{s=1}^m$ given by \begin{equation*}
f_m(u)=u-\sum_{s=1}^{m}\left(f_{s-1}(u)-f_s(u)\right)\\
=\left\{
\begin{array}{ll}
u-\sum_{s=1}^{m}\frac{g_s(u)}{2^{2s}},&\quad  \hbox{if} \ 0\le u\le 1, \\
u,&\quad  \hbox{if} \  u>1.
\end{array}\right.
\end{equation*}
As observed in \citep{yarotsky2017error}, on the interval $[0,1]$, $f_m$ is the piecewise linear interpolation of $u^2$
on $2^m$ subintervals on $[0, 1]$ with breakpoints $\{0,\frac{1}{2^m},...,\frac{2^m}{2^m}=1\}$. 
From (\ref{fm}), we see that $f_m$ can be implemented by a ReLU FNN with $m$ hidden layers each of width $5$. 
All the parameters take values on $[-4,4]$. 

Now, motivated by the identity $u\cdot v = \left|\frac{u+v}{2}\right|^2 - \left|\frac{u-v}{2}\right|^2$, we introduce a novel product gate $\Phi: \RR \times \RR \to \RR$ as \begin{equation} \label{Phi}
    \Phi (u,v) = f_m\left(\left|\frac{u+v}{2}\right|\right) - f_m\left(\left|\frac{u-v}{2}\right|\right), \qquad \forall u,v \in \RR. 
\end{equation}
It follows that the multiplication function $u \cdot v$ can be well-approximated by a ReLU network, as stated in the proposition below.

\begin{proposition} \label{prop:productgate}
Let $m\in \NN$. With the input $(u,v) \in \RR^2$, the function $\Phi: \RR^2 \to \RR$ can be implemented by a ReLU FNN $\in \mathcal{F}(m+1,(4,10,10,\ldots,10))$ with all the parameters take values on $[-4,4]$ such that
\begin{enumerate}
\item $\Phi (u,v) = 0$ if $u=0$ or $v=0$;
\item $|\Phi (u,v)| \leq |u| + |v|$;

\item If $u,v \in [-1,1]$, $\Phi (u,v) \in [-1,1]$ and it achieves an approximation accuracy \begin{equation*}
|\Phi (u,v) - u \cdot v| \leq 4^{-(m+1)}.    
\end{equation*} 
\end{enumerate}
\end{proposition}

\begin{proof}
Note that $|t|= \sigma(t) + \sigma (-t)$ for all $t\in\RR$. Hence $\sigma\left(\frac{u+v}{2}\right)+\sigma\left(-\frac{u+v}{2}\right)$ gives  $\left|\frac{u+v}{2}\right|$, which will be the input  of a subnetwork $f_m$. The following flow chart demonstrates how the product gate $\Phi$ is generated: 
\begin{equation*}
    \begin{bmatrix}
    u\\v
    \end{bmatrix}\rightarrow
    \begin{bmatrix}
    \sigma\left(\frac{u+v}{2}\right) \\
    \sigma\left(-\frac{u+v}{2}\right) \\
    \sigma\left(\frac{u-v}{2}\right) \\
    \sigma\left(-\frac{u-v}{2}\right) 
    \end{bmatrix} \rightarrow
    \begin{bmatrix}
    {\bm \sigma} \left(\left|\frac{u+v}{2}\right|\right)\\
   {\bm \sigma} \left(\left|\frac{u-v}{2}\right|\right)
    \end{bmatrix}
    \begin{matrix}
  \rightarrow \cdots  \rightarrow\\
    \rightarrow \cdots  \rightarrow
    \end{matrix}
    \begin{bmatrix}
    G_m( \left|\frac{u+v}{2}\right|)  \\
     G_m( \left|\frac{u-v}{2}\right|)
    \end{bmatrix}
  \rightarrow
    \Phi (u,v).
\end{equation*}
The complexity of the network $\Phi$ follows from that of $f_m$.

If $u=0$ or $v=0$, we have $\left|\frac{u+v}{2}\right|= \left|\frac{u-v}{2}\right|$ and thereby $\Phi(u, v)=0$. Observe that $0 \leq f_m(u) \leq u$ for $u \in [0, \infty)$, if follows that for $u,v\in \RR$,
\begin{align*}
|\Phi(u,v)|=\left|f_m\left(\left|\frac{u+v}{2}\right|\right)- f_m\left(\left|\frac{u-v}{2}\right|\right) \right|
 \leq \left|\frac{u+v}{2}\right| + \left|\frac{u-v}{2}\right|
 \leq |u|+|v|.
\end{align*}
It is shown in \citep[Proposition 2]{yarotsky2017error} that $0\leq f_m(u)-u^2 \leq 4^{-(m+1)}$ and $0 \leq f_m(u)\leq 1$ for $u\in [0,1]$. We have, for $u,v \in [-1,1]$, $-1 \leq \Phi (u,v)\leq 1$ and 
\begin{align*}
 \Phi (u,v) - u \cdot v 
& = \left\{f_m\left( \left|\frac{u+v}{2}\right|\right) -  \left|\frac{u+v}{2}\right|^2\right\}
- \left\{f_m\left( \left|\frac{u-v}{2}\right|\right) -  \left|\frac{u-v}{2}\right|^2\right\}\\
&\in \left[-4^{-(m+1)}, 4^{-(m+1)}\right] .  
\end{align*}
This proves the proposition. 
\end{proof}
 One of the advantages of our product gate is that for input on the domain $[0,1]$, the output is also on $[0,1]$. Such consistency of the domain and range $[0,1]$ helps us to define the monomial gate in an elegant way in the next subsection.

\subsection{Monomial Gate} 
\label{subsec:monomial}
Following the idea of the product gate introduced in the previous subsection, we construct a network  to approximate monomial functions $u^k$ on $[-1,1]$, with $k\in \NN$. 
This network is a monomial gate (i.e., EBTnet) defined in subsection \ref{section:EBTnet} with specific parameter choices.
Define the functions $\{h_k(u)\}_{k=1}^{2^\ell}$ on $\RR$ for $\ell\in \ZZ_+$  by  \begin{equation} \label{hkk1}
 \{h_k(u)\}_{k=1}^2 = \{h_1(u)=u,\  h_2(u) =\Phi(u,u)\},   
\end{equation} 
and iteratively for $j=1,\ldots, \ell-1$, \begin{equation} \label{hkk2}
 h_{2^j+k}(u) = \Phi(h_{2^j}(u), h_k(u)),\qquad k=1,\ldots, 2^j.   
\end{equation}
Here, we focus on the input domain $[-1,1]$.
\begin{proposition} \label{monomialgate}
 Let $\ell \in \ZZ_+$. Consider the functions $\{h_k(u)\}_{k=1}^{2^\ell}$ defined in (\ref{hkk1}) and (\ref{hkk2}). For input $u \in [-1,1]$, there exists a ReLU FNN that outputs the set of functions $\{h_k(u)\}_{k=1}^{2^\ell}$ such that for $j=0,1,\ldots, \ell-1$ and $k=1,\ldots, 2^j$, 
 \begin{equation}
     h_{2^{j}+k}(u) \in [-1,1]
 \end{equation}
 and \begin{equation}\label{boundofhk}
 \left|h_{2^{j}+k}(u) - u^{2^{j}+k}\right| \leq \frac{2^{j+1}-1}{4^{m+1}}.
\end{equation} 
This ReLU FNN belongs to $\mathcal{F}(\ell(m+1),\bm{p})$, where the width vector $\bm{p}$ is given by, for $j=0\ldots, \ell-1$,
\begin{equation*}
   \begin{cases} \nonumber
      p_{j(m+1)+1} = 5\left(2^j\right), \\
      p_{j(m+1)+i} = 11\left(2^j\right), & \text{for } i=2,\ldots,m+1.
    \end{cases}  
\end{equation*}All parameters take value in $[-4,4]$. 
\end{proposition}

\begin{proof}
 The network described here is an EBTnet defined in subsection \ref{section:EBTnet}. It consists of $\ell$ subnetworks, each with depth $m+1$, forming an expansive binary tree structure. We prove our statements by induction on $j$.
 
The case $j=0$ is obvious: 
the 1st subnetwork takes $u$ a sinput and outputs $\{h_2(u), h_1(u)\}$. For $u \in [-1,1]$, we have 
\begin{equation*}
    h_2(u) =\Phi(u,u) = f_m(|u|) \in [0,1]
\end{equation*} and 
\begin{equation*}
    |h_2(u)-u^2| =\left|f_m(|u|) - u^2\right| = f_m(|u|) - u^2 \leq 4^{-(m+1)} = \frac{2 -1}{4^{(m+1)}}.
\end{equation*}
Now, assume the statements are true for the $j$-th subnetwork, where $0\leq j \leq \ell-1$. The $(j+1)$-th subnetwork takes  $[h_i(u)]_{i=1}^{2^j}$ as inputs, and outputs $h_{2^{j}+1}(u), \ldots, h_{2^{j+1}}(u)$.

By the induction hypothesis, $h_{2^{j-1}+k}(u) \in [-1,1]$ and  $|h_{2^{j-1}+k}(u) - u^{2^{j-1}+k}| \leq \frac{2^j-1}{4^{m+1}}$ for $u\in [-1,1]$ and $k=1,\ldots, 2^{j-1}$. We obtain, for $u\in [-1,1]$ and $k=1,\ldots, 2^j$,
\begin{equation*}
    h_{2^{j}+k}(u) = \Phi(h_{2^{j}}(u),h_k(u)) \in [-1,1],
\end{equation*} and 
\begin{eqnarray*}
  &  &|h_{2^{j}+k}(u) - u^{2^{j}+k}| \\
    &=& | h_{2^{j}+k}(u)- h_{2^j}(u) \cdot h_{k}(u)+h_{2^j}(u) \cdot h_{k}(u)- u^{2^{j}+k}|\\
    &\leq& \left|\Phi(h_{2^j}(u), h_k(u))  - h_{2^j}(u) \cdot h_{k}(u) \right|+\left|\left(h_{2^j}(u) - u^{2^{j}}\right)h_k(u) +u^{2^{j}}(h_{k}(u)- u^k)\right|\\
    &\leq&  \frac{1}{4^{m+1}} + 2\cdot \frac{2^j-1}{4^{m+1}}
    = \frac{2^{j+1}-1}{4^{m+1}}.
\end{eqnarray*}
This completes the induction procedure and the proof of Proposition 
\ref{monomialgate}.
\end{proof}

\subsection{Proof of Theorem \ref{main1}}
\label{subsec:proofmain1}
 Recall the GMM discriminant function $D$ defined in (\ref{target}) by
 \begin{equation*}
    D(x) = D^+(x) - D^-(x) = \sum_{j=1}^K \beta_j \exp(-u_j(x)),        
 \end{equation*}
 where $\beta_j$ are constant coefficients, and $u_j(x)=(x-\mu_j)^T (\Sigma_j)^{-1}(x-\mu_j)/2$. 
 
Recall that $r_{i,j}(x)$ is the $i$-th component of the vector $(\Sigma_j)^{-1/2}\frac{(x-\mu_j)}{\sqrt{2}}$ for $i =1,\ldots,d$ and $j=1,\ldots,K$. 
Also, recall that
$\mu^*= \max_{1\leq j\leq K} \|\mu_j\|$
and $\tilde{\sigma}$ is the smallest eigenvalue of all the $K$ covariance matrices. We have for $j=1,\ldots,K$,
\begin{equation} \label{rij}
    |r_{i,j}(x)| \leq \|(\Sigma_{j})^{-1/2}(x-\mu_j)\| \leq \frac{1}{\sqrt{\tilde{\sigma}}}\|x-\mu_j\|
    \leq \frac{1}{\sqrt{\tilde{\sigma}}} (\|x\|+\mu^*).  
\end{equation}
Here, we present the proof of Theorem \ref{main1}.

\begin{proof}[Proof of Theorem \ref{main1}]
Since the input $x\in \RR^d$ is unbounded, $u_j(x)$ is unbounded for $j=1,\ldots,K$.
We  first apply the preprocessing subnetwork to truncate $x$. This part of the neural network is a fixed network structure (i.e., all network parameters are not free). 
The flowchart below  showcases the preprocessing unit of our neural network:
\begin{equation*}
    \begin{bmatrix}
    x_1\\
    \vdots\\
    x_d
    \end{bmatrix}
    \rightarrow,
\rightarrow 
\begin{bmatrix}
\Psi_b(x)\\
r_{1,1}(x)\\
\vdots\\
r_{d,1}(x)
\\
\vdots\\
r_{1,K}(x)
\\
\vdots\\
r_{d,K}(x)
\end{bmatrix}
\rightarrow \cdots \rightarrow 
\begin{bmatrix}
\Phi(\Psi_b(x),r_{1,1}(x))\\
\vdots\\
\Phi(\Psi_b(x),r_{d,K}(x))
\end{bmatrix}
\end{equation*} 
It follows from Proposition \ref{prop:productgate} that 
   $ |\Phi(\Psi_b(x), r_{i,j}(x))| \leq |\Psi_b(x)|+|r_{i,j}(x)|$.

Then we have, from Lemma \ref{localization}, for $j=1,\ldots,K$, 
\begin{align} \label{boundtruncate}
\begin{cases}
    \Phi(\Psi_b(x),r_{i,j}(x)) = \Phi(0,r_{i,j}(x))= 0,  &\text{if } x \notin [-b-1, b+1]^d,   \\
    \Phi(\Psi_b(x), r_{i,j}(x))=\Phi(1, r_{i,j}(x)) ,  &\text{if }  x \in [-b, b]^d.
\end{cases}
\end{align}

Recall the function space $\mathring{\mathcal{H}}$ defined earlier in (\ref{H0}). Every $\widehat r_{i,j}\in \mathring{\mathcal{H}}$ has the form $\widehat r_{i,j}(x)= \mathring W_{i,j} \cdot x + \mathring b_{i,j}$ with $\|\mathring W_{i,j}\| \leq C_K^0, |b_{i,j}| \leq C_K^1$. In Theorem \ref{main1}, we choose $C_K^0 \geq \frac{\sqrt{d}}{\sqrt{\tilde\sigma}}$ and $ C_K^1 \geq \frac{\mu^*}{\sqrt{\tilde\sigma}}$. We observe from (\ref{rij}) that  $r_{i,j} \in \mathring{\mathcal{H}}$ due to the choices of $C_K^0, C_K^1$ here. 
Also, $|r_{i,j}(x)| \leq \frac{1}{\sqrt{\tilde{\sigma}}} (\sqrt{d}(b+1)+\mu^*)$ for $x \in [-b-1, b+1]^d$.
But $R_2 \geq \sqrt{d}(b+1)C_K^0 + C_K^1$. Then for $j=1,\ldots, K$ and $i =1,\ldots,d$, $|r_{i,j}(x)/R_2| \leq 1$ and
\begin{equation} \label{R2bound}
 \left | \Phi\left(\Psi_b(x),\frac{r_{i,j}(x)}{R_2}\right)\right| \leq 1,\qquad \forall x \in [-b-1, b+1]^d
\end{equation}and thus for all $x \in \RR^d$ by (\ref{boundtruncate}).
Observe that 
\begin{align*}
    D(x) = \sum_{j=1}^K \beta_j \exp(-u_j(x)) &= \sum_{j=1}^K \beta_j \exp\left(-\sum_{i=1}^d(r_{i,j}(x))^2\right)\\
    &= \sum_{j=1}^K \beta_j \exp\left(-(R_2)^2\sum_{i=1}^d \left(\frac{r_{i,j}(x)}{R_2}\right)^2\right). 
\end{align*}
By the Taylor expansion of the exponential function $\exp{\left(-(R_2)^2u\right)}$ for $u\in \RR$, we can further write $D$ as
\begin{equation*}
    D(x) = \sum_{j=1}^K \beta_j \sum_{k=0}^\infty \frac{(-1)^k(R_2)^{2k}d^k}{k!}\left(\frac{1}{d}\sum_{i=1}^d \left(\frac{r_{i,j}(x)}{R_2}\right)^2\right)^k .    
\end{equation*}
Note that $\Phi (u,u)= f_m(|u|)$ by the definition of $\Phi$ in (\ref{Phi}).
Now define $\widetilde{D}:\RR^d \rightarrow \RR$ by
\begin{align} \label{tildeD}
    \widetilde{D}(x)
    = \sum_{j=1}^K \beta_j \left\{1 + \sum_{k=1}^{2^\ell} \frac{(-1)^k(R_2)^{2k}d^k}{k!} h_k\left(\frac{1}{d}\sum_{i=1}^d f_m\left(\left| \Phi\left(\Psi_b(x),\frac{r_{i,j}(x)}{R_2}\right)\right|\right)\right)\right\} .
\end{align}
For brevity, we wrote $f_m\left(\left| \Phi\left(\Psi_b(x),\frac{r_{i,j}(x)}{R_2}\right)\right|\right)$ instead of $\Phi \left(\Phi\left(\Psi_b(x),\frac{r_{i,j}(x)}{R_2}\right),\Phi\left(\Psi_b(x),\frac{r_{i,j}(x)}{R_2}\right)\right)$.
We would like to highlight that $\sigma_\lambda (\widetilde{D}) \in \mathcal{H}$. In other words, $\widetilde{D}$ can be implemented by a ReLU FNN described in Remark \ref{remark} excluding the last scaling layer $\sigma_\lambda$.

For $k = 1,\ldots, 2^\ell$ and $x\in \RR^d$, we have 
\begin{eqnarray*}
     &&h_k\left(\frac{1}{d}\sum_{i=1}^d f_m\left(\left| \Phi\left(\Psi_b(x),\frac{r_{i,j}(x)}{R_2}\right)\right|\right)\right)\\
    &=& h_k\left(\frac{1}{d}\sum_{i=1}^d f_m\left(\left| \Phi\left(\Psi_b(x),\frac{r_{i,j}(x)}{R_2}\right)\right|\right)\right) -\left(\frac{1}{d}\sum_{i=1}^d f_m\left(\left| \Phi\left(\Psi_b(x),\frac{r_{i,j}(x)}{R_2}\right)\right|\right)\right)^k \\
    &&+ \left(\frac{1}{d}\sum_{i=1}^d f_m\left(\left| \Phi\left(\Psi_b(x),\frac{r_{i,j}(x)}{R_2}\right)\right|\right)\right)^k .   
\end{eqnarray*}
Applying Proposition \ref{monomialgate}, by (\ref{R2bound}) and $0\leq f_m(u)\leq 1$ for $u\in [0,1]$, we get for $k = 1,\ldots, 2^\ell$ and $x\in \RR^d$,

\begin{equation} \label{hkbound}
  \left|h_k\left(\frac{1}{d}\sum_{i=1}^d f_m\left(\left| \Phi\left(\Psi_b(x),\frac{r_{i,j}(x)}{R_2}\right)\right|\right)\right)\right| \leq 1 .   \end{equation}
By (\ref{hkbound}), when $x\in [-b-1, b+1]^d$,
\begin{equation}
 |\widetilde{D}(x)|
 \leq \sum_{j=1}^K |\beta_j| \left\{1 + \sum_{k=1}^{2^\ell} \frac{((R_2)^2d)^k}{k!} \right\} \leq \sum_{j=1}^K |\beta_j| \exp{\left((R_2)^2d\right)}.
\end{equation}
When $x\notin [-b-1,b+1]^d$, by (\ref{boundtruncate}), we have $
   f_m\left(\left| \Phi\left(\Psi_b(x),\frac{r_{i,j}(x)}{R_2}\right)\right| \right)=0$ and thereby 
$|\widetilde{D}(x)| \leq \left|\sum_{j=1}^K \beta_j\right| \leq \sum_{j=1}^K |\beta_j|\left( \exp{\left((R_2)^2d\right)}\right)$. 
This proves (\ref{main1unbounded}) in Theorem \ref{main1}.

Now if $x\in [-b,b]^d$, we know $\Psi_b(x)=1$ from Lemma \ref{localization}. Applying the approximation error bound (\ref{Yarotskyprop2}) for $f_m$, we get
\begin{align*}
\left| f_m\left(\left| \Phi\left(\Psi_b(x),\frac{r_{i,j}(x)}{R_2}\right)\right| \right)-\left(\frac{r_{i,j}(x)}{R_2}\right)^2\right| 
=\left| f_m\left(\left| \Phi\left(1,\frac{r_{i,j}(x)}{R_2}\right)\right| \right)-\left(\frac{r_{i,j}(x)}{R_2}\right)^2\right|
\leq 4^{-(m+1)} . 
\end{align*}
It follows from Proposition \ref{monomialgate} that for $x\in [-b,b]^d$, $k = 1,\ldots, 2^\ell$,
{\allowdisplaybreaks
\begin{eqnarray*}
  && \left| h_k\left(\frac{1}{d}\sum_{i=1}^d  f_m\left(\left| \Phi\left(\Psi_b(x),\frac{r_{i,j}(x)}{R_2}\right)\right| \right)\right)- \left(\frac{1}{d}\sum_{i=1}^d \left(\frac{r_{i,j}(x)}{R_2}\right)^2\right)^k \right|\\
  &\leq& \frac{2^{\ell}-1}{4^{m+1}}
  +\left| \left(\frac{1}{d}\sum_{i=1}^d  f_m\left(\left| \Phi\left(\Psi_b(x),\frac{r_{i,j}(x)}{R_2}\right)\right| \right)\right)^k -\left(\frac{1}{d}\sum_{i=1}^d \left(\frac{r_{i,j}(x)}{R_2}\right)^2\right)^k \right|\\
  &\leq& \frac{2^{\ell}-1}{4^{m+1}} + \frac{k}{d}\sum_{i=1}^d\left|   f_m\left(\left| \Phi\left(\Psi_b(x),\frac{r_{i,j}(x)}{R_2}\right)\right| \right) - \left(\frac{r_{i,j}(x)}{R_2}\right)^2\right| \\
&\leq& \frac{2^{\ell}-1}{4^{m+1}} + k\left(4^{-(m+1)} \right)
= \frac{2^{\ell}+k-1}{4^{m+1}}.
\end{eqnarray*}}
Here, we have used the Mean Value Theorem to bound $|u^k-v^k|\leq k|u-v|$ for $u,v \in [0,1]$. 

Then, for $x\in [-b,b]^d$, we have
{\allowdisplaybreaks
\begin{eqnarray*}
    &&\left|\widetilde{D}(x)- D(x)\right|\\
    &\leq& \left|\widetilde{D}(x)-  \sum_{j=1}^K \beta_j \sum_{k=0}^{2^\ell} \frac{(-1)^k(R_2)^{2k}d^k}{k!}\left(\frac{1}{d}\sum_{i=1}^d \left(\frac{r_{i,j}(x)}{R_2}\right)^2\right)^k \right| \\
    &&+ \left|\sum_{j=1}^K \beta_j \sum_{k=2^\ell+1}^{\infty} \frac{(-1)^k(R_2)^{2k}d^k}{k!}\left(\frac{1}{d}\sum_{i=1}^d \left(\frac{r_{i,j}(x)}{R_2}\right)^2\right)^k\right|   \\
    &\leq&\sum_{j=1}^K |\beta_j| \sum_{k=0}^{2^\ell} \frac{(R_2)^{2k}d^k}{k!}\left|h_k\left(\frac{1}{d}\sum_{i=1}^d  f_m\left(\left| \Phi\left(\Psi_b(x),\frac{r_{i,j}(x)}{R_2}\right)\right| \right)\right) - \left(\frac{1}{d}\sum_{i=1}^d \left(\frac{r_{i,j}(x)}{R_2}\right)^2\right)^k \right|\\
    &&+ \left|\sum_{j=1}^K \beta_j \sum_{k=2^\ell+1}^{\infty} \frac{(-1)^k(R_2)^{2k}d^k}{k!}\left(\frac{1}{d}\sum_{i=1}^d \left(\frac{r_{i,j}(x)}{R_2}\right)^2\right)^k\right|\\
    &\leq& \sum_{j=1}^K |\beta_j|\left\{\sum_{k=1}^{2^\ell} \frac{((R_2)^2d)^k}{k!}\left(\frac{2^{\ell}+k-1}{4^{m+1}}\right)\right\} + \sum_{j=1}^K |\beta_j|\sum_{k=2^\ell+1}^{\infty} \frac{((R_2)^{2}d)^k}{k!} \\ 
    &\leq& \sum_{j=1}^K |\beta_j|\left\{ \sum_{k=1}^{2^\ell} \frac{((R_2)^2d)^k}{k!}\left(\frac{2^{\ell+1}}{4^{m+1}}\right)+  \sum_{k=2^\ell+1}^{\infty} \frac{((R_2)^{2}d)^k}{k!}\right\}.
\end{eqnarray*}}
By Stirling's formula,
\begin{equation*}
   \sqrt{2\pi k}\left(\frac{k}{e}\right)^k\exp{\left(\frac{1}{12k+1}\right)}< k! <\sqrt{2\pi k}\left(\frac{k}{e}\right)^k\exp{\left(\frac{1}{12k}\right)},
\end{equation*}
we know that
\begin{equation*}
   \sum_{k=2^\ell+1}^{\infty} \frac{((R_2)^{2}d)^k}{k!}< \sum_{k=2^\ell+1}^{\infty} \frac{((R_2)^{2}d)^k}{\sqrt{2\pi k}} \left(\frac{e}{k}\right)^k \exp{\left(-\frac{1}{12k+1}\right)} \leq \sum_{k=2^\ell+1}^{\infty} \left(\frac{(R_2)^{2}de}{k} \right)^k.
\end{equation*}
Since $\ell \geq \frac{2\log ((R_2)^2 de)}{\log 2}$, we have$ ((R_2)^{2}de)^2 \leq 2^\ell+1$. We then  apply the bound $2^\ell+1 \leq k$ for $k \geq 2^\ell + 1$ and find 
\begin{equation*}
\sum_{k=2^\ell+1}^{\infty} \left(\frac{(R_2)^{2}de}{k} \right)^k 
 \leq \sum_{k=2^\ell+1}^{\infty} \left(\frac{1}{\sqrt{k}} \right)^k 
 \leq  \left(\frac{1}{\sqrt{2^\ell}}\right)^{2^\ell}\sum_{k=2^\ell+1}^{\infty} \left(\frac{1}{\sqrt{2^\ell}} \right)^{k-2^\ell} \leq 2^{-\frac{\ell (2^\ell)}{2}},
\end{equation*}
where we have bounded  $k \geq 2^\ell + 1$ by $2^\ell$ from below and then $\sqrt{2^\ell}$ by $2$ from  below.

Therefore, for $x\in [-b,b]^d$, there holds 
\begin{align*}
 |\widetilde{D}(x)- D(x)|&\leq  \left(\sum_{j=1}^K |\beta_j|\right) \exp{\left((R_2)^2d\right)}\left(\frac{2^{\ell+1}}{4^{m+1}}+ 2^{-\frac{\ell (2^\ell)}{2}}\right),
\end{align*}
which verifies (\ref{main1bounded}) in Theorem \ref{main1}.
The proof of Theorem \ref{main1} is complete. \end{proof}

\section{Proof of Theorem 3: Generalization Analysis}\label{sec:proofmain3}
In this section, we derive the high probability upper bound of excess generalization error $\varepsilon(f_z) - \varepsilon(f_c)$ for proving Theorem \ref{main3}. 

To start, we decompose $\varepsilon(f_z) - \varepsilon(f_c)$ into estimation error terms and an approximation error term (Subsection \ref{subsec:decomposition}). 
Then, we bound the estimation error terms (in Subsection \ref{subsec:estimationerror}) that involve estimating the covering number of our hypothesis space $\mathcal{H}$. After that, we bound the approximation error term (Subsection \ref{subsec:approxerror}). Lastly, by combining all the error estimates together, we are able to derive the proof of Theorem \ref{main3} (Subsection \ref{subsec:approxerror}). 

\subsection{Error decomposition}
\label{subsec:decomposition}
We  consider the following error decomposition.
similar error decompositions can be found in \citep{tianzhoucnn, huang2022fast}.
\begin{lemma}[Decomposition of $\varepsilon(f_z) - \varepsilon(f_c)$]
Let $f_\mathcal{H}$ be any functions in $\mathcal{H}$ defined in Definition \ref{hypothesisspace}. There holds 
\begin{equation} \label{decomposition}
    \varepsilon(f_z) - \varepsilon(f_c) \leq \{\varepsilon(f_z)- \varepsilon_z(f_z)\}  +  \{\varepsilon_z(f_\mathcal{H}) - \varepsilon(f_\mathcal{H})\} + \{\varepsilon(f_\mathcal{H}) - \varepsilon(f_c)\}.
\end{equation}
\end{lemma}
\begin{proof} We express $\varepsilon(f_z) - \varepsilon(f_c)$ by inserting empirical risks as follows
\begin{equation*}
    \varepsilon(f_z) - \varepsilon(f_c) = \{\varepsilon(f_z)- \varepsilon_z(f_z)\} + \{\varepsilon_z(f_z) - \varepsilon_z(f_\mathcal{H})\} +  \{\varepsilon_z(f_\mathcal{H}) - \varepsilon(f_\mathcal{H})\} + \{\varepsilon(f_\mathcal{H}) - \varepsilon(f_c)\}.
\end{equation*}
Both $f_z$ and $f_\mathcal{H}$ lies on the hypothesis space $\mathcal{H}$. From the definition of $f_z$ at (\ref{ERM}), $f_z$ minimizes the empirical risk $\varepsilon_z(f)$ over $\mathcal{H}$. Thus we have $\varepsilon_z(f_z) - \varepsilon_z(f_\mathcal{H}) \leq 0$. This yields the expression (\ref{decomposition}).
\end{proof}
$\{\varepsilon(f_z)- \varepsilon_z(f_z)\}$ is the first estimation error (also known as the sample error) term, $\{\varepsilon_z(f_\mathcal{H}) - \varepsilon(f_\mathcal{H})\}$ is the second estimation error term, whereas $\{\varepsilon(f_\mathcal{H}) - \varepsilon(f_c)\}$ --- which does not depend on the data --- is the approximation error term induced by $f_\mathcal{H}$. To give an upper bound to the excess generalization error, we will proceed to bound these three error terms respectively. 

\subsection{Upper Bound of Estimation Errors}\label{subsec:estimationerror}
In this subsection, we derive an upper bound of the estimation errors  $\varepsilon(f_z)- \varepsilon_z(f_z) + \varepsilon_z(f_\mathcal{H}) - \varepsilon(f_\mathcal{H})$. 

We first rewrite it by inserting $\varepsilon(f_c)$ and $\varepsilon_z(f_c)$:
\begin{eqnarray}
 \varepsilon(f_z) -\varepsilon_z(f_z) + \varepsilon_z(f_\mathcal{H}) - \varepsilon(f_\mathcal{H})  &=&   \varepsilon(f_z) -\varepsilon(f_c)- (\varepsilon_z(f_z)-\varepsilon_z(f_c)) \label{estimationerr1} \\
 &&+ \varepsilon_z(f_\mathcal{H}) - \varepsilon_z(f_c) - (\varepsilon(f_\mathcal{H})- \varepsilon(f_c)).   \label{estimationerr2}    
\end{eqnarray}
In other words, to bound $\varepsilon(f_z)- \varepsilon_z(f_z) + (\varepsilon_z(f_\mathcal{H}) - \varepsilon(f_\mathcal{H}))$, we should bound the R.H.S. of (\ref{estimationerr1}) and the R.H.S. of (\ref{estimationerr2}) respectively. 

The following devotes to an upper bound of R.H.S. of (\ref{estimationerr1}).

\subsubsection{Upper bound of $\varepsilon(f_z) -\varepsilon(f_c)- (\varepsilon_z(f_z)-\varepsilon_z(f_c))$}

The expression (\ref{estimationerr1})
 $$\varepsilon(f_z) -\varepsilon(f_c)- (\varepsilon_z(f_z)-\varepsilon_z(f_c)) \leq \sup_{f\in \mathcal{H}} \{\varepsilon(f)- \varepsilon(f_c) - (\varepsilon_z(f)-\varepsilon_z(f_c))\}$$ can be estimated by the theory of uniform convergence. Since our domain $\mathcal{X}=\RR^d$ is unbounded, deriving covering number estimates for our hypothesis space $\mathcal{H}$ is difficult. 

Let \begin{equation}\label{B}
 B:=\sum_{i\in \mathcal{T}^+}\frac{P^+p_i}{\sqrt{(2\pi)^d|\Sigma_i|}}+ \sum_{j\in \mathcal{T}^-}\frac{P^-p_j}{\sqrt{(2\pi)^d|\Sigma_j|}}.  
\end{equation}
Recall the Tsybakov-type noise condition we stated in Assumption \ref{assumption1}. The following Lemma  presents an upper bound of the second moment and thereby the variance of $\phi(yf(x))-\phi(yf_c(x))$ for any function $f:\mathcal{X} \rightarrow [-1,1]$ under the noise condition (\ref{Tsybakov}). 

\begin{lemma} \label{secondmoment}
Let $0\leq q \leq \infty$. Also let $\phi(t) = \max \{0, 1-t\}$ to be the Hinge loss function. Consider the constant $B$ given by (\ref{B}). If noise condition (\ref{Tsybakov}) holds for some noise exponent $q$ and constant $c_0>0$, then for every function $f: \mathcal{X} \rightarrow [-1,1]$, there holds 
\begin{equation}
    \mathrm{E}\left[\{\phi(yf(x))-\phi(yf_c(x))\}^2\right] \leq 8\left(c_0\right)^{\frac{1}{q+1}}\left (B(\varepsilon(f) - \varepsilon(f_c))\right)^{\frac{q}{q+1}}.
\end{equation}
\end{lemma}
The proof of Lemma \ref{secondmoment} is given in Appendix \ref{app:proofsecondmoment}.

For $\epsilon >0$, denote by $\mathcal{N}(\epsilon,\mathcal{K})=\mathcal{N}(\epsilon,\mathcal{K},\|\cdot\|_\infty)$ the $\epsilon$-covering number of a set of functions $\mathcal{K}$ with respect to $\|\cdot\|_\infty:= \sup_z |f(z)|$. More specifically, $\mathcal{N}(\epsilon,\mathcal{K})$ is the minimal $N\in \NN$ such that there exists functions $\{f_1,\ldots, f_N\}\in \mathcal{K}$ satisfying
\begin{equation}
    \min_{1\leq i \leq N} \|f-f_i\|_\infty \leq \epsilon, \qquad \forall f\in \mathcal{K}.
\end{equation}

Observe that the Hinge loss function $\phi(t) = \max \{0, 1-t\}$ is Lipschitz continuous on $\RR$ with Lipschitz constant $M=1$ because 
\begin{equation} \label{Lip}
   |\phi(t_1)-\phi(t_2)| \leq |t_1-t_2| \qquad \forall t_1, t_2 \in \RR. 
\end{equation}

Next, we construct a function set $ \mathcal{G}$ induced by functions in $\mathcal{H}$. The following Lemma tells us that the covering number of  $ \mathcal{G}$, denoted by $\mathcal{N}(\epsilon,\mathcal{G})$, is no greater than $\mathcal{N}(\epsilon,\mathcal{H})$. 
After that, we will proceed to estimate $\mathcal{N}(\epsilon,\mathcal{H})$ which will help us derive an upper bound of $\varepsilon(f_z) -\varepsilon(f_c)- (\varepsilon_z(f_z)-\varepsilon_z(f_c))$.

\begin{lemma} \label{boundofmathcalG}
For $\phi(t) = \max \{0, 1-t\}$, define the set of functions on $Z = \mathcal{X} \times \mathcal{Y}$ given by 
\begin{equation}\label{mathcalG}
    \mathcal{G}:= \{\phi(yf(x))-\phi(yf_c(x)):f\in \mathcal{H}\},
\end{equation}
where $\mathcal{H}$ is the hypothesis space defined at (\ref{hypothesisspace}). For $\epsilon >0$, there holds 
\begin{equation}
    \mathcal{N}(\epsilon,\mathcal{G}) \leq \mathcal{N}(\epsilon,\mathcal{H}).
\end{equation}
\end{lemma}
\begin{proof}
For any $f_1, f_2 \in \mathcal{H}$ and $(x,y)\in Z$, it follows from (\ref{Lip}) 
\begin{align*}
    \left|\bigl\{\phi(yf_1(x))-\phi(yf_c(x))\bigr\}- \bigl\{\phi(yf_2(x))-\phi(yf_c(x))\bigr\}\right| &= \left|\phi(yf_1(x)) -\phi(yf_2(x)) \right| \\
    &\leq |yf_1(x) - yf_2(x)|\\
    &\leq \|f_1 - f_2\|_\infty,
\end{align*}
which implies $\mathcal{N}(\epsilon,\mathcal{G}) \leq \mathcal{N}(\epsilon,\mathcal{H})$.
\end{proof}

For $\ell \in \NN$, let $c^* = [-C_K,C_K]^{2^\ell}$ 
 with $C_K$ to be a positive constant given in Definition \ref{hypothesisspace}. 
 The following Proposition devotes to an upper bound of the covering number of our hypothesis space $\mathcal{H}$. Its proof is relatively long and is given in Appendix \ref{app:proofcoveringnumber}. 

\begin{proposition}[\textbf{Covering number of the hypothesis space $ \mathcal{H}$}] \label{coveringnumber}
Let $\mathcal{H}$ be defined by Definition \ref{hypothesisspace} with $C_K^0 \geq \frac{\sqrt{d}}{\sqrt{\tilde\sigma}}, C_K^1 \geq \frac{\mu^*}{\sqrt{\tilde\sigma}}$, and $R_2 \geq \sqrt{d}(b+1)C_K^0 + C_K^1$.
For $0<\epsilon \leq 1, 0<\lambda \leq 1, b\geq 1, m,\ell\in \NN$, there holds  
   \begin{equation}
    \log \mathcal{N}(\epsilon,\mathcal{H}) \leq C^\prime m 2^\ell \log \left(\frac{bC_K}{\lambda \epsilon}\right) + 4\ell(2^\ell) \log (C_K)+ C^{\prime\prime}m^2\ell(2^\ell),
\end{equation}
   where $C^\prime, C^{\prime \prime}$ are positive constants independent of $m,\ell,b,\lambda,C_K$ or $\epsilon$. 
\end{proposition}

 As a simple corollary combining Proposition \ref{coveringnumber} and Lemma \ref{boundofmathcalG}, for every $0 < \epsilon\leq 1$, there holds
\begin{eqnarray*}
   \log \mathcal{N}(\epsilon,\mathcal{G}) \leq C^\prime m 2^\ell \log \left(\frac{bC_K}{\lambda \epsilon}\right) + 4\ell(2^\ell) \log (C_K)+ C^{\prime\prime}m^2\ell(2^\ell) ,
\end{eqnarray*}
where 
$C^\prime, C^{\prime\prime}$ are positive constants independent of $\ell, m,b,\lambda,C_K$ or $\epsilon$. We will next apply this covering number estimate to  derive a high probability upper bound of the estimation error term $\varepsilon(f_z) -\varepsilon(f_c)- (\varepsilon_z(f_z)-\varepsilon_z(f_c))$, which is the R.H.S. of (\ref{estimationerr1}). 
The proof of the following Lemma (Lemma \ref{estimation1}) is given in Appendix \ref{app:proofestimation1}.

\begin{lemma} \label{estimation1}
    Let $q \geq 0, m,\ell \in \NN, 0 < \lambda \leq 1$. Suppose noise condition (\ref{Tsybakov}) holds for some noise exponent $q$ and constant $c_0>0$. For any $0<\delta <1, n \geq 3$, with probability $1-\delta/2$, there holds
    \begin{eqnarray*}
      & &\varepsilon(f_z) -\varepsilon(f_c)- (\varepsilon_z(f_z)-\varepsilon_z(f_c)) \\
       & \leq& C_{q,B}\left(m^2 \ell 2^\ell + \ell (2^\ell) \log (C_K) +\log \left(\frac{2}{\delta}\right)+m2^\ell \log\left(\frac{bC_K}{\lambda}\right) \right) ^{\frac{q+1}{q+2}}\left(\frac{\log n}{n}\right)^{\frac{q+1}{q+2}} + \frac{\varepsilon(f_z) -\varepsilon(f_c)}{2},      
    \end{eqnarray*}
    where $C_{q,B}$ is a constant depending on $q, c_0, B$ only.
\end{lemma}

\subsubsection{Upper bound of  $\varepsilon_z(f_\mathcal{H}) - \varepsilon_z(f_c) - (\varepsilon(f_\mathcal{H})- \varepsilon(f_c)) $ }
Now we move on to estimate $\varepsilon_z(f_\mathcal{H}) - \varepsilon_z(f_c) - (\varepsilon(f_\mathcal{H})- \varepsilon(f_c)) $, which is the R.H.S. of (\ref{estimationerr2}). 

Define a random variable $\xi(z):=\xi(x,y)= \phi(yf_\mathcal{H}(x)) -\phi(yf_c(x)) $. We have 
\begin{eqnarray*}
 &  & \varepsilon_z(f_\mathcal{H}) - \varepsilon_z(f_c) -  (\varepsilon(f_\mathcal{H}) - \varepsilon(f_c))\\
   & =& \frac{1}{n} \sum_{i=1}^n \Bigl\{\phi (y_if_\mathcal{H}(x_i)) - \phi(y_if_c(x_i)) \Bigl\}- \int_Z \phi (yf_\mathcal{H}(x))-\phi(yf_c(x)) d\rho  \\
&=& \frac{1}{n} \sum_{i=1}^n \xi(z_i) - \mathrm{E}[\xi(z)]   
\end{eqnarray*}
is a function of a single random variable $\xi$ and thus can be estimated by Bernstein's inequality (see, e.g., \cite[Lemma A.2]{gyorfi2002distribution}).
To apply Bernstein's inequality, we need first to establish an upper bound of the variance of $\xi$, denoted by $\sigma^2 = \mathrm{Var}[\xi]$. To achieve so, we apply Lemma \ref{secondmoment} to $f_\mathcal{H}$. With a bound of $\sigma^2$ in hand, we can obtain a high probability upper bound of $\varepsilon_z(f_\mathcal{H}) - \varepsilon_z(f_c) - (\varepsilon(f_\mathcal{H})- \varepsilon(f_c)) $. 

\begin{lemma} \label{estimation2}
Let $0\leq q \leq \infty$. Suppose noise condition (\ref{Tsybakov}) holds for some $q$ and constant $c_0>0$. For any $0<\delta <1$, with probability $1-\delta/2$, there holds
\begin{equation}
    \varepsilon_z(f_\mathcal{H}) - \varepsilon_z(f_c) - (\varepsilon(f_\mathcal{H})- \varepsilon(f_c)) \leq  \frac{4}{n}\log \left(\frac{2}{\delta}\right) + 2 (c_0)^{\frac{1}{q+1}}B^{\frac{q}{q+1}}(\varepsilon(f_\mathcal{H})-\varepsilon(f_c))^{\frac{q}{q+1}}.
\end{equation}
\end{lemma}
The proof of the above Lemma is given in Appendix \ref{app:proofestimation2}.

\subsection{Upper Bound of the Approximation Error}\label{subsec:approxerror}

Note that $\sigma_\lambda(\widetilde{D}) \in \mathcal{H}$ where $\widetilde{D}$ (defined in (\ref{tildeD})) is the approximation of the function $D$ in Theorem \ref{main1}. 
Recall that we use $f_\mathcal{H}$ denote any functions in $\mathcal{H}$. In this subsection, we derive a tight upper bound for the approximation error $\varepsilon(f_\mathcal{H})- \varepsilon(f_c)$ by taking $f_\mathcal{H} = \sigma_\lambda(\widetilde{D})$. 

Recall $\sigma^*$, the largest eigenvalue of all the covariance matrices $\Sigma_j$ for $j=1,\ldots, K$. Also recall $\mu^*$ defined earlier in (\ref{mu*}) as
$\mu^* =  \max_{1\leq i\leq K }\|\mu_i\|$.

\begin{lemma} \label{lemma:app}
Let $b \geq \max\{2 \mu^*,1\}$, $0 <\lambda \leq 1$. Let $\tau >0$ such that $\|\widetilde{D}-D\|_{L^\infty[-b,b]^d} \leq \tau$. Assume noise condition (\ref{Tsybakov}) holds for some noise exponent $q$ and constant $c_0>0$. There holds
\begin{equation}
  \left|\varepsilon\left(\sigma_\lambda\left(\widetilde{D}\right)\right)- \varepsilon(f_c)\right| \leq  2\left(\sum_{i=1}^K\frac{p_i}{\sqrt{\Sigma_i}}(4\sigma^*)^{\frac{d}{2}}\right) \exp\left(-\frac{b^2}{16\sigma^*}\right) + 2c_0\left(\tau^q + (\tau + \lambda)^q
    \right).
\end{equation}
\end{lemma}
\begin{proof}
We know that 
\begin{align*}
  |\varepsilon(f_\mathcal{H})- \varepsilon(f_c) | 
  &= \int_\mathcal{X} \int_\mathcal{Y} \left| yf_\mathcal{H}(x) - yf_c(x) \right| d\rho(y|x) d \rho_X\\
  &= \int_\mathcal{X}  \left| f_\mathcal{H}(x) - f_c(x)\right| |f_\rho (x)|  d \rho_X\leq \int_\mathcal{X}  \left| f_\mathcal{H}(x) - f_c(x)\right|  d \rho_X.
\end{align*}
Since $\rho_X$ has density function $D^+(x)+D^-(x)$, we know
\begin{equation} \label{appro}
    |\varepsilon(f_\mathcal{H})- \varepsilon(f_c) |
    \leq \int_\mathcal{X}  \left| f_\mathcal{H}(x) - f_c(x)\right|(D^+(x)+D^-(x)) dx .
\end{equation}
To further bound the R.H.S. of (\ref{appro}), we consider the cases $\|x\| \geq b$ and $\|x\| <b$ separately, with $b \geq \max\{2 \mu^*, 1\}$. For the case $\|x\| \geq b$, we have $\|x-\mu_j\| \geq \|x\|-\|\mu_j\| \geq \|x\|-\mu^*\geq b/2$.
Hence, we have \begin{equation*}
    (x-\mu_j)^T \Sigma_j^{-1}(x-\mu_j) \geq \frac{\|x-\mu_j\|^2}{\sigma^*} \geq \frac{b^2}{4\sigma^*}.
\end{equation*}
Observe that $|f_\mathcal{H}(x)|\leq 1$ and  $|f_c(x)|\leq 1$. 
Making use of the decay of Gaussian density function $\mathcal{N}$, we get
\begin{align*}
 &\quad \int_{\{x\in \mathcal{X}: \|x\|\geq b\} }  \left| f_\mathcal{H}(x) - f_c(x)\right|  (D^+(x)+D^-(x)) dx\\ 
  &\leq 2 \int_{\{x\in \mathcal{X}: \|x\|\geq b\} } \left(D^+(x)+D^-(x)\right) dx \\
  &\leq    2 \int_{\{x\in \mathcal{X}: \|x\|\geq b\} } \left(\sum_{i=1}^{K}\frac{p_i}{\sqrt{(2\pi)^d|\Sigma_i|}} \exp\left( 
 -\frac{\|x-\mu_i\|^2}{2\sigma^*}\right)\right) dx.
\end{align*}
But 
\begin{equation*}
    \exp \left(-\frac{\|x-\mu_i\|^2}{2\sigma^*}\right) \leq \exp \left(-\frac{\|x-\mu_i\|^2}{4\sigma^*}\right) \exp \left(-\frac{b^2}{16\sigma^*}\right).
\end{equation*}
We use $\int_{\{x\in \mathcal{X}: \|x\| \geq b\}} e^-\frac{\|x-\mu_j\|^2}{4\sigma^*} dx <\sqrt{(2\pi)^d} \left(2 \sqrt{\sigma^*}\right)^d$ and obtain 
\begin{eqnarray*}
 &&\int_{\{x\in \mathcal{X}: \|x\|\geq b\} }  \left| f_\mathcal{H}(x) - f_c(x)\right|  (D^+(x)+D^-(x)) dx\\
&\leq& 2\exp \left(-\frac{b^2}{16\sigma^*}\right) \sum_{i=1}^{K}\frac{p_i}{\sqrt{|\Sigma_i|}} \int_{\{x\in  \mathcal{X}: \|x\|\geq b\} } \frac{1}{\sqrt{(2\pi)^d}}\exp \left( 
 -\frac{\|x-\mu_i\|^2}{4\sigma^*}\right) dx\\
&\leq& 2\exp \left(-\frac{b^2}{16\sigma^*}\right) \sum_{i=1}^{K}\frac{p_i}{\sqrt{|\Sigma_i|}}(4\sigma^*)^{d/2} .  
\end{eqnarray*}

Take $f_\mathcal{H} = \sigma_\lambda\left(\widetilde{D}\right) \in \mathcal{H}$.
For $\|x\| <b$, we consider the cases $|D(x)| \leq \tau$ and $|D(x)| > \tau$ separately for some $\tau >0$ satisfying $\|\widetilde{D}-D\|_{L^\infty[-b,b]^d} \leq \tau$.
For the case $\|x\| <b$ and $|D(x)| \leq \tau$, it follows from the noise condition (\ref{Tsybakov}) that
\begin{align*}
  \int_{\{x\in  \mathcal{X}: \|x\|< b,\  |D(x)| \leq \tau\} }  \left| f_\mathcal{H}(x) - f_c(x)\right|  d \rho_X 
 &\leq 2 \int_{\{x\in  \mathcal{X}: \|x\|< b,\  |D(x)| \leq \tau\} }  d \rho_X \\
  &=2 \cdot \mathrm{P}(\{x\in  \mathcal{X} :\|x\|< b, |D(x)| \leq \tau\})\\
 &\leq2 \cdot \mathrm{P}(\{x\in  \mathcal{X} : |D(x)| \leq \tau\})\\
 &\leq 2c_0\tau^q .   
 \end{align*}
 
 For $\|x\|<b$ and $|D(x)| > \tau$, since $\|\widetilde{D}-D\|_{L^\infty[-b,b]^d} \leq \tau$, we have $\text{sgn}(\widetilde{D}) = \text{sgn} (D)= f_c$. If $\sigma_\lambda(\widetilde{D}(x))= \{1, -1\}$, then $\sigma_\lambda\left(\widetilde{D}(x)\right)$ is exactly equal to $f_c(x)$ which implies $|\sigma_\lambda\left(\widetilde{D}(x)\right) - f_c(x)|=0$. Thus, we have 
{\allowdisplaybreaks
\begin{eqnarray*}
  &&\int_{\{x\in \mathcal{X}: \|x\|< b,\  |D(x)| > \tau\} }  \left| \sigma_\lambda\left(\widetilde{D}\right)(x) - f_c(x)\right|  d \rho_X  \\
  &=&\int_{\{x\in \mathcal{X}: \|x\|< b,\  |D(x)| > \tau, |\sigma_\lambda\left(\widetilde{D}(x)\right)| < 1\}}  \left| \sigma_\lambda\left(\widetilde{D}\right)(x) - f_c(x)\right|  d \rho_X \\
  &\leq& 2 \cdot \mathrm{P}(x\in \mathcal{X}: \|x\|< b,\  |D(x)| > \tau, |\sigma_\lambda\left(\widetilde{D}\right)(x)| < 1)\\
  &\leq& 2 \cdot \mathrm{P}(x\in \mathcal{X}:  \|x\|< b, |\sigma_\lambda\left(\widetilde{D}\right)(x)| < 1)\\
  &=& 2 \cdot \mathrm{P}(x\in \mathcal{X}: \|x\|< b, |\widetilde{D}(x)| < \lambda) \\
  &\leq& 2 \cdot \mathrm{P}(x\in \mathcal{X}:  |D(x)| < \tau + \lambda)\\
  &\leq& 2c_0(\tau + \lambda)^q.
\end{eqnarray*}}
 Here, we have used the equivalence between $|\sigma_\lambda\left(\widetilde{D}\right)(x)| < 1$ and $|\widetilde{D}(x)| < \lambda$  and the condition $\|\widetilde{D}-D\|_{L^\infty[-b,b]^d} \leq \tau$ for getting $|D(x)| < \tau + \lambda$ when $\|x\|< b$ and $|\widetilde{D}(x)| < \lambda$.
 Combining the above estimates, we get the desired bound and prove the lemma.
 
 \end{proof}
 \subsection{Combining error bounds together}
 Now that we derived the upper bounds of the estimation errors and the approximation error, we can combine them together to prove Theorem \ref{main3}.

\begin{proof}[Proof of Theorem \ref{main3}]
     With probability at least $1-\delta$, we have 
     \begin{align*}
       &\quad\varepsilon(f_z) - \varepsilon(f_c)\\
       &\leq   \{\varepsilon(f_z)- \varepsilon_z(f_z)\}  +  \{\varepsilon_z(f_\mathcal{H}) - \varepsilon(f_\mathcal{H})\} + \{\varepsilon(f_\mathcal{H}) - \varepsilon(f_c)\}\\
       &\leq C_{q,B}\left(m^2 \ell 2^\ell + \ell (2^\ell) \log (C_K)  +\log \left(\frac{2}{\delta}\right)+m2^\ell \log\left(\frac{bC_K}{\lambda}\right) \right) ^{\frac{q+1}{q+2}}\left(\frac{\log n}{n}\right)^{\frac{q+1}{q+2}} \\
       & \quad + \frac{\varepsilon(f_z) -\varepsilon(f_c)}{2} +\frac{4}{n}\log \left(\frac{2}{\delta}\right) + 2 (c_0)^{\frac{1}{q+1}}B^{\frac{q}{q+1}}(\varepsilon(f_\mathcal{H})-\varepsilon(f_c))^{\frac{q}{q+1}} + \varepsilon(f_\mathcal{H}) - \varepsilon(f_c).
     \end{align*}
     Take $\tau$ as in Lemma \ref{lemma:app}. This implies, with probability at least $1-\delta$,
     {\allowdisplaybreaks
     \begin{eqnarray*}
         &&\varepsilon(f_z) - \varepsilon(f_c)\\
         &\leq& 2C_{q,B}\left(m^2 \ell 2^\ell +  \ell (2^\ell) \log (C_K) + \log \left(\frac{2}{\delta}\right)+m2^\ell \log\left(\frac{bC_K}{\lambda}\right) \right) ^{\frac{q+1}{q+2}}\left(\frac{\log n}{n}\right)^{\frac{q+1}{q+2}} \\
         && +\frac{8}{3n}\log \left(\frac{2}{\delta}\right)+\frac{8}{n}\log \left(\frac{2}{\delta}\right) +4(c_0)^{\frac{1}{q+1}}B^{\frac{q}{q+1}}(\varepsilon(f_\mathcal{H})-\varepsilon(f_c))^{\frac{q}{q+1}} + 2(\varepsilon(f_\mathcal{H}) - \varepsilon(f_c))\\
         &\leq& 2C_{q,B}\left(m^2 \ell 2^\ell +  \ell (2^\ell) \log (C_K)  +\log \left(\frac{2}{\delta}\right)+m2^\ell \log\left(\frac{bC_K}{\lambda}\right) \right) ^{\frac{q+1}{q+2}}\left(\frac{\log n}{n}\right)^{\frac{q+1}{q+2}} \\
         &&+\frac{11}{n}\log \left(\frac{2}{\delta}\right) + \left(2+ 4(c_0)^{\frac{1}{q+1}}B^{\frac{q}{q+1}}\right)(\varepsilon(f_\mathcal{H}) - \varepsilon(f_c))\\
         &=& 2C_{q,B}\left(m^2 \ell 2^\ell +  \ell (2^\ell) \log (C_K) +\log \left(\frac{2}{\delta}\right)+m2^\ell \log\left(\frac{bC_K}{\lambda}\right) \right) ^{\frac{q+1}{q+2}}\left(\frac{\log n}{n}\right)^{\frac{q+1}{q+2}} +\frac{11}{n}\log \left(\frac{2}{\delta}\right) \\
         &&+ \left(4+ 8(c_0)^{\frac{1}{q+1}}B^{\frac{q}{q+1}}\right)\left\{\left(\sum_{i=1}^K\frac{p_i}{\sqrt{\Sigma_i}}(4\sigma^*)^{\frac{d}{2}}\right) \exp\left(-\frac{b^2}{16\sigma^*}\right) + c_0\left(\tau^q + (\tau + \lambda)^q  \right)\right\}.  
 \end{eqnarray*}}
     Now take $c^\prime_{q,B} =4+ 8(c_0)^{\frac{1}{q+1}}B^{\frac{q}{q+1}}, c_1 =\sum_{i=1}^K\frac{p_i}{\sqrt{\Sigma_i}}(4\sigma^*)^{\frac{d}{2}}$. We have,  with probability at least $1-\delta$, 
     \begin{eqnarray*}
         \varepsilon(f_z) - \varepsilon(f_c) &\leq& 2C_{q,B}\left(m^2 \ell 2^\ell  +  \ell (2^\ell) \log (C_K) +\log \left(\frac{2}{\delta}\right)+m2^\ell \log\left(\frac{bC_K}{\lambda}\right) \right) ^{\frac{q+1}{q+2}}\left(\frac{\log n}{n}\right)^{\frac{q+1}{q+2}}\\
         && +\frac{11}{n}\log \left(\frac{2}{\delta}\right) + c^\prime_{q,B}\left(c_1\exp\left(-\frac{b^2}{16\sigma^*}\right) + 2c_0(\tau + \lambda)^q \right).
 \end{eqnarray*}
Recall $R_2 \geq \sqrt{d}(b+1)C_K^0 + C_K^1$ and $C_K \geq \sum_{j=1}^K |\beta_j| e^{(R_2)^2d}$ are positive tunable parameters. 
From (\ref{main1bounded}) of Theorem \ref{main1}, we know that $$\|\widetilde{D}-D\|_{L^\infty[-b,b]^d} \leq C_K \left(\frac{2^{\ell+1}}{4^{m+1}}+2^{-\frac{\ell (2^\ell)}{2}}\right) = C_K \left(\frac{2^\ell}{2(4^m)}+2^{-\frac{\ell (2^\ell)}{2}}\right).$$ 

We take $R_2 = 2\sqrt{d}(C_K^0 + C_K^1)b$, $C_K=c_2 \exp(c_3 b^2)$ and $\tau = c_2 \exp(c_3 b^2) \left(\frac{2^\ell}{2(4^m)}+2^{-\frac{\ell (2^\ell)}{2}}\right)$,
where $c_2 \geq \sum_{j=1}^K |\beta_j|, c_3 \geq (2\sqrt{d}(C_K^0 + C_K^1))^2d  =4d^2(C_K^0 + C_K^1))^2 \geq 4d^2 \frac{(\sqrt{d }+ \mu^*)^2}{\tilde \sigma}$ . Then $\log (C_K)= \log (c_2) + c_3b^2$.  It follows that with probability at least $1-\delta$, 
     \begin{eqnarray*}
         &&\varepsilon(f_z) - \varepsilon(f_c) \nonumber\\
         &\leq& 2C_{q,B}\left(m^2 \ell 2^\ell +  \ell (2^\ell) (\log (c_2) + c_3)b^2 +\log \left(\frac{2}{\delta}\right)+m2^\ell\left(\log\left(\frac{b}{\lambda}\right)
         + (\log (c_2) + c_3)b^2\right)\right) ^{\frac{q+1}{q+2}}\left(\frac{\log n}{n}\right)^{\frac{q+1}{q+2}}  \nonumber \\
         &&+\frac{11}{n}\log \left(\frac{2}{\delta}\right) 
         + c^\prime_{q,B}\left(c_1\exp\left(-\frac{b^2}{16\sigma^*}\right) + 2c_0\left(c_2 \exp(c_3 b^2) \left(\frac{2^\ell}{2(4^m)}+2^{-\frac{\ell (2^\ell)}{2}}\right) + \lambda\right)^q \right).
   \end{eqnarray*}
     
We take $m=\ell(2^{\ell-1})$.
We can easily verify that $\left(\frac{2^\ell}{2(4^m)}+2^{-\frac{\ell (2^\ell)}{2}}\right) \leq 2 \cdot 2^{-\frac{\ell (2^\ell)}{2}} = 2^{1-\ell (2^{\ell-1})}$. 
We then choose $b$ such that $\exp\left(-\frac{b^2}{16\sigma^*}\right) =  \left(\exp(c_3 b^2)2^{-\ell (2^{\ell-1})}\right)^q$, that is 
\begin{eqnarray*}
  -\frac{b^2}{16\sigma^*} = c_3qb^2 - \ell q(2^{\ell-1}) (\log 2) \iff  b^2=\frac{q(\log 2)\ell(2^{\ell-1})}{\frac{1}{16\sigma^*} + c_3q}.
\end{eqnarray*}
We thus take $b = \sqrt{c_q^\prime \ell (2^{\ell-1})}$ with $c_q^\prime = \frac{q(\log 2)}{\frac{1}{16\sigma^*} + c_3q}$. 
Next,  we take $\lambda = n^{-\frac{1}{q}}$. By our choices of $b, m,\lambda$ and by applying $(a+b)^q \leq 2^q\max \{a^q,b^q\}$, we can see that 
{\allowdisplaybreaks
\begin{eqnarray*}
&&c^\prime_{q,B}\left(c_1\exp\left(-\frac{b^2}{16\sigma^*}\right) + 2c_0\left(c_2 \exp(c_3 b^2) \left(\frac{2^\ell}{2(4^m)}+2^{-\frac{\ell (2^\ell)}{2}}\right) + \lambda\right)^q \right) \\
&\leq &c^\prime_{q,B}\left(c_1\exp\left(-\frac{b^2}{16\sigma^*}\right) + 2^{q+1}c_0\max \left\{c_2^q \exp(c_3 qb^2) \left(\frac{2^\ell}{2(4^m)}+2^{-\frac{\ell (2^\ell)}{2}}\right)^q,\lambda^q\right\}\right)\\
&\leq &c^\prime_{q,B}\left(c_1\exp\left(-\frac{b^2}{16\sigma^*}\right) + 2^{q+1}c_0c_2^q \exp(c_3 qb^2) \left(2^{1-\ell (2^{\ell-1})}\right)^q+2^{q+1}c_0\lambda^q\right)\\
&\leq& c^\prime_{q,B}\left((c_1+2^{q+1}c_0c_2^q 2^q)\exp\left(-\frac{b^2}{16\sigma^*}\right) + \frac{2^{q+1}c_0}{n}\right).
\end{eqnarray*}}
Finally, we take $\ell$ to be the smallest positive integer such that
$$\exp\left(-\frac{b^2}{16\sigma^*}\right) = \exp\left(-\frac{c_q^\prime \ell (2^{\ell-1})}{16\sigma^*}\right)  \leq \frac{1}{n}$$ 
which means
$\ell 2^{\ell-1} \geq \frac{16\sigma^*}{c_q^\prime} \log n $. 
With this choice, the above quantity is bounded by $\frac{c^{\prime\prime}_{q,B}}{n}$, where $c^{\prime\prime}_{q,B} =  c^\prime_{q,B}((c_1+2^{q+1}c_0c_2^q 2^q)+2^{q+1}c_0)$.

Observe from  $\sigma^* \geq \tilde\sigma$ that
$$\frac{16\sigma^*}{c_q^\prime} = \frac{16\sigma^*}{q(\log 2)}(\frac{1}{16\sigma^*} + c_3q) \geq \frac{16\sigma^* 4d^2 (\sqrt{d} + \mu^*)^2 q}{q(\log 2) \tilde \sigma} \geq  \frac{16\sigma^* 4d^3 }{(\log 2) \tilde \sigma} \geq 16 (4d^3) \geq 64. $$  Thus, the restriction on $\ell$ implies $\ell \geq 4$.
But $\ell$ is the smallest integer satisfying $(\ell -1) 2^{\ell-2} < \frac{16\sigma^*}{c_q^\prime} \log n$ and thereby $\ell 2^\ell < 16 (\ell -1) 2^{\ell-1} <  \frac{16^2\sigma^*}{c_q^\prime} \log n$.
Thus, with our choices of $\ell, b, m,\lambda$, we see that 
$m^2\ell2^\ell = \ell^3 2^{3\ell-2} \leq \left(\frac{16^2\sigma^* \log n}{c_q^\prime}\right)^3$
and 
\begin{eqnarray*}
  \ell (2^\ell) (\log (c_2) + c_3)b^2  \leq (\log (c_2) + c_3) \frac{(\log 2)q}{\frac{1}{16\sigma^*} + c_3q}(\ell (2^\ell))^2 
  \leq (\log (c_2) + c_3) \frac{\left(16^2\sigma^* \log n\right)^2}{c_q^\prime}
\end{eqnarray*}
and 
\begin{eqnarray*}
   &&m2^\ell\left(\log\left(\frac{b}{\lambda}\right)+ (\log (c_2) + c_3)b^2\right) \\
   &\leq&  m2^\ell((\log (c_2) + c_3+1)b^2 - \log \lambda )\\
   &\leq& \ell 2^{2\ell-1}\left((\log (c_2) + c_3+1)c_q^\prime \ell (2^{\ell-1})+ \frac{\log n }{q}\right)\\
   &\leq& (\log (c_2) + c_3+1) \frac{(16^2\sigma^* \log n)^3}{(c_q^\prime)^2} + \frac{1}{q}\left(\frac{16^2\sigma^*}{c_q^\prime}\right)^2(\log n)^3.
\end{eqnarray*}
The proof is complete. 
\end{proof}

\section{Conclusions} \label{sec:conclusions}
In this paper, we establish universal approximation theorems for GMM discriminant functions and general analytic functions using ReLU neural networks. Moreover, with Hinge loss and a Tsybakov-type noise condition, we obtain a fast convergence rate of the excess risk of order $O\left(n^{-\frac{q+1}{q+2}} (\log n)^4\right)$ 
for binary classification of GMM data by deep ReLU networks.
Our convergence rate is better than the existing ones in the literature by leveraging the analyticity of the Gaussian function. 
Also, our convergence rate does not depend on the dimension $d$, demonstrating that neural networks can overcome the curse of dimensionality in classification. 

To our best knowledge, our work is the first to study the generalization of classification with a GMM
without restrictions on model parameters or the number of Gaussian components. This is also the first paper studying the statistical guarantees of neural network classifiers on an unbounded domain. 
Our findings shed light on the practical effectiveness of deep neural networks in classification
problems, considering the universality of the Gaussian distribution across various data feature
spaces such as speeches, images, and texts. 

There has been an active line of research studying the theoretical aspects of multi-class classifications, see e.g., \citep{lei2015multi, bos2022convergence}. A future direction will be to extend our work to a multi-class classification problem. It would also be interesting to study the classification of GMM data with respect to a more general class of convex surrogate losses, e.g., the cross-entropy loss and logistic loss. These problems deserve further study in the future.


\section*{Acknowledgements}The authors are partially sponsored by NSF grants DMS 2015363 and CCF-1740776. Zhou is also partially supported by Georgia Tech Algorithm and Randomness Center (ARC) - Algorithms, Combinatorics and Optimization (ACO) Fellowship.

\appendix
\section*{Appendices}
\section{Proof of Supporting Lemmas}
\subsection{Proof of Lemma \ref{localization}} \label{app:prooflocalization}
\begin{proof}
The proof is straightforward. Since $0 \leq T_b(u_i)\leq 1$, we have $\sum_{i=1}^d T_b(u_i) -(d-1) \leq 1$ then $0 \leq \Psi_b(u) \leq 1$ for $u \in \RR^d$. If $u_j \in [-b,b]$ for all $j \in \{ 1,\ldots, d\}$, then $\sum_{i=1}^d T_b(u_i) = d$ which gives $\Psi_b(u)=1$. If $u_j \notin [-b-1,b+1]$ for at least one $j\in \{1,\ldots,d\}$, then we have $\sum_{i=1}^d T_b(u_i) \leq d-1$ and thus $\Psi_b(u)=0$.
\end{proof}

\subsection{Proof of Lemma \ref{secondmoment}} \label{app:proofsecondmoment}
\begin{proof}
Since $f(x)\in [-1,1]$, $\phi(yf(x))-\phi(yf_c(x)) = 1-yf(x) - (1-yf_c(x)) = y(f_c(x) - f(x))$. It follows that
\begin{equation}
    \mathrm{E}\left[\{\phi(yf(x))-\phi(yf_c(x))\}^2\right] =  \mathrm{E}[y^2(f_c(x) -f(x))^2]=\int_\mathcal{X} (f_c(x) -f(x))^2 d\rho_X
\end{equation}
and 
\begin{align*}
    \varepsilon(f) - \varepsilon(f_c) = \int_\mathcal{X} \int_\mathcal{Y} y(f_c(x) -f(x))  d\rho(y|x)d\rho_X &= \int_\mathcal{X} (f_c(x) -f(x))f_\rho(x) d\rho_X \\
    &=\int_\mathcal{X} |f_c(x) -f(x)||f_\rho(x)|
    d\rho_X 
\end{align*}
because, once again, $f_c = \text{sgn}(f_\rho) \in \{-1,1\}$.

Let $t>0$. Consider these two subsets of domain $\mathcal{X}$: $\mathcal{X}_t^+ = \{x\in \mathcal{X}: |D(x)|>t\}$ and $\mathcal{X}_t^- = \{x\in \mathcal{X}: |D(x)|\leq t\}$. On the set  $\mathcal{X}_t^+$, we apply $|f_c(x) -f(x)| \leq 2$ and get
\begin{equation*}
    |f_c(x) -f(x)|^2 \leq 2 |f_c(x) -f(x)| \frac{|D(x)| }{t}.
\end{equation*}
On the set $\mathcal{X}_t^-$, we have $|f_c(x) -f(x)|^2 \leq 4$. 
Recall $B=\sum_{i\in \mathcal{T}^+}\frac{P^+p_i}{\sqrt{(2\pi)^d|\Sigma_i|}}+ \sum_{j\in \mathcal{T}^-}\frac{P^-p_j}{\sqrt{(2\pi)^d|\Sigma_j|}}$. 
From (\ref{frho}) and (\ref{target}), we have $|D(x)| = |D^+(x) - D^-(x)|=|(D^+(x) + D^-(x))f_\rho(x)| \leq B|f_\rho(x)|$.
It follows from the noise condition (\ref{Tsybakov}) that 
{\allowdisplaybreaks
\begin{align*}
   \mathrm{E}\left[\{\phi(yf(x))-\phi(yf_c(x))\}^2\right]
   &= \int_{\mathcal{X}_t^-} (f_c(x) -f(x))^2 d\rho_X + \int_{\mathcal{X}_t^+} (f_c(x) -f(x))^2 d\rho_X\\
   &\leq 4 \mathrm{P}(\{x\in \mathcal{X}: |D(x)| \leq t\}) + 
   \frac{2}{t}\int_{\mathcal{X}_t^+}  |f_c(x) -f(x)| |D(x)| d\rho_X\\
    &\leq 4 \mathrm{P}(\{x\in \mathcal{X}: |D(x)| \leq t\}) + 
  \frac{2B}{t}\int_{\mathcal{X}}  |f_c(x) -f(x)| |f_\rho(x)| d\rho_X\\
   &\leq 4c_0t^q + \frac{2B}{t}(\varepsilon(f) - \varepsilon(f_c)).
\end{align*}}
Now set $t=\left(\frac{B(\varepsilon(f) - \varepsilon(f_c))}{2c_0}\right)^{1/(q+1)},$
we obtain our desired upper bound. 
\end{proof}

\subsection{Proof of Lemma \ref{estimation1}}
\label{app:proofestimation1}
Lemma \ref{estimation1} present a high probability upper bound of the estimation error term $\varepsilon(f_z) -\varepsilon(f_c)- (\varepsilon_z(f_z)-\varepsilon_z(f_c))$. To prove Lemma \ref{estimation1}, we make use of the following concentration inequality found in \cite[Lemma 10.20]{cucker2007learning}:
{\allowdisplaybreaks
 \begin{lemma} \label{lemmazhou} 
 Let $0\leq \gamma \leq 1$, $C_1,C_2 \geq 0$ and $\mathcal{K}$ be a set of functions on $Z$ such that for every $f\in \mathcal{K}$, $\mathrm{E}[f] \geq 0, \ \|f-\mathrm{E}[f]\|_\infty \leq C_2 \text{ and }\  \mathrm{E}[f^2] \leq C_1 (\mathrm{E}[f])^\gamma.$ Then for all $\epsilon >0$, with probability at least $$1-\mathcal{N}(\epsilon,\mathcal{K})\exp\left(-\frac{n\epsilon^{2-\gamma}}{2\left(C_1+\frac{C_2}{3}\epsilon^{1-\gamma}\right) }\right),$$ there holds
 \begin{equation}
     \sup_{f\in \mathcal{F}} \left\{\frac{\mathrm{E}[f] - \frac{1}{n}\sum_{i=1}^n f(z_i)}{((\mathrm{E}[f])^\gamma +\epsilon^\gamma)^{1/2}}\right\} \leq 4 \epsilon^{1-\frac{\gamma}{2}}.
 \end{equation}
 \end{lemma}}

Recall the function set $ \mathcal{G}= \{\phi(yf(x))-\phi(yf_c(x)):f\in \mathcal{H}\}$ defined earlier in Lemma \ref{boundofmathcalG}. 
To achieve the upper bound given in Lemma \ref{estimation1}, we will apply Lemma \ref{lemmazhou} on the function set $\mathcal{G}$. 

\begin{proof}[Proof of Lemma \ref{estimation1}]
To apply Lemma \ref{lemmazhou} to the function set $ \mathcal{G}= \{\phi(yf(x))-\phi(yf_c(x)):f\in \mathcal{H}\}$, we need first to check that the three assumptions stated in Lemma \ref{lemmazhou} are satisfied. That is, for every function $g \in \mathcal{G}$, there holds \begin{enumerate}
    \item $\mathrm{E}[g] \geq 0$,
    \item $\|f-\mathrm{E}[g]\|_\infty \leq C_2$,
    \item  $\mathrm{E}[g^2] \leq C_1 (\mathrm{E}[g])^\gamma$,
\end{enumerate}
for some constants $C_1, C_2 \geq 0,$ and $ 0\leq \gamma \leq 1$.

We have $\mathrm{E}[g] = \varepsilon(f) - \varepsilon(f_c) \geq 0$ because the Bayes classifier $f_c$ minimizes the generalization error $\varepsilon(f)$ over all functions. 
Since $f\in \mathcal{H}$ is bounded in $[-1,1]$, $|g(x)| =|\phi(yf(x))-\phi(yf_c(x))| = |1-yf(x) - (1-yf_c(x))| = |y(f_c(x) - f(x))| \leq 2$.
We know $|E[g]| \leq 2$. It follows that $\|g-\mathrm{E}[g]\|_\infty \leq 4$.
So the second assumption is satisfied with $C_2=4$.
Recall from Lemma \ref{secondmoment} that  $\mathrm{E}[g^2] \leq C_1 (\mathrm{E}[g])^\gamma$ with $C_1=8\left(c_0\right)^{\frac{1}{q+1}}(B)^{\frac{q}{q+1}}$ and $\gamma = \frac{q}{q+1}$ with $c_0>0, q\geq 0$. 
So we proved that all three assumptions are satisfied for every $g \in \mathcal{G}$. 

We can now apply Lemma \ref{lemmazhou} to $\mathcal{G}$. Lemma \ref{lemmazhou} tells us that for every $0 < \epsilon \leq 1$, with probability at least 
{\allowdisplaybreaks
\begin{eqnarray*}
 &&\ 1-\mathcal{N}(\epsilon,\mathcal{G})\exp\left(-\frac{n\epsilon^{2-\gamma}}{2\left(C_1+\frac{C_2}{3}\epsilon^{1-\gamma}\right) }\right) \\
 &=& 1-\mathcal{N}(\epsilon,\mathcal{G})\exp\left(-\frac{n\epsilon^{2-\gamma}}{16\left(c_0\right)^{\frac{1}{q+1}}B^{\frac{q}{q+1}}+\frac{4}{3}\epsilon^{1-\gamma} }\right)\\
 &\geq& 1- \exp\left(C^\prime m 2^\ell \log \left(\frac{bC_K}{\lambda \epsilon}\right) + 4\ell(2^\ell) \log (C_K)+C^{\prime\prime}m^2\ell(2^\ell) -\frac{n\epsilon^{2-\gamma}}{16\left(c_0\right)^{\frac{1}{q+1}}B^{\frac{q}{q+1}}+ \frac{4}{3}}\right),     
\end{eqnarray*}}
there holds
 \begin{equation*}
     \sup_{g\in \mathcal{G}} \left\{\frac{\mathrm{E}[g] - \frac{1}{n}\sum_{i=1}^n g(z_i)}{((\mathrm{E}[g])^\gamma +\epsilon^\gamma)^{1/2}}\right\} \leq 4 \epsilon^{1-\frac{\gamma}{2}},
 \end{equation*}
 which implies 
  \begin{eqnarray*}
  \mathrm{E}[g] - \frac{1}{n}\sum_{i=1}^n g(z_i)\leq 4 \epsilon^{1-\frac{\gamma}{2}}((\mathrm{E}[g])^\gamma +\epsilon^\gamma)^{1/2}, \qquad \forall g\in \mathcal{G} 
  \end{eqnarray*}
  and thereby
  \begin{eqnarray*}
 \varepsilon(f) -\varepsilon(f_c)- (\varepsilon_z(f)-\varepsilon_z(f_c)) \leq 4 \epsilon^{1-\frac{\gamma}{2}}((\varepsilon(f) -\varepsilon(f_c))^\gamma +\epsilon^\gamma)^{1/2}, \qquad \forall f\in \mathcal{H}.       
  \end{eqnarray*}

For brevity, we choose not to plug in $\gamma = \frac{q}{q+1} \in [0,1]$ for now. 

Setting the above confidence bound to be $1-\delta/2$, then the solution $\epsilon$ satisfies
    \small
    $$
    C^\prime m 2^\ell \log \left(\frac{1}{\epsilon}\right)-\frac{n\epsilon^{2-\gamma}}{16\left(c_0\right)^{\frac{1}{q+1}}B^{\frac{q}{q+1}}+\frac{4}{3}} = \log \left(\frac{\delta}{2}\right)-C^\prime m 2^\ell \log \left(\frac{bC_K}{\lambda}\right) -   4\ell(2^\ell) \log (C_K)- C^{\prime\prime}m^2 \ell (2^\ell).     $$
    \normalsize
Let $\hat \epsilon = \epsilon^{2-\gamma}$ and $c_{q,B} = 16\left(c_0\right)^{\frac{1}{q+1}}B^{\frac{q}{q+1}} + \frac{4}{3} >0$. We can see that $c_{q,B}$ is a constant depending on $q,c_0,B$ only. We then have
\begin{equation*}
\frac{C^\prime m 2^\ell }{2-\gamma} \log \left(\frac{1}{\hat \epsilon}\right)-\frac{n\hat \epsilon}{c_{q,B}} = \log \left(\frac{\delta}{2}\right)-C^\prime m 2^\ell \log \left(\frac{bC_K}{\lambda}\right) -  4\ell(2^\ell) \log (C_K)- C^{\prime\prime}m^2 \ell (2^\ell).
\end{equation*}
We solve for $\hat \epsilon$ with the above equation. Note that the function $T:(0,1] \rightarrow \RR$ defined by $T(u) = \frac{C^\prime m 2^\ell }{2-\gamma} \log \left(\frac{1}{u}\right)-\frac{nu}{c_{q,B}}$ is decreasing. Take
\begin{equation*}
    A = \left(\frac{C^\prime m 2^{\ell} }{2-\gamma} +\log \left(\frac{2}{\delta
    }\right) + C^\prime m 2^\ell \log \left(\frac{bC_K}{\lambda}\right)+  4\ell(2^\ell) \log (C_K)+ C^{\prime\prime}m^2 \ell (2^\ell)\right)c_{q,B}.
\end{equation*}
For $n \geq 3$ (which implies $\log n >1$), there holds $\frac{A(\log n)}{n} \geq \frac{1}{n}$.
It follows that 
\begin{align*}
   & T\left(\frac{A(\log n)}{n}\right) \\
   &\leq \frac{C^\prime m 2^\ell }{2-\gamma}  \log n-(\log n)\left(\frac{C^\prime m 2^{\ell} }{2-\gamma} +\log \left(\frac{2}{\delta
    }\right) + C^\prime m 2^\ell \log \left(\frac{bC_K}{\lambda}\right)+  4\ell(2^\ell) \log (C_K)+ C^{\prime\prime}m^2 \ell (2^\ell)\right)\\
    &\leq - (\log n)\left(\log \left(\frac{2}{\delta
    }\right)+ C^\prime m 2^\ell \log \left(\frac{bC_K}{\lambda}\right)+  4\ell(2^\ell) \log (C_K)+ C^{\prime\prime}m^2 \ell (2^\ell)\right) \\
    &\leq \log \left(\frac{\delta}{2}\right) - C^\prime m 2^\ell \log \left(\frac{bC_K}{\lambda}\right) -  4\ell(2^\ell) \log (C_K)-C^{\prime\prime}m^2 \ell (2^\ell).
\end{align*}
Since $T$ is decreasing, we have 
$\hat \epsilon \leq \frac{A(\log n)}{n}$, which implies 
$  \epsilon \leq \left(\frac{A(\log n)}{n}\right)^{1/(2-\gamma)}$.

  Now, take $f = f_z$ and the above estimate of $\epsilon$, we obtain
\begin{align*}
    \varepsilon(f_z) -\varepsilon(f_c)- (\varepsilon_z(f_z)-\varepsilon_z(f_c)) 
    &\leq  4\left(\frac{A(\log n)}{n}(\varepsilon(f_z) -\varepsilon(f_c))^\gamma\right)^{1/2} +4\left(\frac{A(\log n)}{n}\right)^{1/(2-\gamma)}
\end{align*}
with probability at least $1-\delta/2$. Plug in $\gamma = \frac{q}{q+1}$, we get with probability at least $1-\delta/2$,
\begin{equation*}
     \varepsilon(f_z) -\varepsilon(f_c)- (\varepsilon_z(f_z)-\varepsilon_z(f_c)) \leq 4\left(\frac{A(\log n)}{n}\right)^{1/2}(\varepsilon(f_z) -\varepsilon(f_c))^{\frac{q}{2(q+1)}}+4\left(\frac{A(\log n)}{n}\right)^{\frac{q+1}{q+2}}.  
\end{equation*}

We then apply Young's Inequality for products \citep{young1912classes}: 
\begin{equation*}
        a\cdot b \leq \frac{a^p}{p}+ \frac{b^{p^*}}{p^*} \qquad \text{with    } a\geq 0,b \geq 0, p>1,p^*>1 \text{ and } \frac{1}{p}+ \frac{1}{p^*} = 1
    \end{equation*}
    to further upper bound $4\left(\frac{A(\log n)}{n}\right)^{1/2}(\varepsilon(f_z) -\varepsilon(f_c))^{\frac{q}{2(q+1)}}$. We get
\begin{align} \label{youngbound}
    4\left(\frac{A(\log n)}{n}\right)^{1/2}(\varepsilon(f_z) -\varepsilon(f_c))^{\frac{q}{2(q+1)}} 
    &\leq \frac{\left(4\left(\frac{A(\log n)}{n}\right)^{1/2}\right)^{\frac{2(q+1)}{q+2}}}{\frac{2(q+1)}{q+2}} + \frac{\varepsilon(f_z) -\varepsilon(f_c)}{\frac{2(q+1)}{q}} \nonumber\\
    &\leq 2^{\frac{3q+2}{q+2}}\frac{q+2}{q+1}\left(\frac{A(\log n)}{n}\right)^{\frac{q+1}{q+2}}  + \frac{\varepsilon(f_z) -\varepsilon(f_c)}{2}.
\end{align}
Notice that $A \leq c_{q,B} \left(\frac{C^\prime}{2-(q/(q+1))} + C^{\prime\prime}\right)\left(m^2 \ell (2^\ell) +\log \left(\frac{2}{\delta
    }\right) +  4\ell(2^\ell) \log (C_K)+ m 2^\ell \log \left(\frac{bC_K}{\lambda}\right)\right)$. Plug in this upper bound of $A$ and (\ref{youngbound}), we finally get, with probability at least $1-\delta/2$,
\begin{align*}
 & \qquad\varepsilon(f_z) -\varepsilon(f_c)- (\varepsilon_z(f_z)-\varepsilon_z(f_c)) \\
 &\leq   \left(2^{\frac{3q+2}{q+2}}\frac{q+2}{q+1}+4\right)\left(\frac{A(\log n)}{n}\right)^{\frac{q+1}{q+2}} + \frac{\varepsilon(f_z) -\varepsilon(f_c)}{2} \\
 &\leq C_{q,B}\left(m^2 \ell 2^\ell +\log \left(\frac{2}{\delta}\right)+  \ell(2^\ell) \log (C_K) + m2^\ell \log\left(\frac{bC_K}{\lambda}\right) \right) ^{\frac{q+1}{q+2}}\left(\frac{\log n}{n}\right)^{\frac{q+1}{q+2}} + \frac{\varepsilon(f_z) -\varepsilon(f_c)}{2}, 
\end{align*}
where $C_{q,B}$ is a positive constant depending only on $q,c_0, B$.
The proof is complete.
\end{proof}

\subsection{Proof of Lemma \ref{estimation2}} \label{app:proofestimation2}
\begin{proof}
Notice that $|\xi(z)|= |\phi(yf_\mathcal{H}(x)) -\phi(yf_c(x))| \leq 2$. It follows from Lemma \ref{secondmoment} that 
$$\sigma^2 = \mathrm{Var}[\xi(z)] \leq \mathrm{E}[\xi(z)^2] \leq 8(c_0)^{\frac{1}{q+1}}\left(B(\varepsilon(f_\mathcal{H}) - \varepsilon(f_c))\right)^{\frac{q}{q+1}}.$$

By the one-sided Bernstein's inequality, for any $\eta >0 $, there holds, with probability at least $1- \exp\left(-\frac{n\eta^2}{2(\sigma^2+2\eta/3)}\right)$,
\begin{equation}
     \varepsilon_z(f_\mathcal{H}) - \varepsilon_z(f_c) - (\varepsilon(f_\mathcal{H})- \varepsilon(f_c)) \leq \eta.
\end{equation}
Setting this confidence bound to be $1-\delta/2$, we get a quadratic equation $\frac{n\eta^2}{2(\sigma^2+2\eta/3)}=\log(2/\delta)$ for $\eta$. We solve this equation and get a positive solution $\eta^*$ given by 
{\allowdisplaybreaks
\begin{align*}
    \eta^* &= \frac{\frac{4}{3}\log (\frac{2}{\delta})+\sqrt{\frac{16}{9}\left(\log (\frac{2}{\delta})\right)^2+8n\sigma^2\log (\frac{2}{\delta})}}{2n}\\
    &\leq \frac{2}{3n}\log \left(\frac{2}{\delta}\right) + \frac{2}{3n}\log \left(\frac{2}{\delta}\right) + \frac{\sqrt{2\sigma^2\log(\frac{2}{\delta})}}{\sqrt{n}} \\
    &\leq \frac{4}{3n}\log \left(\frac{2}{\delta}\right) +4\sqrt{\frac{\log(\frac{2}{\delta})}{n}}(c_0)^{\frac{1}{2(q+1)}}\left(B(\varepsilon(f_\mathcal{H})-\varepsilon(f_c))\right)^{\frac{q}{2(q+1)}}\\
    &\leq \frac{4}{3n}\log \left(\frac{2}{\delta}\right) + 2\left(\frac{\log(\frac{2}{\delta})}{n} +(c_0)^{\frac{1}{q+1}}\left(B(\varepsilon(f_\mathcal{H})-\varepsilon(f_c))\right)^{\frac{q}{q+1}}\right)\\
    &\leq  \frac{4}{n}\log \left(\frac{2}{\delta}\right) + 2 (c_0)^{\frac{1}{q+1}}B^{\frac{q}{q+1}}(\varepsilon(f_\mathcal{H})-\varepsilon(f_c))^{\frac{q}{q+1}}.
\end{align*}}
Here, we have used $2\sqrt{ab} \leq a+b$ in the third inequality.
\end{proof}

\section{Proof of Theorem \ref{main2}: Approximation of Analytic Functions} \label{sec:analytic}

In this part, we present the proof of Theorem \ref{main2}. We apply our monomial gate to approximate univariate analytic functions. 

\begin{proof}[Proof of Theorem \ref{main2}]
The convergence of the Taylor series of $t$ at $R_1$ implies 
\begin{equation*}
    \widetilde{M}:=\sup_{i\in \ZZ_+} \left|\frac{t^{(i)}(0)}{i!} (R_1)^i\right| < \infty.
\end{equation*}
It follows that, for $\ell \in \NN$ and $u\in [-1,1]$,
\begin{align*}
   \left| t(u)- t(0) - \sum_{i=1}^{2^\ell} \frac{t^{(i)}(0)}{i!} (R_1)^i \left(\frac{u}{R_1}\right)^i\right|
    =  \left| \sum_{i=2^\ell+1}^\infty \frac{t^{(i)}(0)}{i!} (R_1)^i \left(\frac{u}{R_1}\right)^i\right|
    &\leq \widetilde{M} \sum_{i=2^\ell+1}^\infty \left|\frac{u}{R_1}\right|^i \\
     &\leq \frac{\widetilde{M}}{R_1-1}R_1^{-2^\ell}.
\end{align*}
To further approximate $t(u)$, consider the monomial gate defined in Section 2 with input $u \in [-1,1]$. We can construct a deep ReLU network of depth $(m+1)\cdot \ell$ which outputs the function 
\begin{equation*}
    F(u) = t(0) + \sum_{i=1}^{2^\ell} \frac{t^{(i)}(0)}{i!} (R_1)^i \frac{h_i(u)}{(R_1)^i}, 
\end{equation*}
where $\{h_i\}_{i=1}^{2^\ell}$ are outputs of the monomial gate defined in Proposition \ref{monomialgate}.
Recall from Proposition \ref{monomialgate} that $\{h_i(u)\}_{i=1}^{2^\ell}$ approximate $\left\{u^i\right\}_{i=1}^{2^\ell}$ to an accuracy
\begin{equation*}
\left|h_{2^{j}+k}(u) - u^{2^{j}+k}\right| \leq \frac{2^{j+1}-1}{4^{m+1}}, \qquad \text{for }u\in [-1,1], \ j=0,\ldots, \ell-1,\  k=1,\ldots, 2^j.    \end{equation*} We have
{\allowdisplaybreaks
\begin{align*}
|F(u) - t(u)|  &\leq   \left|\sum_{i=1}^{2^\ell} \frac{t^{(i)}(0)}{i!} (R_1)^i \frac{h_i(u)-u^i}{\left(R_1\right)^i}\right| + \left|\sum_{i=2^\ell+1}^\infty \frac{t^{(i)}(0)}{i!} (R_1)^i \left(\frac{u}{R_1}\right)^i\right|\\
&\leq \widetilde{M} \sum_{i=1}^{2^\ell}  \frac{\left|h_i(u)-u^i\right|}{\left(R_1\right)^i} +  \frac{\widetilde{M}}{R_1-1}R_1^{-2^\ell}\\
&\leq \widetilde{M} \sum_{j=0}^{\ell-1} \sum_{k=1}^{2^{j}}  \frac{\left|h_{2^j+k}(u)-u^{2^j+k}\right|}{\left(R_1\right)^{2^j+k}} +  \frac{\widetilde{M}}{R_1-1}R_1^{-2^\ell}\\
&\leq \widetilde{M}\left( \sum_{j=0}^{\ell-1} \sum_{k=1}^{2^j}  \frac{\frac{2^{j+1}-1}{4^{m+1}}}{\left(R_1\right)^{2^j+k}} + \frac{1}{R_1^{2^\ell}(R_1-1)} \right)\\
&\leq \widetilde{M}\left(\frac{1}{4^{m+1}} \sum_{j=0}^{\ell-1}  \frac{2^{j+1}}{\left(R_1\right)^{2^j}} \sum_{k=1}^{2^j} \frac{1}{(R_1)^k}
+ \frac{1}{R_1^{2^\ell}(R_1-1)} \right)\\
&\leq \frac{\widetilde{M}}{(R_1-1)}\left(\frac{1}{4^{m+1}} \sum_{j=0}^{\ell-1}  \frac{2^{j+1}}{\left(R_1\right)^{2^j}}
+ \frac{1}{R_1^{2^\ell}} \right).
\end{align*}}
Observe that if $j \geq 5$, we have $2^j \geq j^2$ which implies $\left(R_1\right)^{2^j} \geq \left(R_1\right)^{j^2}$. 
Hence, for $j \geq \max \{5,\frac{\log 4}{\log R_1} \}$, $\left(R_1\right)^{2^j} \geq \left(R_1^j\right)^{j} \geq 4^j$. Then, 
{\allowdisplaybreaks
\begin{align*}
  &|F(u) - t(u)|  \\
  &\leq   \frac{\widetilde{M}}{(R_1-1)}\left(\frac{1}{4^{m+1}}\left( \sum_{j=0}^{\max\{4, \lfloor\log 4 / \log R_1 \rfloor\}} 2^{j+1} + 2\sum_{j=\max\{4, \lfloor \log 4 / \log R_1\rfloor\} + 1}^{\ell-1 } \left(\frac{1}{2}\right)^j\right)
+ \frac{1}{R_1^{2^\ell}} \right)\\
&\leq \frac{\widetilde{M}}{(R_1-1)}\left(\frac{1}{4^{m+1}}\left( 2^{4+ \lfloor\log 4 / \log R_1\rfloor +2}+ 2\right)
+ \frac{1}{R_1^{2^\ell}} \right)\\
&\leq \frac{\widetilde{M}}{(R_1-1)}\left(\frac{2^{7+ \lfloor\log 4 / \log R_1\rfloor }}{4^{m+1}}
+ \frac{1}{R_1^{2^\ell}} \right).
\end{align*}}
Take $C=  \frac{\widetilde{M}(2^{7+ \lfloor\log 4 / \log R_1\rfloor })}{(R_1-1)} = \frac{2^{7+ \lfloor\log 4 / \log R_1\rfloor }}{(R_1-1)}\sup_{i\in \ZZ_+} \left|\frac{t^{(i)}(0)}{i!} (R_1)^i\right| $.
The Proof of Theorem \ref{main2} is complete. 
\end{proof}

\section{Proof of Proposition \ref{coveringnumber}: Covering Number of the Hypothesis Space $\mathcal{H}$}
\label{app:proofcoveringnumber}
In this part, we derive the upper bound of the covering number of the hypothesis space $\mathcal{H}$ to prove Proposition \ref{coveringnumber}. 

We will first give uniform bounds of squaring gate (Subsection \ref{app:C1}), product gate (Subsection \ref{app:C2}), and monomial gate (Subsection \ref{app:C3}). Finally, we apply these uniform bounds to prove Proposition \ref{coveringnumber} (Subsection \ref{app:C4}). 

We begin by giving some notations. 
 For any vector $\nu\in \RR^{n}$, define $\|\nu\|_\infty:= \max_{1\leq i\leq n} |\nu_i|$ and $\|\nu\|_1:= \sum_{i=1}^n |\nu_i|$. For any matrix $A\in \RR^{m\times n}$, define $\|A\|_\infty:= \max_{1\leq i\leq m} \sum_{j=1}^n |A_{i,j}|$, which is the maximum absolute row sum of the matrix and equals the operator norm of $A: (\RR^n, \|\cdot\|_\infty ) \rightarrow (\RR^m, \|\cdot\|_\infty )$. For a function $f:[-T,T] \rightarrow \RR$, define $\|f\|_{L^\infty [-T,T]}:= \sup_{x\in [-T,T]} |f(x)|$. Denote further $$|f|_{Lip1}:= \sup_{x,y \in [-T,T], x\neq y} \frac{|f(x)-f(y)|}{|x-y|}$$
 as the Lipschitz-$1$ seminorm of a function on $[-T,T]$.
 \subsection{Uniform Bound of Squaring Gate $\widehat f_m$} \label{app:C1}
For $T>0$ and $x \in [-T,T]$, let $H^{(0)}=x$ and define iteratively for  $j=1,\ldots, m,$ 
 \begin{equation}\label{H:layer}
      H^{(j)}(x):= H^{(j)}_{\bm{W}, \bm{b}}(x) =\sigma ( W^{(j)}H^{(j-1)}(x)+b^{(j)})
     \end{equation} 
with    $W^{(1)} \in [-4,4]^{5 \times 1}$, $W^{(j)} \in [-4,4]^{5 \times 5}$ for $j \geq 2$ and $b^{(j)} \in  [-4,4]^5$.
For each  $H^{(j)} = \{(H^{(j)})_1, \ldots, (H^{(j)})_5\}$, define $$\|H^{(j)}\|_{L^\infty [-T,T]}:= \max_{1\leq i \leq 5} \|(H^{(j)})_i\|_{L^\infty [-T,T]}$$ and  
 $$|H^{(j)}|_{Lip1}:= \max_{1\leq i \leq 5} \left|(H^{(j)})_i\right|_{Lip1}.$$
\begin{lemma} \label{boundofH}
For each $j=1,\ldots, m$ and $H^{(j)}$ defined by (\ref{H:layer}) with $T>0$, there holds
\begin{equation} \label{boundofH:1}
    \|H^{(j)}\|_{L^\infty [-T,T]} \leq 20^j T+4 \left(\frac{20^j-1}{20-1}\right),
\end{equation}
and 
\begin{equation}\label{boundofH:2}
    |H^{(j)}|_{Lip1} \leq 20^j.
\end{equation}
\end{lemma}
\begin{proof}
From $|\sigma(u)| \leq |u|$, we have, for $i=1,\ldots, 5$, $j=1,\ldots,m$,
 \begin{align*}
    \left|\left(H^{(j)}(x)\right)_i\right| = \left|\sigma ( W^{(j)}H^{(j-1)}(x)+b^{(j)})_i\right|
    &\leq \left|(W^{(j)}H^{(j-1)}(x))_i + (b^{(j)})_i\right|\\
     &\leq \|W^{(j)}\|_{\infty}\|H^{(j-1)}\|_{L^\infty [-T,T]}+\|b^{(j)}\|_\infty\\
     &\leq 20 \|H^{(j-1)}\|_{L^\infty [-T,T]} +4 .
 \end{align*}
 This leads us to a recurrence relationship:
 \begin{align*}
     \|H^{(j)}\|_{L^\infty [-T,T]}   \leq 20 \|H^{(j-1)}\|_{L^\infty [-T,T]} +4 \leq \cdots &\leq 20^j\|H^{(0)}\|_{L^\infty [-T,T]} +4 \left(\frac{20^j-1}{20-1}\right) \\
     &= 20^j T +4 \left(\frac{20^j-1}{20-1}\right), \qquad \forall j=1,\ldots, m.
 \end{align*}
This completes the proof of (\ref{boundofH:1}).
 Applying $|\sigma(u)-\sigma(v)| \leq |u-v|$, we have for all $u \neq v \in [-1,1]$,
\begin{align*}
    \|H^{(j)}(u) - H^{(j)}(v)\|_\infty 
    &\leq \|W^{(j)}H^{(j-1)}(u)+b^{(j)} - W^{(j)}H^{(j-1)}(v)-b^{(j)}\|_\infty\\
    &\leq \|W^{(j)}\|_{\infty} \|H^{(j-1)}(u) - H^{(j-1)}(v)\|_\infty\\
    &\leq 20 \|H^{(j-1)}(u) - H^{(j-1)}(v)\|_\infty\\
    &\leq \cdots \leq 20^{j}\|H^{(0)}(u) - H^{(0)}(v)\|_\infty = 20^{j} |u-v|.
\end{align*}
Hence, we obtain 
$$|H^{(j)}|_{Lip1} =  \sup_{u,v \in [-T,T], u \neq v}\frac{ \|H^{(j)}(u) - H^{(j)}(v)\|_\infty}{|u-v|}\leq 20^j.$$
This completes the proof of (\ref{boundofH:2}).
\end{proof}

Recall that the squaring gate $\widehat f_m =: \widehat f_{m,\theta}$ is a ReLU FNN $\in \mathcal{F}(m,(5,5,\ldots,5))$ defined in Subsection \ref{section:EBTnet}. It has all the trainable parameters 
 $\theta = \{\bm{W},\bm{b},a\}$ taking values on $[-4,4]$. 
 For $T>0$, the next Lemma  devotes to a uniform bound of $\|\widehat f_{m,\theta} - \widehat f_{m,\widetilde \theta}\|_{L^\infty [-T,T]}$ where $\theta = \{\bm{W},\bm{b},a\}$ and $\widetilde \theta = \{\widetilde {\bm{W}},\widetilde {\bm{b}},\widetilde a\}$ represent two different collections of network parameters. 
 Denote by 
 \begin{equation*}
     \bm{W}^* := \left\{\bm{W}=W^{(1)} \in [-4,4]^{5 \times 1},(W^{(j)})_{j=2}^m \in \RR^{5 \times 5}: |W^{(j)}_{i,k}|\leq 4\right\}
 \end{equation*}
and 
 \begin{equation*}
     \bm{b}^* := \left\{\bm{b}=(b^{(j)})_{j=1}^m \in \RR^5: |b^{(j)}_i|\leq 4\right\}
 \end{equation*}
 and 
  \begin{equation*}
     a^* := \left\{a \in \RR^5: |a_i|\leq 4\right\}.
 \end{equation*}
 
 \begin{definition}[$\eta$-net]\label{etanet}
 For arbitrary $\eta >0$, let $\bm{W}^*_\eta,\ \bm{b}^*_\eta, a^*_\eta$ be $\eta$-nets of $\bm{W}^*, \bm{b}^*, a^*$, respectively, meaning that, for each $\bm{W} \in \bm{W}^*, \bm{b} \in \bm{b}^*$ and $ a \in a^*$, there exist $\widetilde{\bm{W}} \in \bm{W}^*_\eta,\ \widetilde{\bm{b}} \in \bm{b}^*_\eta,\ \widetilde{a} \in a^*_\eta$ such that 
 \begin{equation}
     \|\bm{W} -\widetilde{\bm{W}} \|_{\infty, \infty} \leq \eta, \qquad \|\bm{b} -\widetilde{\bm{b}} \|_\infty \leq \eta,\qquad  \|a-\widetilde{a} \|_1 \leq \eta. 
 \end{equation}
 Here, $\|\bm{W}\|_{\infty, \infty} := \max_{1\leq j \leq m} \|W^{(j)}\|_\infty$ and $\|\bm{b}\|_\infty := \max_{1\leq j \leq m} \|b^{(j)}\|_\infty$.
 \end{definition}
 
 \begin{lemma} \label{boundoffm}
 Let $\eta >0$ , $T>0$, and $m\in\NN$. Let $\bm{W}^*_\eta, \bm{b}^*_\eta, a^*_\eta$ be $\eta$-nets of $\bm{W}^*, \bm{b}^*, a^*$ defined above in Definition \ref{etanet}. 
With the network input $x \in [-T,T]$, there hold
\begin{equation}\label{boundoffm:1}
         |\widehat f_{m,\theta}|_{Lip1} \leq 20^{m+1}
\end{equation}
and
 \begin{equation}\label{boundoffm:2}
     \|\widehat f_{m,\theta} - \widehat f_{m,\widetilde \theta}\|_{L^\infty [-T,T]} \leq (T+1)(m+1)20^m \eta. 
 \end{equation}
 \end{lemma}
  \begin{proof}
 From Inequality (\ref{boundofH:2}) in Lemma \ref{boundofH}, we know 
\begin{align*}
    \|H^{(j)}(u) - H^{(j)}(v)\|_\infty 
    \leq  20^{j} |u-v|, \qquad \forall j=1,\ldots, m.
\end{align*}
It follows that
 $$\|H^{(m)}(u) - H^{(m)}(v)\|_\infty  \leq 20^{m} |u-v|, \qquad \forall u,v \in [-T,T].$$
Then, we obtain
\begin{align*}
 |\widehat f_m(u) -\widehat f_m(v)| 
  &= \left|\sum_{i=1}^5 a_i \left(H^{(m)}(u)\right)_i - \sum_{i=1}^5 a_i \left(H^{(m)}(v)\right)_i \right| \\
  &\leq 20 \left\| H^{(m)}(u) -  H^{(m)}(v) \right\|_\infty \\
  &\leq 20 (20^{m}) |u-v|
  = 20^{m+1}|u-v|, \qquad \forall u,v \in [-T,T].
\end{align*}
In other words, $\widehat f_m$ is Lipschitz continuous with  Lipschitz constant $20^{m+1}$, thus (\ref{boundoffm:1}) holds. 

Now we move on to prove (\ref{boundoffm:2}). From (\ref{boundofH:1}) in Lemma \ref{boundofH}, we know
$$
     \|H^{(m)}_{\bm{W},\bm{b}}\|_{L^\infty [-T,T]} \leq 20^mT+4 \left(\frac{20^m-1}{20-1}\right).
$$
As a result, there holds
{\allowdisplaybreaks
 \begin{align*}
     |\widehat f_{m,\theta}(x) - \widehat f_{m,\widetilde \theta}(x)| &= \left|\sum_{i=1}^5 a_i \left(H^{(m)}_{\bm{W},\bm{b}}(x)\right)_i - \sum_{i=1}^5 \widetilde a_i\left(H^{(m)}_{\widetilde{\bm{W}},\widetilde{\bm{b}}}(x)\right)_i\right|\\
     &=\left| \sum_{i=1}^5 (a_i- \widetilde a_i) \left(H^{(m)}_{\bm{W},\bm{b}}(x)\right)_i + \sum_{i=1}^5 \widetilde a_i\left(H^{(m)}_{\bm{W},\bm{b}} (x)- H^{(m)}_{\widetilde{\bm{W}}, \widetilde{\bm{b}}}(x)\right)_i\right|\\
     &\leq  \sum_{i=1}^5 |a_i- \widetilde a_i|\max_{1\leq j \leq 5}\left|\left(H^{(m)}_{\bm{W},\bm{b}}(x)\right)_j\right|+ \sum_{i=1}^5 |\widetilde a_i| \max_{1\leq j \leq 5} \left|\left(H^{(m)}_{\bm{W},\bm{b}} (x)- H^{(m)}_{\widetilde{\bm{W}}, \widetilde{\bm{b}}}(x)\right)_j\right|\\
     &\leq \eta \|H^{(m)}_{\bm{W},\bm{b}}\|_{L^\infty [-T,T]} + 20 \|H^{(m)}_{\bm{W},\bm{b}} - H^{(m)}_{\widetilde{\bm{W}}, \widetilde{\bm{b}}}\|_{L^\infty [-T,T]}\\
     &\leq \eta \left(20^mT+4 \left(\frac{20^m-1}{20-1}\right)\right) + 20 \|H^{(m)}_{\bm{W},\bm{b}} - H^{(m)}_{\widetilde{\bm{W}}, \widetilde{\bm{b}}}\|_{L^\infty [-1,1]}.
 \end{align*}}
 To proceed, we need to compute $\|H^{(m)}_{\bm{W},\bm{b}} - H^{(m)}_{\widetilde{\bm{W}}, \widetilde{\bm{b}}}\|_{L^\infty [-T,T]}$. Applying $|\sigma(u)-\sigma(v)| \leq |u-v|$, for $i=1,\ldots, 5$, $j=1,\ldots,m$, we have
 \begin{align*}
     &\left|\left(H^{(j)}_{\bm{W},\bm{b}}(x) - H^{(j)}_{\widetilde{\bm{W}}, \widetilde{\bm{b}}}(x)\right)_i\right |\\
     &\leq \left|\left((W^{(j)}-\widetilde W^{(j)})H^{(j-1)}_{\bm{W},\bm{b}}(x)\right)_i + \left(\widetilde W^{(j)}(H^{(j-1)}_{\bm{W},\bm{b}}(x) -H^{(j-1)}_{\widetilde{\bm{W}},\widetilde{\bm{b}}}(x)) \right)_i + (b^{(j)} - \widetilde b^{(j)})_i\right|\\
     &\leq \eta \|H^{(j-1)}_{\bm{W},\bm{b}}\|_{L^\infty [-T,T]}  + 20 \|H^{(j-1)}_{\bm{W},\bm{b}} - H^{(j-1)}_{\widetilde{\bm{W}},\bm{b}}\|_{L^\infty [-T,T]} + \eta \\
     &\leq\eta \left(20^{j-1}T+4 \left(\frac{20^{j-1}-1}{20-1}\right)\right) + 20 \|H^{(j-1)}_{\bm{W},\bm{b}} - H^{(j-1)}_{\widetilde{\bm{W}},\bm{b}}\|_{L^\infty [-T,T]} + \eta.
 \end{align*}
 This brings us to a recurrence relationship:
 \begin{align*}
 \left\|H^{(j)}_{\bm{W},\bm{b}} - H^{(j)}_{\widetilde{\bm{W}}, \widetilde{\bm{b}}}\right \|_{L^\infty [-T,T]} &\leq 20 \|H^{(j-1)}_{\bm{W},\bm{b}} - H^{(j-1)}_{\widetilde{\bm{W}},\bm{b}}\|_{L^\infty [-T,T]}+ \eta   \left(20^{j-1}T +4 \left(\frac{20^{j-1}-1}{20-1}\right)+1 \right) \\
 &\leq  20 \|H^{(j-1)}_{\bm{W},\bm{b}} - H^{(j-1)}_{\widetilde{\bm{W}},\bm{b}}\|_{L^\infty [-T,T]}+(T+1) 20^{j-1}\eta .
 \end{align*}
 We thereby obtain
 \begin{align*}
  \left\|H^{(m)}_{\bm{W},\bm{b}} - H^{(m)}_{\widetilde{\bm{W}}, \widetilde{\bm{b}}}\right \|_{L^\infty [-T,T]} 
  &\leq 20^{m-1}  \|H^{(1)}_{\bm{W},\bm{b}} - H^{(1)}_{\widetilde{\bm{W}},\bm{b}}\|_{L^\infty [-T,T]} + (T+1)(m-1) 20^{m-1}\eta\\
  &\leq (T+1)(20^{m-1}) \eta +(T+1)(m-1) 20^{m-1}\eta \\
  &= (T+1)m 20^{m-1}\eta,
 \end{align*}
 where we have used $\|H^{(1)}_{\bm{W},\bm{b}} - H^{(1)}_{\widetilde{\bm{W}},\bm{b}}\|_{L^\infty [-T,T]} \leq \|W^{(1)}-\widetilde W^{(1)}\|_\infty \|H^{(0)}\|_\infty +\|b^{(1)}-\widetilde b^{(1)}\|_\infty \leq (T+1)\eta$.
 Finally, we get 
 \begin{align*}
     \|\widehat f_{m,\theta} - \widehat f_{m,\widetilde \theta}\|_{L^\infty [-T,T]} 
     &\leq \eta \left(20^mT +4 \left(\frac{20^m-1}{20-1}\right)\right)+ 20(T+1)m 20^{m-1}\eta\\
     &\leq (T+1)(m+1)20^m \eta.
 \end{align*}
 The proof of (\ref{boundoffm:2}) is complete.
 \end{proof}

\subsection{Uniform Bound of Product Gate $\widehat{\Phi}$}
\label{app:C2}
We proceed by looking at the product gate $ \widehat{\Phi}(u,v) :=\widehat{\Phi}_\theta(u,v) =  \hat{f}_{m,\theta} \left(\left|\frac{u+v}{2}\right|\right) - \hat{f}_{m,\theta}\left(\left|\frac{u-v}{2}\right|\right)$ defined earlier in Subsection \ref{section:EBTnet}. Here, $\theta = \{\bm{W},\bm{b},a\}$ represents a set of trainable parameters 
taking values on $[-4,4]$.  $\widehat{\Phi}$ is  a ReLU FNN $\in \mathcal{F}(m+1,(4,10,10,\ldots,10))$ with all parameter values in $[-4,4]$. Using the results in Lemma \ref{boundoffm}, we are able to derive  the Lipschitz-1 seminorm of $\widehat{\Phi}$ and the uniform bound $\|\widehat{\Phi}_\theta - \widehat{\Phi}_{\widetilde\theta}\|_\infty$. Denote 
$$|\widehat{\Phi}_\theta|_{Lip1} = \sup_{x,y\in [-T,T]^2, x\neq y}\frac{|\widehat{\Phi}_\theta(x) -  \widehat{\Phi}_\theta(y)|}{\|x-y\|_1}.$$

\begin{lemma} \label{boundofPhi}
Let $\eta >0, T>0$, and $m\in\NN$. Let $\bm{W}^*_\eta, \bm{b}^*_\eta, a^*_\eta$ be $\eta$-nets of $\bm{W}^*, \bm{b}^*, a^*$ defined above in Definition \ref{etanet}. With the input $(u,v)\in [-T,T]^2$, there hold
\begin{equation}\label{boundofPhi:1}
\|\widehat{\Phi}_\theta \|_{L^\infty [-T,T]^2} \leq 20^{m+1}T
 \end{equation}
 and
\begin{equation}\label{boundofPhi:2}
         |\widehat{\Phi}_\theta|_{Lip1} \leq 20^{m+1}
\end{equation}
and
     \begin{equation}\label{boundofPhi:3}
\|\widehat{\Phi}_\theta - \widehat{\Phi}_{\widetilde\theta}\|_{L^\infty [-T,T]^2} \leq 2(T+1)(m+1) 20^m \eta .
 \end{equation}
\end{lemma}
\begin{proof}
Let us first prove (\ref{boundofPhi:1}). Applying the Lipschitz-1 seminorm of $\widehat{f}_m$ from (\ref{boundoffm:1}), we have,  for every $u,v \in [-T,T]$, \begin{align*}
  |\widehat{\Phi}_\theta(u,v)|= \left|\widehat{f}_{m,\theta} \left(\left|\frac{u+v}{2}\right|\right) - \widehat{f}_{m,\theta}\left(\left|\frac{u-v}{2}\right|\right) \right| &\leq 20^{m+1}  \left| \left|\frac{u+v}{2}\right| - \left|\frac{u-v}{2}\right| \right|\\
  &\leq 20^{m+1}|v|\\
  &\leq 20^{m+1}T.
\end{align*} This proves (\ref{boundofPhi:1}).
Next, for every $u_1,u_2,v_1,v_2 \in [-T,T]$, we have
\begin{eqnarray*}
 & &|\widehat{\Phi}(u_1,v_1) -  \widehat{\Phi}(u_2,v_2)| \\
  & \leq& \left| \widehat f_m \left(\left|\frac{u_1+v_1}{2}\right|\right) -\widehat f_m \left(\left|\frac{u_2+v_2}{2}\right|\right) \right| + \left| \widehat f_m \left(\left|\frac{u_1-v_1}{2}\right|\right) -\widehat f_m \left(\left|\frac{u_2-v_2}{2}\right|\right) \right|\\
  &\leq& 20^{m+1} \left|\left|\frac{u_1+v_1}{2}\right|- \left|\frac{u_2+v_2}{2}\right| \right| + 20^{m+1} \left|\left|\frac{u_1-v_1}{2}\right|- \left|\frac{u_2-v_2}{2}\right| \right|\\
  &\leq& 20^{m+1}\left(|u_1-u_2| + |v_1-v_2|\right).
\end{eqnarray*}
This proves (\ref{boundofPhi:2}). 

According to (\ref{boundoffm:2}) in Lemma \ref{boundoffm}, 
     $\|\widehat f_{m,\theta} - \widehat f_{m,\widetilde \theta}\|_{L^\infty [-T,T]} \leq (T+1)(m+1)20^m \eta$.
Then for $(u,v)\in [-T,T]^2$, 
\begin{align*}
 & \|\widehat{\Phi}_\theta (u,v)- \widehat{\Phi}_{\widetilde\theta}(u,v)\|_\infty \\  
  &\leq \left\|\widehat{f}_{m,\theta} \left(\left|\frac{u+v}{2}\right|\right) -\widehat{f}_{m,\widetilde\theta} \left(\left|\frac{u+v}{2}\right|\right)\right\|_\infty + \left\|\widehat{f}_{m,\widetilde\theta}\left(\left|\frac{u-v}{2}\right|\right)- \widehat{f}_{m,\theta}\left(\left|\frac{u-v}{2}\right|\right)\right\|_\infty\\
  &\leq \|\widehat f_{m,\theta} - \widehat f_{m,\widetilde \theta}\|_{L^\infty [-T,T]} + \|\widehat f_{m,\theta} - \widehat f_{m,\widetilde \theta}\|_{L^\infty [-T,T]}\\
  &\leq 2(T+1)(m+1)20^m \eta.
\end{align*}
This proves (\ref{boundofPhi:3}). The proof of Lemma \ref{boundofPhi} is complete.
\end{proof}

\subsection{Uniform Bound of Monomial Gate $\widehat h_k$}
\label{app:C3}
We proceed by looking at the monomial gate $\left\{\widehat h_k:= \widehat h_{k,\theta}\right\}_{k=1}^{2^\ell}$ for some $\ell \in \NN$. Recall the definitions we made at (\ref{hk1}) and (\ref{hk2}):
\begin{equation*} 
 \left\{\widehat h_k(u)\right\}_{k=1}^2 = \left\{\widehat h_1(u)=u,\  \widehat h_2(u) =\widehat\Phi(u,u)=\widehat f_m(|u|) - \widehat f_m(0)\right\},   
\end{equation*} 
and iteratively for $j=1,\ldots, \ell-1$, and $i=1,\ldots, 2^j$,
\begin{equation*}
 \widehat h_{2^j+i}(u) = \widehat \Phi\left(\widehat h_{2^j}(u), \widehat h_i(u)\right) = \hat{f}_{m} \left(\left|\frac{\widehat h_{2^j}(u)+\widehat h_i(u)}{2}\right|\right) - \hat{f}_{m}\left(\left|\frac{\widehat h_{2^j}(u)-\widehat h_i(u)}{2}\right|\right).
\end{equation*}
We define 
\begin{equation}\label{Bj}
B_j:= B_{j,\theta , \widetilde\theta} =\max_{1\leq k \leq 2^j}  \|\widehat{h}_{k,\theta}- \widehat{h}_{k,\widetilde\theta}\|_{L^\infty [-T,T]},\qquad \forall j=1,\ldots, \ell,
\end{equation}
where $\theta = \{\bm{W},\bm{b},a\}$ and $\widetilde \theta = \{\widetilde {\bm{W}},\widetilde {\bm{b}},\widetilde a\}$ represent two different collections of network parameters. The following Lemma presents the uniform bound of $\|\widehat{h}_{k,\theta}\|_{L^\infty [-T,T]}$ and $B_j$, respectively.

\begin{lemma}\label{boundofhk}
Let $T>0, m, \ell \in\NN$. 
There holds, for $j=1,\ldots, \ell$, 
\begin{equation}\label{boundofhk:1}
\max_{1\leq k \leq 2^j} \|\widehat{h}_{k,\theta}\|_{L^\infty [-T,T]} \leq    20^{j(m+1)}
\end{equation}
and \begin{equation}\label{boundofhk:2}
    \max_{1\leq k \leq 2^j} |\widehat{h}_{k,\theta}|_{Lip1} \leq    2 \cdot 20^{j(m+1)}.
\end{equation}
Let $\eta >0$. Also let $\bm{W}^*_\eta, \bm{b}^*_\eta, a^*_\eta$ be $\eta$-nets of $\bm{W}^*, \bm{b}^*, a^*$ defined above in Definition \ref{etanet}. There holds, for $j=1,\ldots, \ell$, 
\begin{equation}\label{boundofhk:3}
    B_j \leq2^{j}(T+1)(m+1)(20^{m+1})^j\eta.
\end{equation}
\end{lemma}
\begin{proof}
Let us first prove (\ref{boundofhk:1}). Recall that from (\ref{boundoffm:1}) Lemma \ref{boundoffm}, we derived $         |\widehat f_{m,\theta}|_{Lip1} \leq 20^{m+1}$ which implies 
\begin{equation*}
    |\widehat f_m(u) -\widehat f_m(v)| \leq 20^{m+1}|u-v|, \qquad \forall u,v \in [-T,T].
\end{equation*}
For $j=1,\ldots, \ell-1$, $i=1,\ldots, 2^j$,
\begin{align*}
    \|\widehat h_{2^j +i}\|_{L^\infty [-T,T]} 
    &= \sup_{u\in [-T,T]} \left|\hat{f}_{m} \left(\left|\frac{\widehat h_{2^j}(u)+\widehat h_i(u)}{2}\right|\right) - \hat{f}_{m}\left(\left|\frac{\widehat h_{2^j}(u)-\widehat h_i(u)}{2}\right|\right)\right|\\
    &\leq \sup_{u\in [-T,T]} 20^{m+1} |\widehat h_i(u)| = 20^{m+1} \|\widehat h_i\|_{L^\infty [-T,T]}.
\end{align*}
We thereby obtain the relation
\begin{equation*}
    \max_{1\leq k \leq 2^{j}}  \|\widehat h_k\|_{L^\infty [-T,T]} \leq 20^{m+1} \max_{1\leq k \leq 2^{j-1}}  \|\widehat h_k\|_{L^\infty [-T,T]}, \qquad \forall j=1,\ldots, \ell.
\end{equation*}
By induction, we have $\max_{1\leq k \leq 2^j}  \|\widehat h_k\|_{L^\infty [-T,T]} \leq 20^{(j-1)(m+1)} \max_{1\leq k \leq 2}  \|\widehat h_k\|_{L^\infty [-T,T]}\leq 20^{j(m+1)}$. This proves (\ref{boundofhk:1}).
In the same way, 
\begin{align*}
    |\widehat h_{2^j+i}(u) - \widehat h_{2^j+i}(v)| &= \left|\widehat \Phi\left(\widehat h_{2^j}(u), \widehat h_i(u)\right) - \widehat \Phi\left(\widehat h_{2^j}(v), \widehat h_i(v)\right)\right| \\
    &\leq |\widehat{\Phi}|_{Lip1}\left(|\widehat h_{2^j}(u)-\widehat h_{2^j}(v)| + |\widehat h_i(u)-\widehat h_i(v)|\right)\\
    &\leq 20^{m+1} \left(|\widehat h_{2^j}|_{Lip1}+ |\widehat h_i|_{Lip1}\right)|u-v|,
\end{align*}
which implies by induction
\begin{equation*}
    \max_{1\leq k \leq 2^j} |\widehat{h}_k|_{Lip1} \leq 20^{m+1}  \max_{1\leq k \leq 2^{j-1}}  2 |\widehat h_k|_{Lip1}\leq   2 \cdot 20^{j(m+1)}.
\end{equation*} 
This proves (\ref{boundofhk:2}).

Next, we move on to prove (\ref{boundofhk:3}). From (\ref{boundofPhi:3}) of Lemma \ref{boundofPhi}, we have 
\begin{equation*}
    B_1 = \max_{1\leq k \leq 2}  \|\widehat{h}_{k,\theta}- \widehat{h}_{k,\widetilde\theta}\|_{L^\infty [-T,T]} 
    \leq
    \|\widehat{\Phi}_\theta - \widehat{\Phi}_{\widetilde\theta}\|_{L^\infty [-T,T]^2} \leq 2 (T+1)(m+1)20^m \eta.
\end{equation*}
Then, by the definition of $B_j$ in (\ref{Bj}), for $j=1,\ldots, \ell-1$,
\begin{align*}
  B_{j+1}  = \max_{1\leq k \leq 2^{j+1}}  \|\widehat{h}_{k,\theta}- \widehat{h}_{k,\widetilde\theta}\|_{L^\infty [-T,T]}
  =\max \ \left\{B_j,  \max_{1\leq k \leq 2^j} \|\widehat{h}_{2^j+k,\theta}- \widehat{h}_{2^j+k,\widetilde\theta}\|_{L^\infty [-T,T]}\right\}.
\end{align*}
For $1\leq k \leq 2^j$,
{\allowdisplaybreaks
\begin{eqnarray*}
 &  &\|\widehat{h}_{2^j+k,\theta}- \widehat{h}_{2^j+k,\widetilde\theta}\|_{L^\infty [-T,T]} \\
   &\leq& \|\widehat{\Phi}_\theta (\widehat{h}_{2^j,\theta}, \widehat{h}_{k,\theta})- \widehat{\Phi}_{\widetilde\theta}(\widehat{h}_{2^j,\theta}, \widehat{h}_{k,\theta})\|_{L^\infty [-T,T]}  
  +\|\widehat{\Phi}_{\widetilde\theta}(\widehat{h}_{2^j,\theta}, \widehat{h}_{k,\theta}) - \widehat{\Phi}_{\widetilde\theta}(\widehat{h}_{2^j,\widetilde\theta}, \widehat{h}_{k,\widetilde\theta})\|_{L^\infty [-T,T]}\\
   &\leq& 2(T+1)(m+1)20^m \eta + \|\widehat{\Phi}_{\widetilde\theta}(\widehat{h}_{2^j,\theta}, \widehat{h}_{k,\theta}) - \widehat{\Phi}_{\widetilde\theta}(\widehat{h}_{2^j,\widetilde\theta}, \widehat{h}_{k,\widetilde\theta})\|_{L^\infty [-T,T]}\\
   &\leq & 2(T+1)(m+1)20^m \eta +2\left(20^{m+1}\right)B_j.
\end{eqnarray*}}
Now, plugging this into the above iteration relation, we get
\begin{equation*}
    B_{j+1} \leq  2(T+1)(m+1)20^m \eta +2\left(20^{m+1}\right)B_j,\qquad \forall j=1,\ldots, \ell-1,
\end{equation*}
which is a recurrence relationship. We finally get
{\allowdisplaybreaks
\begin{align*}
 B_j &\leq   \left(2\cdot 20^{m+1}\right)^{j-1} B_1 +(2(T+1)(m+1)20^m \eta) \left(1+ 2\cdot 20^{m+1} + \cdots + (2\cdot 20^{m+1})^{j-2}\right)\\
 &\leq \frac{2^{j+1}(20^{m+1})^j}{2(20^{m+1})-1}(T+1)(m+1)20^m \eta \\
 &\leq 2^j(T+1)(m+1)(20^{m+1})^j\eta.
\end{align*}}
The proof is complete.
\end{proof}

\subsection{Proof of proposition \ref{coveringnumber}}\label{app:C4}
Recall $c^* = [-C_K,C_K]^{2^\ell+1}$ 
 with $C_K$ to be a positive constant given in Definition \ref{hypothesisspace}. 
 For $\eta >0$, let $c^*_\eta$ be an $\eta$-net of $c^*$ such that for each $c\in c^*$, there exists $\widetilde c \in c^*_\eta$ such that
\begin{equation} \label{etafc}
    \|c-\widetilde c\|_1 = \sum_{k=0}^{2^\ell} |c_k- \widetilde c_k| \leq \eta. 
\end{equation}

\noindent
{\bf Proof of Proposition \ref{coveringnumber}}.
Applying (\ref{boundofhk:2}) in Lemma \ref{boundofhk} and the fact that $\widehat{h}_{1,\theta}(u) = \widehat{h}_{1,\widetilde \theta}(u) = u \in [-T,T]$, we obtain
\begin{align*}
\sum_{k=1}^{2^\ell}\|\widehat{h}_{k,\theta}- \widehat{h}_{k,\widetilde\theta}\|_{L^\infty [-T,T]} = \sum_{k=2}^{2^\ell}\|\widehat{h}_{k,\theta}- \widehat{h}_{k,\widetilde\theta}\|_{L^\infty [-T,T]} 
&\leq \sum_{j=1}^{\ell} 2^{j-1}B_j\\
&\leq \sum_{j=1}^{\ell} 2^{j-1}  2^{j}(T+1)(m+1)(20^{m+1})^j\eta\\
&= \frac{(T+1)(m+1)}{2}\sum_{j=1}^{\ell} (4(20^{m+1}))^j\eta\\
&\leq (T+1)(m+1)(4(20^{m+1}))^\ell\eta.
\end{align*}
We then have
{\allowdisplaybreaks
\begin{align*}
 & \|\sum_{k=1}^{2^\ell}(c_k\widehat{h}_{k,\theta} + c_0- \widetilde c_k \widehat{h}_{k,\widetilde\theta} - \widetilde c_0)\|_{L^\infty [-T,T]}  \\
  &= \|\sum_{k=1}^{2^\ell}(c_k\widehat{h}_{k,\theta}- \widetilde c_k\widehat{h}_{k,\theta}) + \sum_{k=1}^{2^\ell} (\widetilde c_k\widehat{h}_{k,\theta}- \widetilde c_k \widehat{h}_{k,\widetilde\theta})+ (c_0-\widetilde c_0)\|_{L^\infty [-T,T]} \\
  &\leq  \max_{1\leq k \leq 2^\ell} \|\widehat{h}_{k,\theta}\|_{L^\infty [-T,T]}\eta + 2^\ell C_K  (T+1)(m+1)(4(20^{m+1}))^\ell\eta\\
 \ &\leq20^{\ell(m+1)}\eta + 2^\ell C_K  (T+1)(m+1)(4(20^{m+1}))^\ell\eta\\
  &\leq C (T+1)(m+1)( 8(20^{m+1}))^\ell\eta,
\end{align*}}
where $C=1+ C_K$ is a constant greater than $1$.

Next, recall the function $\sigma_\lambda$ defined in Definition \ref{hypothesisspace} for $0 < \lambda \leq 1$. We observe that 
\begin{equation} \label{sigmalambdaLip}
    |\sigma_\lambda(u) - \sigma_\lambda(v)| \leq \frac{1}{\lambda} |u-v|, \qquad \forall u,v \in \RR.
\end{equation}
Note that each  $f_\mathcal{H} \in \mathcal{H}$ has the form $\sigma_\lambda \left(\sum_{j=1}^{K}\sum_{k=1}^{2^\ell}c_{k,j,\theta} h^*_{k,j,\theta}(x)+ c_{0,\theta} \right)$. 
Let $f_{\mathcal{H},\theta}$ and $f_{\mathcal{H},\widetilde \theta}$ represents two functions in $\mathcal{H}$ with the set of parameters $\theta = \{\mathring W, \mathring b, \bm{W},\bm{b},a,c\}\in \Theta$ and $\widetilde \theta = \{\widetilde{\mathring W},\widetilde {\mathring b},\widetilde {\bm{W}},\widetilde {\bm{b}},\widetilde a, ,\widetilde c\} \in \widetilde \Theta$ respectively. To estimate  $\mathcal{N}(\epsilon,\mathcal{H})$ for any $0 < \epsilon \leq 1$, we need to find a set of functions in $\mathcal{H}$ that forms an $\epsilon$-net. 

To obtain such a function set, we choose $\widetilde \Theta$ in such a way that for any $\theta \in \Theta$, there exists $\widetilde \theta \in \widetilde \Theta$ such that 
{\allowdisplaybreaks
 \begin{eqnarray}\label{etarequire}
     &&\|\mathring W -\widetilde{\mathring W} \|_\infty \leq \eta, \qquad \|\mathring b-\widetilde{\mathring b} \|_\infty \leq \eta, \qquad \|\bm{W} -\widetilde{\bm{W}} \|_{\infty, \infty} \leq \eta,\\
     &&\qquad \|\bm{b} -\widetilde{\bm{b}} \|_\infty \leq \eta,\qquad  \|a-\widetilde{a} \|_1 \leq \eta,  \qquad \|c-\widetilde c\|_1 \leq \eta.    \nonumber
 \end{eqnarray} }
 Note that for  $x\in [-b-1, b+1]^d$,
 \begin{align*}
     |\widehat r_{i,j,\theta}(x) - \widehat r_{i,j,\widetilde \theta}(x) |
     &= |\mathring W_{i,j,\theta} \cdot x + \mathring b_{i,j,\theta} - \mathring W_{i,j,\widetilde\theta} \cdot x - \mathring b_{i,j,\widetilde\theta}|\\
     &\leq |(\mathring W_{i,j,\theta}  - \mathring W_{i,j,\widetilde\theta})\cdot x| + |\mathring b_{i,j,\theta}-\mathring b_{i,j,\widetilde\theta}|
     \leq d(b+1)\eta +\eta,
 \end{align*}
 which implies by (\ref{boundofPhi:2}) in Lemma \ref{boundofPhi} that 
 $$\left|\widehat \Phi_\theta(\Psi_b(x), \widehat r_{i,j,\theta}(x))- \widehat \Phi_\theta(\Psi_b(x), \widehat r_{i,j,\widetilde \theta}(x))\right| \leq 20^{m+1}(d(b+1)+1)\eta.$$
 
Observe that $|\widehat r_{i,j}(x)| = |\mathring W_{i,j} \cdot x + \mathring b_{i,j}| \leq C_K^0 \sqrt{d}(b+1)+ C_K^1 $ for $x\in [-b-1, b+1]^d$.
Then by (\ref{boundofPhi:3}) in Lemma \ref{boundofPhi} with $T_1=\max\{C_K^0 \sqrt{d}(b+1)+ C_K^1,1\}$, we have 
 \begin{eqnarray*}
 & &\left|\widehat\Phi_\theta(\Psi_b(x), \widehat r_{i,j,\theta}(x))- \widehat\Phi_ {\widetilde\theta }(\Psi_b(x), \widehat r_{i,j,\widetilde\theta}(x))\right| \\
  &\leq& 20^{m+1}(d(b+1)+1)\eta +  2(T_1+1)(m+1) 20^m \eta\\
  &\leq& (20(d(b+1)+1) + (2C_K^0 \sqrt{d}(b+1)+ 2C_K^1+4)(m+1))20^m \eta.
 \end{eqnarray*}
 Also, $|\widehat\Phi_\theta(\Psi_b(x), \widehat r_{i,j,\theta}(x))| \leq T_1 \cdot 20^{m+1}$ by (\ref{boundofPhi:1}) of Lemma \ref{boundofPhi} which implies 
 \begin{align*}
 \left|\widehat f_{m,\theta}\left(\left|\widehat \Phi_\theta(\Psi_b(x),\widehat r_{i,j,\theta}(x))\right|\right)\right| 
 &\leq 20 \|H^{(m)}_\theta\|_{L^\infty [-T_1 \cdot 20^{m+1},T_1 \cdot 20^{m+1}]}\\
 &\leq 20 \left(20^m (T_1 \cdot 20^{m+1})+4 \left(\frac{20^m-1}{20-1}\right)\right) \leq (T_1+1)20^{2m+2} =:T_2.
 \end{align*}
 
It follows from (\ref{sigmalambdaLip}), (\ref{boundofhk:2}), and (\ref{boundoffm:2}) that for $x\in [-b-1, b+1]^d$,
{\allowdisplaybreaks
\begin{eqnarray*}
  &  &|f_{\mathcal{H},\theta}(x) - f_{\mathcal{H},\widetilde \theta}(x)|\\
    &=& \left|\sigma_\lambda \left(\sum_{j=1}^{K}\sum_{k=1}^{2^\ell}c_{k,j,\theta} h^{*}_{k,j,\theta}(x) + c_{0,\theta} \right) -\sigma_\lambda \left(\sum_{j=1}^{K}\sum_{k=1}^{2^\ell}c_{k,j,\tilde\theta} h^*_{k,j,\tilde\theta}(x) +c_{0,\widetilde\theta}\right)\right| \\
    &\leq& \frac{1}{\lambda}\left|\left(\sum_{j=1}^{K}\sum_{k=1}^{2^\ell}c_{k,j,\theta} h^*_{k,j,\theta}(x)+ c_{0,\theta} \right) - \left(\sum_{j=1}^{K}\sum_{k=1}^{2^\ell}c_{k,j,\tilde\theta} h^*_{k,j,\widetilde\theta}(x) +c_{0,\widetilde\theta}\right)\right|\\
    & \leq& \frac{1}{\lambda} K C (T_2+1)(m+1)( 8(20^{m+1}))^\ell\eta + \Biggr|\frac{1}{\lambda} \sum_{j=1}^{K}\sum_{k=1}^{2^\ell}c_{k,j,\widetilde\theta} \Biggl\{\widehat h_{k,j, \widetilde\theta}\left(\frac{1}{d}\sum_{i=1}^d \widehat f_{m,\theta}\left(\left|\widehat \Phi_{\theta}(\Psi_b(x),\widehat r_{i,j,\theta}(x))\right|\right)\right)\\
    &&-\widehat h_{k,j, \widetilde\theta}\left(\frac{1}{d}\sum_{i=1}^d \widehat f_{m,\widetilde\theta}\left(\left|\widehat \Phi_{\widetilde\theta}(\Psi_b(x),\widehat r_{i,j,\widetilde\theta}(x))\right|\right)\right)\Biggr\}\Biggr|\\
    & \leq& \frac{1}{\lambda} K C (T_2+1)(m+1)( 8(20^{m+1}))^\ell\eta 
    + \frac{1}{\lambda}K2^{\ell+1} C_K 20^{\ell(m+1)} (((T_1 \cdot 20^{m+1}+1)(m+1) 20^m \eta \\
    &&+ 20^{m+1} (20(d(b+1)+1) + (2C_K^0 \sqrt{d}(b+1)+ 2C_K^1+4)(m+1))20^m \eta)\\
    &\leq& \frac{K}{\lambda}  C_K^\prime (m+1)(b+1)20^{(\ell+2)(m+2)} (C+C_K) \eta,
\end{eqnarray*}}
where $C_K^\prime >0$ is a constant depending on $d,C_K^0 ,C_K^1$.

When $x \notin [-b-1, b+1]^d$, we have $\Psi_b(x)=0$ and thereby $h^*_{k,j} (x)= \widehat h_{k}\left(\frac{1}{d}\sum_{i=1}^d \widehat f_{m}(0)\right) = \widehat h_{k}(\widehat f_{m}(0)) $ and by (\ref{boundofH:1}) with $T=1$ and the bound for 
{\allowdisplaybreaks
\begin{eqnarray*}
&&|f_{\mathcal{H},\theta}(x) - f_{\mathcal{H},\widetilde \theta}(x)|\\
&\leq&\frac{1}{\lambda}\left|\left(\sum_{j=1}^{K}\sum_{k=1}^{2^\ell}c_{k,j,\theta}  \widehat h_{k,\theta}(\widehat f_{m,\theta}(0))+c_{0,\theta}\right) - \left(\sum_{j=1}^{K}\sum_{k=1}^{2^\ell}c_{k,j,\widetilde\theta}  \widehat h_{k,\widetilde\theta} (\widehat f_{m,\widetilde\theta}(0))+c_{0,\widetilde\theta} \right)\right| \\
&\leq& \frac{1}{\lambda}K C (2\cdot 20^{m+1})(m+1)( 8(20^{m+1}))^\ell\eta + \frac{1}{\lambda} \sum_{j=1}^{K}\sum_{k=1}^{2^\ell}|c_{k,j,\widetilde\theta}| \left| \widehat h_{k,\widetilde\theta} (\widehat f_{m,\theta}(0)) - \widehat h_{k,\widetilde\theta} (\widehat f_{m,\widetilde\theta}(0))\right|\\
&\leq&  \frac{1}{\lambda}K C (2\cdot 20^{m+1})(m+1)( 8(20^{m+1}))^\ell\eta + \frac{1}{\lambda} K 2^\ell C_K 2\cdot 20^{\ell(m+1)}|f_{m,\theta}(0) - f_{m,\widetilde\theta}(0)|\\
&\leq&  \frac{1}{\lambda}K C (2\cdot 20^{m+1})(m+1)( 8(20^{m+1}))^\ell\eta + \frac{1}{\lambda} K 2^\ell C_K 2\cdot 20^{\ell(m+1)}(2m+2)20^m \eta \\
&\leq& \frac{K}{\lambda}(2\cdot 20^{m+1})(m+1) 20^{(\ell+1)(m+2)}  (C+C_K )\eta.
\end{eqnarray*}}
Therefore, $\|f_{\mathcal{H},\theta} - f_{\mathcal{H},\widetilde \theta}\|_\infty \leq \frac{K}{\lambda}(C_K^\prime+1)(2m+2)(b+1)20^{(\ell+2)(m+2)} (C+C_K) \eta$.
Hence, $\|f_{\mathcal{H},\theta} - f_{\mathcal{H},\widetilde \theta}\|_\infty\leq \epsilon $ if
\begin{equation}\label{boundofeta}
    \eta = \frac{\lambda \epsilon}{K(b+1)(C_K^\prime+1)(C+C_K)(2m+2)20^{(\ell+2)(m+2)}}.
\end{equation}
Since $0 < \lambda \leq 1, 0<\epsilon\leq 1,C>1, C_K, C_K^\prime >0, b>1, m,\ell,K \in \NN$, we can see that $0 <\eta \leq1$.

Functions in $\mathcal{H}$ can be implemented by a neural network consisting of a preprocessing subnetwork, a stack of $dK$ product gates $\Phi$, and a stack of $K$ EBTnets. For more details, please refer to Remark \ref{remark}.

The preprocessing subnetwork consists of $dK$ product gates $\widehat\Phi$ and $dK$ units of $\mathcal{F}(2,(2,1))$ that is equipped with one $\mathring W \in \RR^d$ with $\|\mathring W\| \leq C_K^0$, and one $\mathring b \in \RR$ with $|\mathring b|  \leq C_K^1$. This implies that the preprocessing subnetwork consists of $2dK\  \bm{W} \in \bm{W}^*$, $\bm{b} \in \bm{b}^*$, $a\in a^*$, and $dK\ \mathring W$ and $dK\ \mathring b$. 

Each EBTnet consists of 
 $1+2+\ldots+2^{\ell-1} = 2^\ell -1$ product gates $\widehat\Phi$, and each $\widehat\Phi$ consists of two $\widehat f_m$. Each $\widehat f_m$ consists of one $\bm{W} \in \bm{W}^*$, one $\bm{b} \in \bm{b}^*$ and one $a\in a^*$. This implies that each EBTnet is equipped with $2(2^\ell-1) \bm{W} \in \bm{W}^*$, $\bm{b} \in \bm{b}^*$, $a\in a^*$.  Finally, we see that such a neural network  is equipped with $2K(2^\ell + 2d-1)$ $\bm{W} \in \bm{W}^*,\bm{b} \in \bm{b}^*, a\in a^*$, and $K\ c \in c^*$, and $dK$ $\mathring W$ and $dK$ $\mathring b$.

To satisfy the requirements in (\ref{etarequire}), it suffices to choose a set $\widetilde \Theta$ which, with $\tilde C = C_K^0/\sqrt{d}$, has cardinality at most 
{\allowdisplaybreaks
\begin{align*}
    & \left\lceil\frac{20}{\eta}\right\rceil^{5(2K)(2^\ell+2d-1)} \left\lceil\frac{4}{\eta}\right\rceil^{5m(2K)(2^\ell+2d-1)}\left\lceil\frac{20}{\eta}\right\rceil^{25m(2K)(2^\ell+2d-1)}\left\lceil\frac{2^\ell C_K}{\eta}\right\rceil^{K2^\ell}\left\lceil\frac{\tilde C}{\eta}\right\rceil^{d^2K}\left\lceil\frac{C_K^1}{\eta}\right\rceil^{dK}\\
    &\leq \left(\frac{21}{\eta}\right)^{(10+50m)K(2^\ell+2d)} \left(\frac{5}{\eta}\right)^{10mK(2^\ell+2d)}\left(\frac{2^\ell(C_K+1)}{\eta}\right)^{K2^\ell} \left(\frac{\tilde C+1}{\eta}\right)^{d^2K}\left(\frac{C_K^1+1}{\eta}\right)^{dK}\\
    &\leq 21^{(60m)K(2^\ell+2d)} 5^{10mK(2^\ell+2d)} \left(2^{\ell K2^\ell}\right)(C_K+1) ^{\ell K 2^\ell} (\tilde C+1)^{d^2K} (C_K^1+1)^{d^2K}\eta^{-K(70m(2^\ell+2d)+2^\ell +2d^2)}\\
    &\leq (21^{60}5^{10})^{mK(2^\ell+2d)} (2(C_K+1))^{\ell K2^\ell}((\tilde C+1)(C_K^1+1))^{d^2K} \left(\frac{1}{\eta}\right)^{71mK(2^\ell+2d^2)}.
\end{align*}}

Then, plugging in $\eta$ given in (\ref{boundofeta}), we obtain
\begin{eqnarray*}
   \mathcal{N}(\epsilon,\mathcal{H}) 
    &\leq &(21^{60}5^{10})^{mK(2^\ell+2d)} (2(C_K+1))^{\ell K2^\ell}((\tilde C+1)(C_K^1+1))^{d^2K} \\
    &&\left(\frac{K(b+1)(C_K^\prime+1)(C+C_K)(2m+2)20^{(\ell+2)(m+2)}}{\lambda \epsilon}\right)^{71mK(2^\ell+2d^2)}.   
\end{eqnarray*}
Then, we have 
{\allowdisplaybreaks
\begin{eqnarray*}
   & &\frac{1}{K}\log \mathcal{N}(\epsilon,\mathcal{H}) \\
     &\leq& m(2^\ell+2d) \log (21^{60}5^{10}) + 4\ell 2^\ell \log (C_K)+ d^2 \log ((\tilde C+1)(C_K^1+1))\\
    &&+71m(2^\ell+2d^2)\left(\log (96KC_K^\prime)+\log \left(\frac{bmC_K}{\lambda \epsilon}\right)\right)+ 639m^2\ell(2^\ell+2d^2)\log (20)\\
    & \leq& 72m(2^\ell)(2d^2)\left(\log (21^{60}5^{10}96KC_K^\prime(\tilde C+1)(C_K^1+1))+\log \left(\frac{bmC_K}{\lambda \epsilon}\right) \right) \\
    &&+ 4\ell 2^\ell \log (C_K) + 639 \log (20) m^2\ell(2^\ell)(2d^2)\\
     &\leq& \frac{C^\prime}{K} m 2^\ell \log \left(\frac{bmC_K}{\lambda \epsilon}\right) + 4\ell 2^\ell \log (C_K)+ 639 \log (20) m^2\ell(2^\ell)(2d^2)\\
     & \leq& \frac{C^\prime}{K} m 2^\ell \log \left(\frac{bC_K}{\lambda \epsilon}\right) + 4\ell 2^\ell \log (C_K) +\frac{C^{\prime\prime}}{K}m^2\ell(2^\ell), 
\end{eqnarray*}}
where $C^\prime, C^{\prime\prime}$ are positive constants independent of $\ell, m,b,\lambda,C_K$ or $\epsilon$.
\small
 \bibliography{ref}

\end{document}